\documentclass[twoside,11pt]{article}

\usepackage{blindtext}

\usepackage[preprint]{jmlr2e}
\usepackage{amsmath}
\usepackage{dsfont}
\usepackage{algorithm2e}
\usepackage{xcolor}
\usepackage{comment}
\usepackage{graphicx}
\usepackage[T1]{fontenc}
\newcommand{\red}[1]{\textcolor{black}{#1}}

\DeclareMathOperator*{\argmin}{arg\,min}
\DeclareMathOperator*{\trace}{tr}
\DeclareMathOperator*{\sign}{sign}

\usepackage{lastpage}
\jmlrheading{?}{2024}{1-\pageref{LastPage}}{7/24; Revised ?/?}{?/?}{?}{Bertille FOLLAIN and Francis BACH}

\ShortHeadings{Feature Learning with Neural Networks and Kernels}{B. FOLLAIN and F. BACH}
\firstpageno{1}
\begin{document}

\title{Enhanced Feature Learning via Regularisation: \\
Integrating Neural Networks and Kernel Methods}

\author{\name Bertille Follain \email bertille.follain@inria.fr \\
       \addr Inria, Département d’Informatique de l’Ecole Normale Supérieure, PSL Research University \\
       48 rue Barrault, 75013, Paris, France
       \AND
       \name Francis Bach \email francis.bach@inria.fr \\
       \addr Inria, Département d’Informatique de l’Ecole Normale Supérieure, PSL Research University \\
       48 rue Barrault, 75013, Paris, France}

\editor{}

\maketitle

\begin{abstract}
We propose a new method for feature learning and function estimation in supervised learning via regularised empirical risk minimisation. Our approach considers functions as expectations of Sobolev functions over all possible one-dimensional projections of the data. This framework is similar to kernel ridge regression, where the kernel is $\mathbb{E}_w ( k^{(B)}(w^\top x,w^\top x^\prime))$, with $k^{(B)}(a,b) := \min(|a|, |b|)\mathds{1}_{ab>0}$ the Brownian kernel, and the distribution of the projections $w$ is learnt. This can also be viewed as an infinite-width one-hidden layer neural network, optimising the first layer's weights through gradient descent and explicitly adjusting the non-linearity and weights of the second layer. We introduce a gradient-based computational method for the estimator, called \textsc{Brownian Kernel Neural Network} (\textsc{BKerNN}), using particles to approximate the expectation, where the positive homogeneity of the Brownian kernel \red{leads to improved robustness to local minima}. Using Rademacher complexity, we show that \textsc{BKerNN}'s expected risk converges to the minimal risk with explicit high-probability rates of $O( \min((d/n)^{1/2}, n^{-1/6}))$ (up to logarithmic factors). Numerical experiments confirm our optimisation intuitions, and \textsc{BKerNN} outperforms kernel ridge regression, and favourably compares to a one-hidden layer neural network with ReLU activations in various settings and real data sets.
\end{abstract}


\section{Introduction}
\label{sec:introduction}
In the era of high-dimensional data, effective feature selection methods are crucial. Representation learning aims to automate this process, extracting meaningful information from complex data sets. Non-parametric methods often struggle in high-dimensional settings, making the multi-index model, which assumes a few relevant linear features explain the relationship between response and factors, an attractive alternative. Formally, the multiple index model \citep{multi_index_model} is expressed as $Y = f^*(X) + {\rm noise} = g^*(P^\top X) + {\rm noise}$, with $Y$ the response, $X$ the $d$-dimensional covariates, $g^*$ the unknown link function, $P \in \mathbb{R}^{d \times k}$ the features and $k \leq d$, the number of such relevant linear features. The components $P^\top X$ are linear features of the data that need to be learnt, reducing the dimensionality of the problem, which may allow to escape the curse of dimensionality, while the more general function $g$ increases the capacity of the model.

Multiple index models have been extensively studied, leading to various methods for estimating the feature space. \citet{brillinger} introduced the method of moments for Gaussian data and one feature, by using specific moments to eliminate the unknown function. For features of any dimension, several methods have been proposed. Sliced inverse regression (\textsc{SIR}) \citep{Li1991} uses second-order moments to identify effective dimensions by slicing the response variable and finding linear combinations of predictors, while improvements have been proposed \citep{balasubramanian}, these methods heavily rely on assumptions about the covariate distribution shape and prior knowledge of the distribution. Iterative improvements have been an interesting line of work \citep{dalalyan:hal-00128129}, while optimisation-based methods like local averaging minimise an objective function to estimate the subspace \citep{kernel_dimension_reduction_francis, Xia2002}. Despite their practical performance, particularly the \textsc{MAVE} method \citep{Xia2002}, the theoretical guarantees show exponential dependence on the original data dimension, making them less suitable for high-dimensional settings.
\red{More recently, neural networks have also been applied to this problem, as demonstrated by \cite{mousavihosseini2024learningmultiindexmodelsneural} and \cite{loucas} (for the single-index model), who both focused on the continuous limit of the optimisation process. Furthermore, \cite{bach2017breaking} has shown that one-hidden-layer neural networks of infinite width with ReLU activation are adaptive to hidden linear features. Unlike earlier methods, these neural network approaches simultaneously learn both the feature
space and the prediction function, a key difference from the other presented methods. Extensions have also been proposed to tackle non-linear feature learning as in the introductory work by \cite{two_layer_feature_learning}.}

In this work, we tackle feature learning and function estimation jointly through the paradigm of empirical risk minimisation. We consider a classical supervised learning problem. We have i.i.d.~samples $(x_i, y_i)_{i \in [n]}$ from a random variable $(X,Y) \in \mathcal{X} \times \mathcal{Y} \subset \mathbb{R}^d \times \mathbb{R}$. Our goal is to minimise the expected risk, which is defined as~$\mathcal{R}(f) := \mathbb{E}_{X,Y}[\ell(Y, f(X))]$ over some class of functions $\mathcal{F}$, where $\ell$ is a loss function mapping from $\mathbb{R} \times \mathbb{R}$ to $\mathbb{R}$. This can be achieved through the framework of regularised empirical risk minimisation, where the empirical risk is defined as $ \widehat{\mathcal{R}}(f) := \frac{1}{n} \sum_{i=1}^n \ell(y_i, f(x_i))$. Our interest in regularised empirical risk minimisation stems from its flexibility, allowing it to be applied to a wide range of problems as long as the objective can be defined as the optimisation of an expected loss. Our primary objective is to achieve the lowest possible risk, which we study in theory and in practice, while we explore the recovery of underlying features in numerical experiments. Our method draws inspiration from several lines of work, namely positive definite kernels and neural networks with their mean field limit, which we briefly review, together with the main limitations we aim to alleviate.

\paragraph{Kernel methods and multiple kernel learning.} A well-known method in supervised learning is kernel ridge regression \citep[\textsc{KRR},][]{kernelridge}, which implicitly maps data into high-dimensional feature spaces using kernels. It benefits from dimension-independent rates of convergence if the model is well-specified, i.e., if the target function belongs to the related Hilbert space. However, \textsc{KRR} does not benefit from the existence of linear features in terms of convergence rates of the risk when the model is misspecified \citep[Section 9.3.5]{francis_book}, as it relies on pre-specified features. To address the limitations of single-kernel methods, multiple kernel learning (\textsc{MKL}) optimally combines multiple kernels to capture different data aspects \citep{bach2004multiple,gonen2011multiple}. However, \textsc{MKL} suffers from significant computational complexity and the critical choice of base kernels, which can introduce biases if not selected properly. Furthermore, \textsc{MKL} does not resolve the issue of leveraging hidden linear features effectively.

\paragraph{Neural networks.}
Now consider another type of supervised learning methods, namely neural networks with an input layer of size $d$, a hidden layer with $m$ neurons, an activation function $\sigma$, followed by an output layer of size 1. Functions which can be represented are of the form $f(x) = \sum_{j=1}^m \eta_j \sigma(w_j^\top x + b_j)$, where $\sigma$ can be the ReLU, $\sigma(z) = \max(0, z)$ or the step function $\sigma(z) = \mathds{1}_{z>0}$. Neural networks benefit from hidden linear features, achieving favourable rates dependent on $k$, the number of relevant features, rather than on the data dimension \citep[Section~9.3.5]{francis_book}. However, this formulation requires multiple $b$ values to fit a function with the same $w$, particularly in single-index models $f^*(x) = g^*(w^\top x)$, which is inefficient.

\red{Regularising the empirical risk minimization objective with a penalty term helps guide estimator behaviour (towards few relevant features for example) and can often be interpreted as a form of function space regularisation. Recent advances in norm-based capacity measures in neural networks have deepened the understanding of how such regularisation influences generalisation. In particular, the path norm, originally introduced for ReLU networks, provides strong generalization guarantees due to its scale-invariant properties \citep{liu}. This notion has been refined through the basis-path norm, which disentangles independent basis paths, offering a more precise characterization of neural network complexity \citep{zheng2018capacity}. Building on these ideas, \citet{parhi2020banach} extend the analysis from a variational perspective, linking path norm regularisation to Banach space representations of neural networks, while \citet{norm-based_bounds} focus on norm-based bounds for sparse networks.}

Now, in the specific context of feature learning \red{(and not necessarily neural networks)}, \citet{JMLR:v14:rosasco13a} used derivatives for regularisation in nonparametric models, focusing on variable selection. While their method reduces to classical regularisation techniques for linear functions, it faces limitations: functions depending on a single variable do not belong in the chosen RKHS, using derivatives at data points limits the exploitation of regularity, and there is no benefit from hidden variables in the misspecified case. An improvement over this method was studied for both feature learning and variable selection by \cite{follain2024nonparametriclinearfeaturelearning}, where a trace norm penalty \citep{koltchinskii2011nuclear,christophe_giraud} on the derivatives was used for the feature learning case. However, the dependency on the dimension of the rate did not allow high-dimensional learning. We can justify the use of trace norm penalties by considering the structure of neural networks. Under the multiple-index model, the weights~$w_1, \ldots, w_m $ of the first layer are expected to lie in a low-rank subspace of rank at most $k$. However, directly enforcing a rank constraint is not practical for optimisation. Therefore, we could use a relaxation such as $ \Omega(f) = \trace\big(\big(\sum_{j=1}^m w_j w_j^\top\big)^{1/2}\big)$, which is the trace norm of a matrix containing the weights, to approximate the rank constraint effectively. However, there is still the issue of multiple constant terms \red{(often referred to as ``biases'')} for a single weight. We will see specialised penalties for feature learning for a different family of functions.  

\paragraph{Mean-field limit.} To apply a similar framework to our future estimator, we introduce the mean-field limit of an over-parameterised one-hidden layer neural network~\citep{nitanda2017stochastic,mei2019mean,chizat_bach_optim,sirignano2020mean}. When the number of neurons $m$  is very large, the network can be rescaled as follows
\begin{equation}
\label{eq:description_neural_net}
   f(x) =  \frac{1}{m} \sum_{j=1}^m \eta_j \sigma(w_j^\top x + b_j), \text{ which approximates }  \int \eta \sigma(w^\top x + b) \, {\rm d} \mu{(\eta, w, b)},
\end{equation}
where $\mu$ is a probability distribution, and we can take the weights and constant terms $(w,b)$ to be constrained when the activation is 1-homogeneous,\footnote{A function $\Phi$ is positively $1$-homogeneous if, for any $\kappa > 0$, $\Phi(\kappa w) = \kappa \Phi(w)$.} such as the ReLU or step function.  This approach is valuable because, as noted by \citet{chizat_bach_optim}, under certain conditions (convexity of the loss and penalty functions, homogeneity of the activation function), the regularised empirical risk problem optimised via gradient descent in the infinitely small step-size limit converges to the minimiser of the corresponding problem with infinitely many particles. This allows us to use a finite number of particles $ m $ in practice while still leveraging the theoretical benefits derived from the continuous framework.

\subsection{Plan of the Paper and Notations}
In this paper, we introduce the Brownian kernel neural network (\textsc{BKerNN}), a novel model for feature learning and function estimation. Our approach combines kernel methods and neural networks using regularised empirical risk minimisation. Section~\ref{sec:fusion} presents the theoretical foundations and formulation of \textsc{BKerNN}. Section~\ref{sec:computing_the_estimator} details the practical implementation, including the optimisation algorithm and convergence insights. Section~\ref{sec:stat_analysis} provides a statistical analysis using Rademacher complexity to show high-probability convergence to the minimal risk with explicit rates. Section~\ref{sec:experiments} evaluates \textsc{BKerNN} through experiments on simulated and real data sets, comparing it with neural networks and kernel methods. Finally, Section~\ref{sec:conclusion} summarises the findings and suggests future research directions.

We use the following notations. For a positive integer $ m $, we define $[m] := \{1, \ldots, m\}$. For a $ d $-dimensional vector $ \alpha $ and $ i \in [d] $, $ \alpha_i $ denotes its $ i $-th element. For a matrix $ A $, $ \trace A $ denotes its trace when $ A $ is square, $A^{-1}$ its inverse when well defined, while $ A_{i,j} $ the element in its $ i $-th row and $ j $-th column, and $ A^\top $ its transpose. $ I_d $ is the $ d \times d $ identity matrix. We use $ \mathcal{S}^{d-1} $ to denote the unit sphere in $ \mathbb{R}^d $ for $ \|\cdot\| $ a generic norm and $ \|\cdot\|_* $ its dual norm. The $ \ell_2 $, $ \ell_1 $, and $ \ell_\infty $ norms are denoted as $ \|\cdot\|_2 $, $ \|\cdot\|_1 $, and $ \|\cdot\|_\infty $ respectively. We use $ O(\cdot) $ to denote the asymptotic behaviour of functions, indicating the order of growth. The set of probability measures on a given space $S$ is denoted by $ \mathcal{P}(S) $. A normal random variable is denoted as following the law $ \mathcal{N}(\text{mean}, \text{variance}) $. $ \mathds{1} $ is the indicator function. For two spaces~$S_1, S_2$, $S_1^{S_2}$ is the set of functions from $S_2$ to $S_1$.

\section{Neural Networks and Kernel Methods Fusion}
\label{sec:fusion}

Building on the limitations of current methods discussed in the introduction, we propose a novel architecture that integrates neural networks with kernel methods. This approach can be interpreted in two ways: as learning with a kernel that is itself learned during training, or as employing a one-hidden layer neural network where the weights from the input layer to the hidden layer are learned through gradient descent, while the weights and non-linearity from the hidden layer to the output are optimised explicitly. In this section, we introduce the custom function space we propose, revisit key properties of reproducing kernel Hilbert spaces (RKHS), and explore the connections between \textsc{BKerNN} model, kernel methods, and neural networks. Additionally, we present the various regularisation penalties we consider throughout our analysis.

\subsection{Custom Space of Functions}
\label{sec:space_of_functions}

We begin by considering the continuous setting, which mirrors the mean-field limit of over-parameterised one-hidden layer neural networks discussed in Section~\ref{sec:introduction}. 
\red{\begin{definition}[Infinite-Width Function Space]
\label{def:description_functions_f_infty}
Let
\begin{equation*}
  \mathcal{F}_\infty :=\left\{ f \mid f(\cdot) = c + \int_{\mathcal{H} \times \mathcal{S}^{d-1}} g(w^\top \cdot) {\rm d }\mu(g,w), \quad \Omega_0(f) < \infty \right\},
\end{equation*}
where $c$ is a constant in $\mathbb{R}$, $\mathcal{S}^{d-1}$ is the unit sphere for some norm $\|\cdot\|$ on $\mathbb{R}^d$ (typically either $\ell_2$ or $\ell_1$), $\mathcal{H}$ is a space of functions, and $\mu$ is a probability measure on $\mathcal{H} \times \mathcal{S}^{d-1}$. We define $\mathcal{H}$ as $\{ g: \mathbb{R} \to \mathbb{R} \mid g(0)=0, \ g $ has a weak derivative $ g^\prime, \ \int_{\mathbb{R}} (g^\prime)^2 < \infty \}$. $\mathcal{H}$ is a Hilbert space and a Sobolev space, with the inner product defined as $\langle \tilde{g}, g \rangle_\mathcal{H} = \int \tilde{g}^\prime g^\prime $. We define $\Omega_0$ as the infimum over all measures which can be used to define $f$, i.e.
\begin{equation}
\label{eq:basic_penalty}
    \Omega_0(f) := \inf_{\mu \in \mathcal{P}(\mathcal{H}\times \mathcal{S}^{d-1})} \int_{\mathcal{H}\times \mathcal{S}^{d-1}} \|g\|_{\mathcal{H}} \, {\rm d}\mu(g,w),
\end{equation}
with the infimum over $\mu$ such that $f= c + \int_{\mathcal{H} \times \mathcal{S}^{d-1}} g(w^\top \cdot) \, {\rm d}\mu(g,w)$, where $\|g\|_{\mathcal{H}}^2 := \int_{-\infty}^{+\infty} (g^\prime)^2$.
\end{definition}
\red{Note that $c$ corresponds to the value of the function $f$ evaluated at the null vector and is therefore unique, which is not necessarily the case of the measure $\mu$.} $\mathcal{F}_\infty$ is well-defined using a ``variation norm'' on couples $(g,w)$ integrated w.r.t a Borel measure on $\mathcal{H}\times \mathcal{S}^{d-1}$. This is the variation norm associated to the map $(g,w) \in \mathcal{H} \times \mathcal{S}^{d-1} \to g(w^\top \cdot)$ (which is a function from $\mathbb{R}^d$ to $\mathbb{R}$, see \citet{kurkova} and \citet[Section~9.3.2]{francis_book}.}

The function space $\mathcal{F}_\infty$ is inspired by infinite-width single hidden layer neural networks: with the addition of the intercept $c$, each function in this space can be seen as the integral of a linear part $w$ and a non-linearity $\red{g}$ over some probability distribution, as in Equation~\eqref{eq:description_neural_net} where the non-linearity is $\eta \sigma( \cdot)$. Thus here the activation functions are learnt. \\

The approximation of $\mathcal{F}_\infty$ with $m$ particles can then be obtained as follows.
\begin{definition}[Finite-Width Function Space]
Let
\label{def:final_space_of_functions}
    \begin{equation*}
    \mathcal{F}_m := \left\{f \mid  f(\cdot) = c + \frac{1}{m} \sum_{j=1}^m g_j(w_j^\top \cdot), w_j \in \mathcal{S}^{d-1}, g_j \in \mathcal{H}, c \in \mathbb{R} \right\}.
\end{equation*}
\end{definition}
Remark that $\forall m \in \mathbb{N}^*, \mathcal{F}_m \subset \mathcal{F}_\infty$, by taking the discrete probability measure uniformly supported by the particles $w_1, \ldots, w_m$.

\red{We now consider regularised empirical risk minimisation starting with the basic penalty~$\Omega_0$.  This penalty enforces the regularity of the function and, because we use penalisation with non-squared norms, limits the cardinality of the support of $\mu$ on the space $\mathcal{H}$. While this penalty is not specifically aimed at feature learning, by limiting the number of non-zero particles, it indirectly promotes feature learning to some extent. This serves as a starting point, and we introduce more targeted penalties in Section~\ref{sec:other_penalties} with a stronger feature learning behaviour. For $f \in \mathcal{F}_m$ written as in Definition~\ref{def:final_space_of_functions}, the penalty simplifies to $\Omega_0(f) = \frac{1}{m} \sum_{j=1}^m \|g_j\|_{\mathcal{H}}$. The learning objective is thus defined~as}

\begin{equation}
\label{eq:basic_ERM}
    \hat{f}_\lambda := \argmin_{f \in \mathcal{F}} \ \widehat{\mathcal{R}}(f) + \lambda \Omega(f),
\end{equation}
where $\lambda > 0$ is a regularisation parameter and $\Omega$ is currently $\Omega_0$ from Equation~\eqref{eq:basic_penalty}. The function space $\mathcal{F}$ is either $\mathcal{F}_\infty$ or $\mathcal{F}_m$. For statistical analysis in Section~\ref{sec:stat_analysis}, we consider $\mathcal{F}_\infty$, while in practice, we compute the estimator using $\mathcal{F}_m$ as discussed in Section~\ref{sec:computing_the_estimator}. The rationale for using $\mathcal{F}_m$ and expecting the statistical properties of $\mathcal{F}_\infty$ is elaborated in Section~\ref{sec:optim_guarantees}.

In the continuous setting, Equation~\eqref{eq:basic_ERM} corresponds to 
\begin{equation}
\label{eq:original_optim_infinite}
    \min_{c \in \mathbb{R}, \red{\mu \in \mathcal{P}(\mathcal{H} \times \mathcal{S}^{d-1})}} \frac{1}{n} \sum_{i=1}^n \ell\bigg(y_i, c + \int_{\red{\mathcal{H}\times \mathcal{S}^{d-1}}} g(w^\top x_i) {\rm d}\red{\mu(g,w)} \bigg) + \lambda \int_{\red{\mathcal{H} \times \mathcal{S}^{d-1}}} \|g\|_\mathcal{H} {\rm d}\red{\mu(g,w)},
\end{equation}
while in the $m$-particles setting, Equation~\eqref{eq:basic_ERM} becomes
\begin{equation}
\label{eq:original_optim}
    \min_{c \in \mathbb{R}, w_1, \ldots, w_m \in \mathcal{S}^{d-1}, g_1, \ldots, g_m \in \mathcal{H}} \frac{1}{n} \sum_{i=1}^n \ell\bigg(y_i, c + \frac{1}{m} \sum_{j=1}^m g_j(w_j^\top x_i)\bigg) + \lambda\frac{1}{m}\sum_{j=1}^m \|g_j\|_\mathcal{H},
    \end{equation}
where we clearly see a sparsity-inducing ``grouped'' penalty~\citep{yuan2006model}.

\subsection{Properties of Reproducing Kernel Hilbert Space \texorpdfstring{$\mathcal{H}$}{H} and Kernel \texorpdfstring{$k$}{k}}
In this subsection, we succinctly present some properties of reproducing kernel Hilbert spaces (RKHS) that are essential for our analysis. See \cite{aronszajn50reproducing,berlinet2011reproducing} for an introduction to RKHS. Recall that we defined the Hilbert space $\mathcal{H}$ as
\[
\mathcal{H} := \left\{ g : \mathbb{R} \to \mathbb{R} \mid g(0)=0, \int_{\mathbb{R}} (g^\prime)^2 < +\infty \right\}, 
\]
with the inner product $\langle \tilde{g}, g \rangle_{\mathcal{H}} = \int_{\mathbb{R}} \tilde{g}^\prime g^\prime$. This space is a reproducing kernel Hilbert space with the reproducing kernel  $k^{(B)}(a,b)= (|a| + |b| - |a-b|)/2 = \min(|a|, |b|)\mathds{1}_{ab>0}$. This kernel, which can be referred to as the ``Brownian'' kernel, corresponds to the covariance of the Brownian motion at times $a$ and $b$ \citep{Brownian_motion}. \red{To see that $k^{(B)}$ is the reproducing kernel of $\mathcal{H}$, it suffices to check that we have the reproducing property, i.e. that}
\[
\forall g \in \mathcal{H}, \forall a \in \mathbb{R}, g(a) = \langle g, k^{(B)}_a \rangle_{\mathcal{H}},
\]
where $k^{(B)}_a : b \in \mathbb{R} \to k^{(B)}(a,b) \in \mathbb{R}$.\footnote{\red{Indeed, we have for $g \in \mathcal{H}$, and without loss of generality $a\in \mathbb{R}^+$ that $\langle g, k^{(B)}_a \rangle_{\mathcal{H}} = \int_{\mathbb{R}} g^\prime(b) (k^{(B)}_a)^\prime(b) \ {\rm d} b= \int_0^a g^\prime(b)  \ {\rm d} b = g(a) -g(0) = g(a)$}} As a reproducing kernel, it is positive definite, meaning that for any $n \in \mathbb{N}$, $\alpha \in \mathbb{R}^n$, and $a \in \mathbb{R}^n$, we have $\sum_{i,j=1}^n \alpha_i k^{(B)}(a_i,a_j) \alpha_j \geq 0$. Additionally, we observe that \red{$\|k^{(B)}_a\|^2_{\mathcal{H}} = \langle k^{(B)}_a, k^{(B)}_a \rangle = k^{(B)}(a,a) =  |a|$} and we can obtain $\|k^{(B)}_a - k^{(B)}_b\|^2_\mathcal{H} = |a-b|$ in a similar manner. It is also noteworthy that by definition, the functions in $\mathcal{H}$ are necessarily continuous, in fact even $1/2$-Hölder continuous as we see in Lemma~\ref{lemma:caracterisation_f_infty}.

The usual Hilbert/Sobolev space is $W^{1,2}(\mathbb{R})$ (also written as $\mathcal{H}^1$) with inner product equal to $\langle f, g \rangle  = \int fg + \int f^\prime g^\prime$. This space is also an RKHS for the reproducing kernel $k^{{ \rm exp}}(a,b)= \exp(-|a-b|)$~\citep[see, e.g.,][]{williams2006gaussian}. We demonstrate that for optimisation purposes, the Brownian kernel is more advantageous due to its positive homogeneity in Section~\ref{sec:optim_guarantees}.

\subsection{Characterisation of \texorpdfstring{$\mathcal{F}_\infty$}{F infinite}}
\label{sec:charact_f_infty}
In this subsection, we discuss the properties of the function space $\mathcal{F}_\infty$ and its relationship to other relevant spaces, such as the space of functions of one-hidden-layer neural networks presented in Section~\ref{sec:introduction}. We first present the following lemma.
\begin{lemma}[Properties of Functions in $\mathcal{F}_\infty$]
\label{lemma:caracterisation_f_infty}
\red{$\mathcal{F}_\infty$ is a vector space and $\max(f(0), \Omega_0(f))$ is a norm on $ \mathcal{F}_\infty $. For $f \in \mathcal{F}_\infty$, the function $f$ is $1/2$-Hölder continuous with constant $\Omega_0(f)$, i.e., $|f(x) - f(x^\prime)| \leq \Omega_0(f) \sqrt{\|x - x^\prime\|^*}$.}
\end{lemma}

The proof can be found in Appendix~\ref{proof:lemma:caracterisation_f_infty}. \red{This lemma indicates that the space of functions $\mathcal{F}_\infty$ is contained within the space of $1/2$-Hölder continuous functions.} Recall that on a compact, all Lipschitz functions are Hölder continuous functions, indicating that the Hölder condition is less restrictive.

Now, we consider the relationship of $\mathcal{F}_\infty$ to other function spaces. Starting with the one-dimensional case, \textsc{BKerNN} reduces to kernel ridge regression with the Brownian kernel, which is also equivalent to learning with natural cubic splines  \citep[for an introduction to splines, see][]{wahba}. \red{Thus, in dimension one, we can compare precisely our function space (square-integrable derivative), to the function space associated with neural networks~\citep[][Theorem 1]{heeringa2024embeddings}, with rectified linear units (integrable second-derivative), and with sign activation function (integrable first-derivative).}

For multi-dimensional data, we use the Fourier decomposition of functions to bound the defining norms of function spaces, enabling us to make comparisons.
\begin{lemma}[Functions Spaces Included in $\mathcal{F}_\infty$]
\label{lemma:caracterisation_f_infty_2}
\red{We consider a function $f$ with support on the ball centred at $0$ with radius $R$ and norm $\|\cdot\|^*$ (Note that we do not assume that $f\in \mathcal{F}_\infty$). Assume $f$ has a Fourier transform and can be written using the inverse transform\footnote{A sufficient condition is that both $f$ and $\hat{f}$ belong to $L^1(\mathbb{R}^d)$. } as 
}
\[
f(x) = \frac{1}{(2\pi)^d} \int_{\mathbb{R}^d} \hat{f}(\omega) e^{i \omega^\top x} \, {\rm d} \omega,
\]
then, it follows that
\[
\Omega_0(f) \leq \frac{\sqrt{2R}}{(2\pi)^d } \int_{\mathbb{R}^d} |\hat{f}(\omega)| \cdot \|\omega\| \, {\rm d} \omega.
\]
Hence, if $\int_{\mathbb{R}^d} |\hat{f}(\omega)| \cdot \|\omega\| \, {\rm d} \omega < \infty$, then $f$ belongs to $\mathcal{F}_\infty$.
\end{lemma}
The proof is given in Appendix~\ref{proof:lemma:caracterisation_f_infty_2}. We remark that the condition $\int_{\mathbb{R}^d} |\hat{f}(\omega)| \cdot \|\omega\| \, {\rm d} \omega < \infty$ is a form of constraint on the regularity of the first-order derivatives \red{(as the Fourier transform of the gradient of $f$ is equal to $i \omega$ times the Fourier transform of $f$)}. 

According to \citet[Section~9.3.4]{francis_book}, the space of one-hidden-layer neural networks with ReLU activations in the mean-field limit with $\|w\|_2 =1, |b| \leq R$ can be equipped with the Banach norm $\gamma_1(f) = \int |\eta| {\rm d}\mu(\eta, w, b)$, which can be then bounded as in \red{Lemma~\ref{lemma:caracterisation_f_infty_2}} by 
\begin{equation}
\label{eq:bound_gamma_1}
    \frac{2}{(2\pi)^d R} \int_{\mathbb{R}^d} |\hat{f}(\omega)| (1 + 2R^2 \|\omega\|_2^2) \, {\rm d} \omega.
\end{equation} 
Now remark that the bound on $\Omega_0$ contains a factor $\|\red{\omega}\|$ in the integral, whereas for ReLU neural networks with $\gamma_1$ norm it is $1+2R^2\|\red{\omega}\|^2_2$. Hence, the constraint is stronger on the neural network space space, no matter what norm $\|\cdot\|$ corresponds to, suggesting that $\mathcal{F}_\infty$ is a larger space of functions. 

Also note that the bound from Equation~\eqref{eq:bound_gamma_1} can be shown to be smaller (up to a constant) than the norm defining the Sobolev space penalising derivatives up to order $s:=d/2 + 5/2$, which is $\int_{\mathbb{R}^d} |\hat{f}(\omega)|^2 (1 + 2R^2 \|\omega\|_2^2)^{s} \, {\rm d} \omega.$ \citep[Section 9.3.5]{francis_book}. This space is an RKHS because $s > d/2$, and the inequality on norms yields that the space of neural networks with ReLU activations equipped with the norm $\gamma_1$ (which is a Banach space) contains this RKHS. Another interesting remark is that if we used the norm $\gamma_2(f) = \int \eta^2 {\rm d}\mu(\eta, w, b)$ instead of~$\gamma_1$, the space that we would obtain is an RKHS and is strictly included in the one defined by~$\gamma_1$ \citep[Section 9.5.1]{francis_book}

For neural networks with step activations, i.e., $\sigma(z) = \mathds{1}_{z>0}$ in the mean-field limit, a similar bound holds for the $\gamma_1$ norm
\begin{equation}
\label{eq:step_bound}
\frac{1}{(2\pi)^d } \int_{\mathbb{R}^d} |\hat{f}(\omega)| (1 + R \|\omega\|_2) \, {\rm d} \omega. 
\end{equation}
This can be seen by applying the same proof technique as for Equation~\eqref{eq:bound_gamma_1} from \citet[Section~9.3.4]{francis_book}\footnote{The only difference being that we use $e^{iu\|w\|_2} = 1 + \int_0^R i\|w\|_2 e^{it\|w\|_2}\mathds{1}_{t\leq u}{\rm d}t $ instead of Taylor's formula, yielding $\gamma_1(x \to e^{i\omega^\top x}) \leq  1 + R\|w\|_2$.}. Learning with this space is not practically feasible due to optimisation issues as the step function is incompatible with gradient descent methods. However, the bound from Equation~\eqref{eq:step_bound} is similar to the one on $\Omega_0$, hinting that $\mathcal{F}_\infty$ is comparably large even though learning is possible with $\mathcal{F}_\infty$, as discussed in Section~\ref{sec:computing_the_estimator}. For a discussion on this topic, see \citet[Chapter 9]{francis_book} and \citet{liu}.

\subsection{Learning the Kernel or Training a Neural Network?}
\label{sec:learning_the_kernel_or_NN}
We first transform the optimisation problem before considering our setup from two different perspectives: one through kernel learning and the other through neural networks. To transform the optimisation problem, we use the representer theorem, a well-known result in RKHS theory that allows us to replace the optimisation over functions in the RKHS with optimisation over a finite weighted sum of the kernel at the data points.
\begin{lemma}[Kernel Formulation of Finite-Width]
\label{lemma:rewrite_problem}
Equation~\eqref{eq:original_optim} is equivalent to
\begin{equation}
\label{eq:final_kernel}
     \min_{w_1, \ldots, w_m \in \mathbb{R}^d, c \in \mathbb{R}, \alpha \in \mathbb{R}^n} \frac{1}{n} \sum_{i=1}^n \ell(y_i, (K\alpha)_i + c) + \frac{\lambda}{2} \alpha^\top K\alpha + \frac{\lambda}{2} \frac{1}{m} \sum_{j=1}^m  \|w_j\|,
\end{equation}
where $K = \frac{1}{m} \sum_{j=1}^m K^{(w_j)}$, and $K^{(w_j)}\in \mathbb{R}^{n \times n}$ is the kernel matrix for kernel $k^{(B)}$ and projected data $(w_j^\top x_1, \ldots, w_j^\top x_n)$, i.e., $K^{(w_j)}_{i, i^\prime} =  (|w_j^\top x_i| + |w_j^\top x_{i^\prime}| -|w_j^\top(x_i - x_{i^\prime})|)/2$. Notice that there are no constraints on the particles $(w_j)_{j \in [m]}$ to belong to the unit sphere anymore.
\end{lemma}

The proof is provided in Appendix~\ref{proof:lemma:rewrite_problem}. This lemma shows that we only need to solve a problem over finite-dimensional quantities. For computational complexity considerations, see Section~\ref{sec:optim_procedure}. We can view Equation~\eqref{eq:final_kernel} using kernels. In a classical kernel supervised learning problem with an unregularised intercept, we would have a fixed kernel matrix $K$ and consider
\begin{equation*}
    \min_{c \in \mathbb{R}, \alpha \in \mathbb{R}^n} \frac{1}{n} \sum_{i=1}^n \ell(y_i, (K\alpha)_i + c) + \frac{\lambda}{2} \alpha^\top K\alpha.
\end{equation*}

For infinitely many particles, the analogue of Lemma~\ref{lemma:rewrite_problem} is Lemma~\ref{lemma:rewrite_problem_infinite}.
\begin{lemma}[Kernel Formulation of Infinite-Width]
\label{lemma:rewrite_problem_infinite}
Equation~\eqref{eq:original_optim_infinite} is equivalent to:
\begin{equation}
\label{eq:final_kernel_infinite}
     \min_{\nu \in  \mathcal{P}(\mathbb{R}^d), c \in \mathbb{R}, \alpha \in \mathbb{R}^n} \frac{1}{n} \sum_{i=1}^n \ell(y_i, (K\alpha)_i + c) + \frac{\lambda}{2} \alpha^\top K\alpha + \frac{\lambda}{2} \int_{\mathbb{R}^d} \|w\| \, {\rm d}\nu(w),
\end{equation}
with $K = \int_{\mathbb{R}^d} K^{(w)} \, {\rm d}\nu(w)$ and and $K^{(w)}\in \mathbb{R}^{n \times n}$ is the kernel matrix for kernel $k^{(B)}$ and data $(w^\top x_1, \ldots, w^\top x_n)$. 
\end{lemma}

Notice that there is a shift in spaces, as $\nu$ is a probability distribution on $\mathbb{R}^d$, whereas $\mu$ was a probability distribution on $\red{\mathcal{H}\times\mathcal{S}^{d-1}}$. The proof is provided in Appendix~\ref{proof:lemma:rewrite_problem_infinite}.

\red{In summary, we aim to perform multiple kernel learning with a kernel parametrized by a probability measure (Lemma~\ref{lemma:rewrite_problem_infinite}), and approximate it using particles (Lemma~\ref{lemma:rewrite_problem}).}

\subsubsection{Kernel Perspective}
Lemma~\ref{lemma:rewrite_problem} shows that we are solving a regularised kernel ridge regression problem where the kernel  $\frac{1}{m}\sum_{j=1}^m (|w_j^\top x| + |w_j^\top x^\prime| - |w_j^\top(x-x^\prime)|)/2$ is also learnt through the weights $(w_j)_{j \in [m]}$, and the third term $\frac{\lambda}{2} \frac{1}{m} \sum_{j=1}^m \|w_j\|$ serves as a penalty to improve kernel learning.

The homogeneity of the kernel $k^{(B)}$ leads to well-behaved optimisation, as we discuss in Section~\ref{sec:optim_guarantees} and see in Experiment~1 in Section~\ref{sec:exp1}. The kernel matrix $K$ is indeed positively $1$-homogeneous in the particles $(w_j)_{j \in [m]}$. If we had chosen $\mathcal{H}$ to be the RKHS corresponding to the exponential kernel (or the Gaussian kernel), we would have faced the challenge of learning  the kernel $\sum_{j=1}^{m} e^{-|w_j^\top(x-x^\prime)|}$, which exhibits a complex and non-homogeneous dependency on the weights $(w_j)_{j \in [m]}$. By using the Brownian kernel instead of the exponential kernel, we only slightly change the regularisation, regularising with $\int_{\mathbb{R}} (g^\prime)^2$ instead of $\int_{\mathbb{R}} g^2 + \int_{\mathbb{R}} (g^\prime)^2$ while making the optimisation more tractable.

Compared to multiple kernel learning, \textsc{BKerNN} offers notable advantages. Multiple kernel learning (MKL) involves combining several predefined kernels, which is prone to overfitting as the number of kernels increases. Additionally, selecting the optimal kernel combination is challenging and often requires sophisticated algorithms. In contrast, \textsc{BKerNN} adapts the kernel through the learned weights $(w_j)_{j \in [m]}$, making the optimisation process simpler and more efficient, as discussed in Section~\ref{sec:computing_the_estimator}.

\subsubsection{Neural Network Perspective}
Our architecture can also be interpreted as a special type of neural network with one hidden layer. Recall that $\mathcal{F}_\infty$ is inspired by neural networks as it involves linear components $w$ followed by a non-linear part. In neural networks, this non-linear part is typically $\eta\sigma(\cdot)$, which we replaced with \red{$g(\cdot) \in \mathcal{H}$} in our setting. The functions in $\mathcal{F}_m$ are expressed similarly with the number of particles $m$ equivalent to the number of neurons in the hidden layer.

As we discuss in Section~\ref{sec:optim_procedure}, we learn the weights $(w_j)_{j \in [m]}$ through gradient descent, while the functions $(g_j)_{j \in [m]}$ are learned explicitly, leveraging a closed-form solution. This approach resonates with the work of \citet{marion2023leveraging} and \citet{loucas}. \citet{marion2023leveraging} examine a one-hidden layer neural network where the step-sizes for the inner layer are much smaller than those for the outer layer. They prove that the gradient flow converges to the optimum of the non-convex optimisation problem in a simple univariate setting and that the number of neurons does not need to be asymptotically large, which is a stronger result than the usual study of mean-field regimes or neural tangent kernel. \citet{loucas} consider learning the link function in a non-parametric way infinitely faster than the low-rank projection subspace, which resonates with our method, although they focus on Gaussian data.

We have also established that the function space $\mathcal{F}_\infty$ is more extensive than the space of neural networks with ReLU activations in Section~\ref{sec:charact_f_infty}. In Section~\ref{sec:optim_guarantees}, we demonstrate that this enlargement is compatible with efficient optimisation.

\subsection{Other Penalties}
\label{sec:other_penalties}

We now present other penalties designed to achieve different effects. The three terms in Equation~\eqref{eq:final_kernel} correspond to the empirical risk, the standard penalty from \textsc{KRR} on the RKHS norm of the function, and an extra regularisation term on the learnt kernel weights. This additional term, $\frac{\lambda}{2m}\sum_{j=1}^m\|w_j\|$, originates from the penalty $\Omega_0(f)$ in Equation~\eqref{eq:basic_penalty}. However, we can explore other penalties on $w_1, \ldots, w_m$ that induce various additional sparsity effects, even if they do not directly correspond to penalties on $f \in \mathcal{F}_m$. Let $W \in \mathbb{R}^{d \times m}$ be the matrix with $(w_1, \ldots, w_m)$ as columns, denote by $W^{(a)}$ the $a$-th row of $W$, and let $W=USV^\top$ be its singular value decomposition, with $S$ a diagonal matrix composed of $S_1,  \ldots, S_{\min(m,d)}$. Recall that $\nu$ is a probability distribution on $\mathbb{R}^d$. 
\begin{enumerate}
     \item \textbf{Basic penalty:} $\Omega_{{ \rm basic}}(w_1, \ldots, w_m) = \frac{1}{2m} \sum_{j=1}^m \|w_j\|$, which we discussed in Section~\ref{sec:space_of_functions}. In the continuous setting, it corresponds to $\frac{1}{2} \int_{\mathbb{R}^d} \|w\|{\rm  d}\nu(w)$. This penalty, which does not target any specific pattern in the data-generating mechanism, is the one for which we provide theoretical results in Section~\ref{sec:stat_analysis}. However, it does not work as well in practice as the following penalties.
    \item \textbf{Variable penalty:} $\Omega_{{ \rm variable}}(w_1, \ldots, w_m) = \frac{1}{2} \sum_{a=1}^{\red{d}} \big( \frac{1}{m}\sum_{j=1}^m (w_j)_a^2\big)^{1/2}$, which is also equal to $ \frac{1}{2\sqrt{m}} \sum_{a=1}^{\min(m,d)} \|W^{(a)}\|_2$. This penalty, inspired by the group Lasso \citep{yuan2006model}, is designed for variable selection, pushing quantities $\|W^{(a)}\|_2$ towards zero, thus encouraging dependence on a few variables. In the continuous setting, it corresponds to $ \frac{1}{2} \sum_{a=1}^{\min(m,d)} \big(\int_{\mathbb{R}^d} |w_a|^2 {\rm d}\nu(w) \big)^{1/2} $.  
    \item \textbf{Feature penalty:} $\Omega_{{ \rm feature}}(w_1, \ldots, w_m) = \frac{1}{2} \trace\big( \big( \frac{1}{m} \sum_{j=1}^m w_j w_j^\top\big)^{1/2}\big)$, which is also equal to $\frac{1}{2} \sum_{a=1}^{\min(m,d)} \frac{S_a}{\sqrt{m}}$ and to the nuclear norm of  $W$ divided by $2\sqrt{m}$. It is used for feature learning as it is a convex relaxation of the rank, encouraging $W$ to have low rank and thus dependence on only a few linear transformations of the data. Regularisation using the nuclear norm in the context of feature learning is well-established in the literature, as demonstrated by \citet{argyriou2008convex}. It corresponds to $ \frac{1}{2} \trace\big( \big(\int_{\mathbb{R}^d} w w^\top {\rm d}\nu(w) \big)^{1/2}\big) $ in the continuous setting. 
    \item \textbf{Concave variable penalty:} The concave version of the penalty for variable selection, $\Omega_{{ \rm concave \ variable}}(w_1, \ldots, w_m) = \frac{1}{2s} \sum_{a=1}^{\red{d}} \log\big( 1 + \frac{s}{\sqrt{m}} \|W^{(a)}\|_2 \big)$, with $s\geq0$. The appeal of the added concavity is discussed below. In the continuous setting, it corresponds to $\frac{1}{2s} \sum_{a=1}^d  \log\big(1+ s\int_{\mathbb{R}^d} (w_a)^2{\rm d}\nu(w)\big)^{1/2} \big)$.
    \item  \textbf{Concave feature penalty:} The concave version of the penalty intended for feature learning, $\Omega_{{ \rm concave \ feature}}(w_1, \ldots, w_m) = \frac{1}{2s} \sum_{a=1}^{\min(m,d)} \log\big( 1 + \frac{s}{\sqrt{m}} S_a \big)$ for feature selection, with $s\geq0$. The appeal of the added concavity is discussed below. In the continuous setting it corresponds to $\frac{1}{2s} \sum_{a=1}^{d} \log\big( 1 + s \big( \big(\int_{\mathbb{R}^d} w w^\top {\rm d}\nu(w) \big)^{1/2}\big)_{a,a} \big)$.
\end{enumerate}

The first penalty is convex in both $\nu$ and $W$, making it straightforward to optimise. The second and third penalties, while not convex in $\nu$, are convex in $W$ due to the presence of squared and square root terms on the components of $W$, easing optimisation in the $m$ particles setting. The fourth and fifth penalties are neither convex in $\nu$ nor $W$, instead, they are concave in $W$. As $s$ approaches zero, these penalties revert to their non-concave versions. Convex penalties, while easier to handle, can be detrimental by diminishing relevant variables or features to achieve sparsity. Mitigating this effect can involve retraining with the selected variables/features or employing concave penalties, which is the choice we made here. Although concave penalties are more complex to analyse, they often yield better performance because they drive the solution towards the boundary, promoting sparsity \citep{fan2001variable,livreviolet}. We discuss the impact of the choice of regularisation in Experiment~3 in Section~\ref{sec:exp2&3}. 

\red{Note that the link with infinite-dimensional normed function spaces and the theoretical analysis of Section~\ref{sec:stat_analysis}. only applies to the first ``basic'' penalty. Extending the analysis of concave penalties from finite dimension~\citep{fan2001variable,zhang2010nearly} to infinite dimension remains an open problem.}

\section{Computing the Estimator}
\label{sec:computing_the_estimator}
In this section, we detail the process of computing the estimator for each of the penalties presented in Section~\ref{sec:other_penalties}. We then discuss the importance of the homogeneity of the Brownian kernel and how the optimisation with particles relates to the continuous setting.

\subsection{Optimisation Procedure}
\label{sec:optim_procedure}
In this section, we focus on the square loss $\ell(y, y^\prime) = \frac{1}{2} (y - y^\prime)^2$, which allows for explicit computations. However, the method can be extended to other loss functions using gradient-based techniques, \citep[see][Chapter 5]{francis_book}. Recalling Equation~\eqref{eq:final_kernel} and the penalties described in Section~\ref{sec:other_penalties}, the optimisation problem we aim to solve is
\begin{equation}
\label{eq:optim_square_loss}
     \min_{w_1, \ldots, w_m \in \mathbb{R}^d, c \in \mathbb{R}, \alpha \in \mathbb{R}^n} \frac{1}{2n} \|Y - K\alpha - c\mathds{1}_n\|_2^2 + \frac{\lambda}{2} \alpha^\top K\alpha + \lambda \Omega_{{ \rm weights}}(w_1, \ldots, w_m),
\end{equation}
where $K = \frac{1}{m} \sum_{j=1}^m K^{(w_j)}$ and $\Omega_{{ \rm weights}}$ represents any of the penalties from Section~\ref{sec:other_penalties}.

To solve this problem, we alternate between \red{exact} minimisation with respect to $\alpha$ and $c$, which is done in closed-form, and \red{one improvement step for the} minimisation with respect to $w_1, \ldots, w_m$ which is done using one step of proximal gradient descent. \red{This corresponds to minimizing with respect to the $w_1, \ldots, w_m$ the infimum with respect to $c$ and $\alpha$.}

\subsubsection{Fixed Particles \texorpdfstring{$w_1, \ldots, w_m$}{w1, ..., wm}}
When the weights $w_1, \ldots, w_m$ are fixed, the kernel matrix $K$ is also fixed, allowing us to find the solution for the constant $c$ and the coefficients $\alpha$ in closed-form. By centring both the kernel matrix and the response $Y$, we transform the problem into a classical kernel ridge regression problem, for which explicit solutions are well-known.
\begin{lemma}[Optimisation for Fixed Particles]
\label{lemma:fixed_K}
For fixed $w_1, \ldots, w_m$ and hence a fixed~$K$, define
\begin{align*}
    G(w_1, \ldots, w_m) := \min_{\alpha \in \mathbb{R}^n, c \in \mathbb{R}} \frac{1}{2n} \| Y - K\alpha - c\mathds{1}_n \|_2^2 + \frac{\lambda}{2} \alpha^\top K\alpha.
\end{align*}
The optimisation problem defining $G$ is solved by
\[
\alpha = (\tilde{K} + n\lambda I)^{-1} \tilde{Y} \quad { \rm and} \quad c = \frac{\mathds{1}^\top Y}{n} - \frac{\mathds{1}^\top K\alpha}{n},
\]
where $\tilde{K} := \Pi K \Pi$ and $\tilde{Y} := Y - \frac{\mathds{1}\mathds{1}^\top Y}{n}$, with $\Pi = I - \frac{\mathds{1}\mathds{1}^\top}{n}$ being the centring matrix. The objective value is then
\[
G(w_1, \ldots, w_m) = \frac{\lambda}{2} \tilde{Y}^\top (\tilde{K} + \lambda n I)^{-1} \tilde{Y}.
\]
\end{lemma}

The proof is provided in Appendix~\ref{proof:lemma:fixed_K}. Lemma~\ref{lemma:fixed_K} allows us to optimise $\alpha$ and $c$ explicitly during the optimisation process. The complexity of this step is $O(n^3 + n^2d)$, which can be challenging when the sample size $n$ is large, a common drawback of kernel methods. However, techniques like the Nyström method \citep{nystrom}, which approximates the kernel matrix, can help mitigate this issue. Alternatively, we could use gradient descent techniques, but as shown in \citet{marion2023leveraging}, it may be beneficial to learn the weights from the hidden layer to the output layer (corresponding to learning $g_1, \ldots, g_m$ and hence $\alpha$) with a much larger step-size than the weights from the input layer to the hidden layer (corresponding to learning $w_1, \ldots, w_m$). Learning $\alpha$ and $c$ explicitly represents the limit of this two-timescale regime.

\subsubsection{Proximal Step to Optimise the Weights \texorpdfstring{$w_1, \ldots, w_m$}{w1, ..., wm}}
Next, we focus on optimising $w_1, \ldots, w_m$ while keeping $c$ and $\alpha$ fixed. The goal is to solve
\begin{equation}
\label{eq:fixed_c_alpha}
    \min_{w_1, \ldots, w_m \in \mathbb{R}^{d}} G(w_1, \ldots, w_m) + \lambda \Omega_{{ \rm weights}}(w_1, \ldots, w_m),
\end{equation}
where the dependence on $(w_j)_{j \in [m]}$ in the first term is through the kernel matrix $K$. Note that $G$ is convex in $K$ but not in $w_1, \ldots, w_m$. Additionally, $G$ is differentiable almost everywhere, except where $w_j^\top(x_i - x_{i^\prime})$ for some $j \in [m], i \neq i^\prime \in [n]$. However, standard practice assumes that these non-differentiabilities average out with many data points. Meanwhile, the penalties $\Omega_{{ \rm weights}}$ are not differentiable at certain fixed points, independently of the data, similarly to the Lasso penalty. Therefore, we use proximal gradient descent to solve Equation~\eqref{eq:fixed_c_alpha}. With a step-size $\gamma > 0$, this involves minimising
\begin{equation*}
    \sum_{j=1}^m \frac{\partial G}{\partial w_j}(w^{\rm old})^\top(w_j - w_j^{\rm old}) + \frac{1}{2\gamma} \sum_{j=1}^m \|w_j - w_j^{\rm old}\|_2^2 + \lambda \Omega_{{ \rm weights}}(w_1, \ldots, w_m), 
\end{equation*}
over $w_1, \ldots, w_m \in \mathbb{R}^d$. This corresponds to the simultaneous proximal gradient descent steps $w_j \gets {\rm prox}_{\lambda \gamma \Omega}(w_j - \gamma \frac{\partial G}{\partial w_j})$. We therefore compute the gradient and the proximal operator. For the gradient, we have the following lemma.
\begin{lemma}[Gradient of $G$]
\label{lemma:derivative_of_G}
Let $j \in [m]$, then 
\begin{equation*}
\frac{\partial G}{\partial w_j} = \frac{\lambda}{4} \frac{1}{m} \sum_{i, i^\prime = 1}^n z_i z_{i^\prime} \sign(w_j^\top(x_i - x_{i^\prime}))(x_i - x_{i^\prime}),
\end{equation*}
where $z = (\tilde{K} + n\lambda I)^{-1} \tilde{Y}$.
\end{lemma}
The proof is in Appendix~\ref{proof:lemma:derivative_of_G}. Note that $G$ is not differentiable around $0$, which is also the case of common activation functions in neural networks such as the ReLU, but this is not an issue in practice.

Next, we compute the proximal operator for the described penalties. Recall the definition of the proximal operator
\begin{equation*}
    {\rm prox}_{\Omega}(W) = \argmin_{(u_1, \ldots, u_m) \in \mathbb{R}^{d \times m}} \frac{1}{2} \sum_{j=1}^m \|w_j - u_j\|_2^2 + \Omega(u_1, \ldots, u_m).
\end{equation*}
We use $W \in \mathbb{R}^{d\times m}$ and $(w_1, \ldots, w_m)$ interchangeably, with $W = USV^\top$ (SVD). We denote the rows of $W$ by $W^{(a)}$ as before. The following lemma provides the proximal operators.
\newline
\begin{lemma}[Proximal Operators] We describe the proximal operators.
\label{lemma:prox}
\begin{enumerate}
\item For $\Omega_{\rm basic}(W) = \frac{1}{2m} \sum_{j=1}^m \|w_j\|$, then  $\big({\rm prox}_{\lambda \gamma \Omega}(W)\big)_j = \big( 1 - \frac{\lambda \gamma}{2m} \frac{1}{\|w_j\|} \big)_+ w_j$.
\item For $\Omega_{{ \rm variable}}(W) = \frac{1}{2\sqrt{m}} \sum_{a=1}^d \|W^{(a)}\|_2$, $({\rm prox}_{\lambda \gamma \Omega}(W))^{(a)} = \big( 1 - \frac{\lambda \gamma}{2\sqrt{m}} \frac{1}{\|W^{(a)}\|_2} \big)_+ W^{(a)}$.
\item For $\Omega_{{ \rm feature}}(W) = \frac{1}{2} \mathrm{trace}\big( \big( \frac{1}{m} \sum_{j=1}^m w_j w_j^\top \big)^{1/2} \big)$, then we have ${\rm prox}_{\lambda \gamma \Omega}(W) = U\tilde{S}V^\top$ with $\tilde{S} = \big(1 - \frac{\lambda \gamma}{2\sqrt{m} |S|}\big)_+ S$.
\item For $\Omega_{{ \rm concave \ variable}}(W) = \frac{1}{2s} \sum_{a=1}^d \log\big( 1 + \frac{s}{\sqrt{m}} \|W^{(a)}\|_2\big)$, then  with $c$  obtained from $(\|W^{(a)}\|_2)_{a \in [d]}$ by an explicit (albeit lengthy) formula $({\rm prox}_{\lambda \gamma \Omega}(W))^{(a)} = c W^{(a)}$.
\item For $\Omega_{{ \rm concave \ feature}}(W) = \frac{1}{2s} \sum_{a=1}^d \log\big( 1 + \frac{s}{\sqrt{m}} S_a \big)$, then with $c$ which obtained from $S$ by an explicit (albeit lengthy) formula ${\rm prox}_{\lambda \gamma \Omega}(W) = U\tilde{S}V^\top$ with $\tilde{S} = c S$.
\end{enumerate}
\end{lemma}
The proof is in Appendix~\ref{proof:lemma:prox}. Each proximal step is easy to compute using the explicit formulas above, with complexities $O(md)$ for the basic, variable, and concave variable cases, and $O(md\min(m,d))$ for the feature and concave feature cases, due to the SVD computation.

\subsubsection{Algorithm Pseudocode}
\label{sec:pseudocode}
We now have all the components necessary to provide the pseudocode (Algorithm~\ref{algo:one_and_only}) of the proposed method \textsc{BKerNN}, specifically for the square loss. For other losses, the main difference is that $\alpha$ and $c$ might not be solvable in closed-form and would need to be computed through alternative methods such as gradient descent.

\begin{algorithm}
\label{algo:one_and_only}
\caption{\textsc{BKerNN} pseudocode.}
\KwData{$X, Y, m, \lambda, \gamma, \Omega_{{ \rm weights}}$}
\KwResult{$w_1, \ldots, w_m, c, \alpha$}
$W=(w_1, \ldots, w_m) \in \mathbb{R}^{d \times m} \gets \big(\mathcal{N}(0, 1/d)\big)^{d\times m}$\;
\For{$i \in [n_{\rm iter}]$}{
    Compute $K$\;
    $\alpha \gets  (\Tilde{K} + n\lambda I)^{-1}\Tilde{Y}, c \gets \frac{\mathds{1}^\top Y}{n} - \frac{\mathds{1}^\top}{n}K\alpha$\;
    Compute $\frac{\partial G}{\partial W}$\;
    $\gamma \gets \gamma \times 1.5$\;
    \While{$G({\rm prox}_{\lambda\gamma\Omega}(W - \gamma \frac{\partial G}{\partial W}) > G(W) - \gamma \frac{\partial G}{\partial W} \cdot G_\gamma(W) + \frac{\gamma}{2}\|G_\gamma(W)\|_2^2 $}{
        $\gamma \gets \gamma / 2$\;
    }
    $W \gets {\rm prox}_{\lambda\gamma\Omega}(W - \gamma \frac{\partial G}{\partial W})$\;
}
\end{algorithm}

To select the step-size $\gamma$ for the proximal gradient descent step appropriately, we use a backtracking line search, assuming $G$ is locally Lipschitz. Starting with the previous step-size, we multiply it by 1.5. If the backtracking condition is not satisfied, we divide $\gamma$ by 2 and repeat. The backtracking condition is that $G({\rm prox}_{\lambda\gamma\Omega}(W - \gamma \frac{\partial G}{\partial W}))$ should be smaller than $G(W) - \gamma \frac{\partial G}{\partial W} \cdot G_\gamma(W) + \frac{\gamma}{2}\|G_\gamma(W)\|_2^2$, where $G_\gamma(W) = \big(W - {\rm prox}_{\lambda\gamma\Omega}(W - \gamma \frac{\partial G}{\partial W})\big) / \gamma$. This method was taken from \citet{beck2017first}.

With the outputted $w_1, \ldots, w_m$, $c$, and $\alpha$ from the algorithm, the estimator is the function $\hat{f}_\lambda$ defined as $\hat{f}_\lambda(x) = c + \sum_{i=1}^n \alpha_i \sum_{j=1}^m \frac{1}{m} (|w_j^\top x_i| + |w_j^\top x| - |w_j^\top (x - x_i)|) / 2$. This formulation enables us to perform predictions on new data points and facilitates the extraction of meaningful linear features through the learned weights $(w_j)_{j \in [m]}$. Remark that we do not take into account the optimisation error in the rest of the paper.

\subsection{Convergence Guarantees on Optimisation Procedure}
\label{sec:optim_guarantees}
In this section, we discuss the convergence properties of the optimisation procedure. Although we do not provide a formal proof due to differentiability issues, we highlight the importance of the homogeneity of the Brownian kernel and present arguments suggesting the robustness of the optimisation process. \red{However, no formal results are presented because of the lack of differentiability.}

We aim to apply the insights from \citet{chizat_bach_optim} and \citet{chizat2018global}, which state that under certain assumptions, in the limit of infinitely many particles and an infinitely small step-size, gradient descent optimisation converges to the global optimum of the infinitely-many particles problem. Key assumptions include convexity with respect to the probability distribution in the and homogeneity of a specific quantity $\Psi$, which we define below. We reformulate our problem in line with \citet{chizat_bach_optim}.

Considering the square loss with the basic penalty $\Omega_{{ \rm basic}}$, the optimisation problem with~$m$ particles from Equation~\eqref{eq:optim_square_loss} can be rewritten as
\begin{align*}
    &\min_{w_1, \ldots, w_m \in \mathbb{R}^d} \bigg\{ \inf_{\alpha \in \mathbb{R}^n, c \in \mathbb{R}} \frac{1}{2n} \|Y - K\alpha - c\mathds{1}_n\|_2^2 + \frac{\lambda}{2} \alpha^\top K\alpha + \frac{\lambda}{2m} \sum_{j=1}^m \|w_j\| \bigg\} \\
    &= \min_{w_1, \ldots, w_m \in \mathbb{R}^d} \bigg\{ \frac{\lambda}{2} \Tilde{Y}^\top \big( \Tilde{K} + \lambda n I \big)^{-1} \Tilde{Y} + \frac{\lambda}{2m} \sum_{j=1}^m \|w_j\| \bigg\},
\end{align*}
where $K = \frac{1}{m} \sum_{j=1}^m K^{(w_j)}$ is the final kernel matrix, $K^{(w)} \in \mathbb{R}^{n \times n}$ is the kernel matrix for kernel $k^{(B)}$ with projected data $(w^\top x_1, \ldots, w^\top x_n)$, $\Pi = I_n - \mathds{1}_n \mathds{1}_n^\top$ is the centring matrix, while $\Tilde{Y} = \Pi Y$ is the centred output, and $\Tilde{K} = \Pi K \Pi$ is the centred kernel matrix. We solve this using proximal gradient descent. For the continuous case, the problem is
\begin{align*}
    &\min_{\nu \in \mathcal{P}(\mathbb{R}^d)} \left( \inf_{\alpha \in \mathbb{R}^n, c \in \mathbb{R}} \frac{1}{2n} \|Y - K\alpha - c\mathds{1}_n\|_2^2 + \frac{\lambda}{2} \alpha^\top K\alpha + \frac{\lambda}{2} \int_{\mathbb{R}^d} \|w\| \, {\rm d}\nu(w) \right) \\
    &= \min_{\nu \in \mathcal{P}(\mathbb{R}^d)} \left( \frac{\lambda}{2} \Tilde{Y}^\top \big( \Tilde{K} + \lambda n I \big)^{-1} \Tilde{Y} + \frac{\lambda}{2} \int_{\mathbb{R}^d} \|w\| \, {\rm d}\nu(w) \right),
\end{align*}
where $K = \int_{\mathbb{R}^d} K^{(w)} \, {\rm d}\nu(w)$ and $\nu$ is a probability measure on $\mathbb{R}^d$. 

In both cases we minimise $F(\nu)$ (defined right below) over $\mathcal{P}(\mathbb{R}^d)$ for the continuous case and over $\mathcal{P}_m(\mathbb{R}^d)$, which is the set of probability distributions anchored at $m$ points on $\mathbb{R}^d$, in the $m$-particles case. $F$ is defined as
\begin{equation*}
    F(\nu) := Q\left(\int_{\mathbb{R}^d} \Psi(w) \, {\rm d}\nu(w)\right),
\end{equation*}
where $Q: \mathbb{R}^{n\times n} \times \mathbb{R} \to \mathbb{R}$, $Q(K, c^\prime) = \frac{\lambda}{2} \Tilde{Y}^\top \big(\Tilde{K} + \lambda n I \big)^{-1} \Tilde{Y} + \frac{\lambda}{2} c^\prime$, and $\Psi: \mathbb{R}^d \to \mathbb{R}^{n\times n} \times \mathbb{R}$, $\Psi(w) = (K^{(w)}, \|w\|)$. Note that $\Psi$ is indeed positively $1$-homogeneous, as the necessary condition  $\forall w \in \mathbb{R}^d, \forall \kappa > 0$, $\Psi(\kappa w) = \kappa \Psi(w)$ is verified. Moreover, $Q$ is convex in $\nu$, indicating the optimisation is well-posed (while we perform computations for the square loss, we would also obtain a convex function for any convex loss).

Our method employs proximal gradient descent instead of basic gradient descent, which is acceptable as both methods approximate the differential equation arising in the infinitesimal step-size limit. Gradient descent is an explicit method, whereas proximal gradient descent combines implicit and explicit updates \citep{Süli_Mayers_2003}. Moreover, it allows to deal efficiently with the non-smoothness of the sparsity-inducing penalties (no additional cost and improved convergence behaviour).

While our framework aligns with that of \citet{chizat_bach_optim}, we cannot directly apply their results due to the non-differentiability of $\Psi$ around zero, a common issue in such analyses. Despite this, our setup meets the crucial assumptions of convexity in $Q$ and the homogeneity of $\Psi$. See Experiment~1 in Section~\ref{sec:exp1} for a numerical evaluation of the practical significance of the homogeneity assumption. \red{Obtaining formal results despite the lack of differentiability remains an open problem.}

\section{Statistical Analysis}
\label{sec:stat_analysis}
In this section our objective is to obtain high-probability bounds on the expected risk of the \textsc{BKerNN} estimator to understand its generalisation capabilities. To achieve this, we bound the Gaussian complexity (a similar concept to the Rademacher complexity,) of the sets $\{ f \in \mathcal{F}_\infty \mid \max(f(0), \Omega_0(f)) \leq D \}$ for $D > 0$. Recall that $\mathcal{F}_\infty$ is defined in Definition~\ref{def:description_functions_f_infty}. We begin by introducing the  Gaussian complexity in Definition~\ref{def:gaussian_complexity}, followed by Lemma~\ref{lemma:rademacher_without_integral}, which is used to simplify the quantities for subsequent bounding. We then bound the Gaussian complexities using two distinct techniques in Sections~\ref{sec:dim_dependant} and~\ref{sec:dim_independant}. The first technique yields a dimension-dependent bound with better complexity in sample size, while the second provides a dimension-independent bound. Finally, in Section~\ref{sec:bound_true_risk}, we derive the high-probability bound on the expected risk of \textsc{BKerNN} with explicit rates for data with subgaussian square-rooted norm, using an extension of McDiarmid's inequality from \citet{meir_zhang}, before detailing the data-dependent quantities of the rates. All of these results require few assumptions on the problem, and on the data-generating mechanism in particular.

While our method resembles multiple kernel learning, the theoretical results from \textsc{MKL}, which are often related to Rademacher chaos \citep[e.g., ][]{lanckriet, rademacher_chaos} are not directly applicable. This is because, in our approach, the learned weights are multi-dimensional and embedded within the kernel, rather than being simple scalar weights used to combine predefined kernels. Thus, the unique structure of our model requires different theoretical considerations.

\subsection{Gaussian Complexity}
\label{sec:gaussian_complex}
Recall that the estimator \textsc{BKerNN} is defined as
\begin{equation*}
 \hat{f}_\lambda := \argmin_{f \in \mathcal{F}} \frac{1}{n} \sum_{i=1}^n \ell(y_i, f(x_i)) + \lambda \Omega(f)
\end{equation*}
where $\mathcal{F}$ is $\mathcal{F}_m := \{f \mid f(x)= c + \frac{1}{m}\sum_{j=1}^m g_j(w_j^\top x), w_j \in \mathcal{S}^{d-1}, g_j \in \mathcal{H}, c \in \mathbb{R}\}$ in practice for optimisation and \red{$\mathcal{F}_\infty := \{f \mid f(x) = c + \int_{\mathcal{H} \times \mathcal{S}^{d-1}} g(w^\top x) \, {\rm d} \mu(g,w), \mu \in \mathcal{P}(\mathcal{H} \times \mathcal{S}^{d-1}), c \in \mathbb{R}, \red{\Omega_0(f) < \infty}\}$} for statistical analysis. Although we considered various penalties in Section~\ref{sec:other_penalties}, here we focus on $f \in \mathcal{F}_\infty$ with $\Omega(f) = \max(\Omega_0(f), c)$, where $\Omega_0(f)$ was defined as
\begin{equation*}
    \Omega_0(f) = \inf_{\red{\mu \in \mathcal{P}(\mathcal{H}\times \mathcal{S}^{d-1})}} \red{\int_{\mathcal{H}\times \mathcal{S}^{d-1}} \|g\|_{\mathcal{H}} \, {\rm d}\mu(g,w),}
\end{equation*}
such that $f(\cdot) = c + \red{\int_{\mathcal{H}\times \mathcal{S}^{d-1}} g(w^\top \cdot){\rm d}\mu(g,w)}$ and corresponds to the basic penalty for  $\Omega_{\rm weights}$. This is made possible through a well-defined mean-field limit; we leave the other penalties for future work.

We now introduce the concept of Gaussian complexity \citep[for more details, see][]{Bartlett}.
\begin{definition}[Gaussian Complexity]
\label{def:gaussian_complexity}
The Gaussian complexity of a set of functions $\mathcal{G}$ is defined as
\begin{equation*}
     G_n(\mathcal{G}) := \mathbb{E}_{\varepsilon, \mathcal{D}_n} \bigg( \sup_{f \in \mathcal{G}} \frac{1}{n}\sum_{i=1}^n \varepsilon_i f(x_i) \bigg),
\end{equation*} 
with $\varepsilon$ a centred Gaussian vector with identity covariance matrix, and $\mathcal{D}_n := (x_1, \ldots, x_n)$ is the data set consisting of i.i.d.~samples drawn from the distribution of the random variable~$X$. Note that it only contains the covariates, not the response.
\end{definition}

We aim to bound $ G_n(\{ f \in \mathcal{F}_\infty \mid \Omega(f) \leq D \}) $ for some $ D > 0 $. The discussion on the Gaussian complexity of the space $\mathcal{F}_\infty$ would yield the same bounds if $\mathcal{F}_m$ were considered instead. However, since we demonstrated  in Section~\ref{sec:optim_guarantees} that optimisation in $\mathcal{F}_m$ and optimisation in $\mathcal{F}_\infty$ are closely related, we focus exclusively on $\mathcal{F}_\infty$ in this section.

First, we note that we can study the Gaussian complexity of a simpler class of functions, as indicated by the following lemma, which allows us to deal with the constant and remove the integral present in the definition of $\mathcal{F}_\infty$. 
\begin{lemma}[Simplification of Gaussian Complexity]
\label{lemma:rademacher_without_integral}
Let $ D > 0 $. Then,
\begin{equation*}
    G_n(\{ f \in \mathcal{F}_\infty \mid \Omega(f) \leq D \}) \leq D\left(\frac{1}{\sqrt{n}} + G_n\right),
\end{equation*}
with $ G_n := G_n(\{ f \mid f(\cdot) = g(w^\top \cdot), \|g\|_\mathcal{H} \leq 1, w \in \mathcal{S}^{d-1} \}) $.
\end{lemma}

The proof can be found in Appendix~\ref{proof:lemma:rademacher_without_integral}. We now need to bound $ G_n $, which we approach in two different ways. First, in Section~\ref{sec:dim_dependant}, we use covering balls on the sphere $\mathcal{S}^{d-1}$, resulting in a dimension-dependent bound. Then, in Section~\ref{sec:dim_independant}, we approximate functions in $\mathcal{F}_\infty$ by Lipschitz functions, before using a covering argument, leading to a dimension-independent bound at the cost of worst dependency in the sample size $ n $. With these bounds on $ G_n $, we will derive results on the expected risk of the \textsc{BKerNN} estimator, providing explicit rates depending on the upper bounds of $ G_n $, without exponential dependence on dimension.

\subsubsection{Dimension-Dependent Bound}
\label{sec:dim_dependant}
First, we note that the supremum over the functions $ g $ with $\|g\|_\mathcal{H} \leq 1$ can be obtained in closed-form (see Lemma~\ref{lemma:best_g} in Appendix~\ref{sec:lemma_best_g}). This reduces the problem to considering the expectation of a supremum over the sphere, which we address using a covering of $\mathcal{S}^{d-1}$.

\begin{theorem}[Dimension-Dependent Bound] 
\label{theo:dim}
We have
\begin{equation*}
G_n \leq 8\sqrt{\frac{d}{n}} \sqrt{\log(n+1)} \sqrt{ \mathbb{E}_X \|X\|^*},
\end{equation*}
where $\|\cdot \|^*$ is the dual norm of $\|\cdot \|$. Recall that $\|\cdot\|$ defines the sphere $\mathcal{S}^{d-1}$.
\end{theorem}

The bound on the Gaussian complexity obtained here is dimension-dependent due to the covering of the unit ball in $\mathbb{R}^d$, but it has a favourable dependency on the sample size. For ease of exposition, we have replaced the original factor $\sqrt{\log(1+ n/(2d)) +1/(2d)} + 1$ with~$8\sqrt{\log(n+1)}$. Recall that $\|\cdot\|^*=\|\cdot\|_2$ for the $\ell_2$ sphere and $\|\cdot\|^*=\|\cdot\|_\infty$ for the $\ell_1$ sphere.  Note that the dependency on the data distribution is explicit and can be easily bounded in different data-generating mechanisms, as discussed in Lemma~\ref{lemma:data_dependent} at the end of Section~\ref{sec:stat_analysis}.
\begin{proof}[Proof of Theorem~\ref{theo:dim}]
\red{First, recall the definition of $G_n$
\begin{equation*}
    G_n = \mathbb{E}_{\varepsilon, \mathcal{D}_n} \bigg( \sup_{f  = g(w^\top \cdot), \|g\|_\mathcal{H} \leq 1, w \in \mathcal{S}^{d -1} } \frac{1}{n}\sum_{i=1}^n \varepsilon_i f(x_i) \bigg),
\end{equation*}
which we can rewrite as
\begin{equation*}
    G_n = \mathbb{E}_{\varepsilon, \mathcal{D}_n} \bigg( \sup_{ \|g\|_\mathcal{H} \leq 1, w \in \mathcal{S}^{d -1} } \frac{1}{n}\sum_{i=1}^n \varepsilon_i g(w^\top x_i) \bigg),
\end{equation*}
and finally by splitting the supremum as everything is finite
\begin{equation*}
    G_n = \mathbb{E}_{\varepsilon, \mathcal{D}_n} \bigg( \sup_{ w \in \mathcal{S}^{d -1} }  \sup_{\|g\|_\mathcal{H} \leq 1}\frac{1}{n}\sum_{i=1}^n \varepsilon_i g(w^\top x_i) \bigg).
\end{equation*}
}
Using Lemma~\ref{lemma:best_g}, we then have
\begin{equation*}
    G_n = \mathbb{E}_{\varepsilon, \mathcal{D}_n} \left( \sup_{w \in \mathcal{S}^{d-1}}\frac{\sqrt{\varepsilon^\top K^{(w)} \varepsilon}}{n}\right),
\end{equation*}
where $ K^{(w)} $ is the kernel matrix for the Brownian kernel with data $(w^\top x_1, \ldots, w^\top x_n)$.

We bound the supremum inside of the expectation using covering balls. Let $ M \in \mathbb{N}^* $ and $ \mathcal{W}^M $ be such that $ \forall w \in \mathcal{S}^{d-1}, \exists \Tilde{w} \in \mathcal{W}^M  \subset \mathcal{S}^{d-1}$  such that $ \|w - \tilde{w}\| \leq \zeta $, i.e., we have a $\zeta$-covering of the sphere with its own norm in $ d $ dimensions, \red{with $M$ the covering number}. Fix $ w \in \mathcal{S}^{d-1} $ and $ \Tilde{w} $ such that $ \|w-\Tilde{w}\| \leq \zeta $. We then have
\begin{align*}
    | \sqrt{\varepsilon^\top K^{(w)} \varepsilon} - \sqrt{\varepsilon^\top K^{(\tilde{w})} \varepsilon} | &= \left| \bigg\|\sum_{i=1}^n \varepsilon_i k_{w^\top x_i} \bigg\|_\mathcal{H} - \bigg\|\sum_{i=1}^n \varepsilon_i k_{\tilde{w}^\top x_i} \bigg\|_\mathcal{H} \right| \\
    &\leq \bigg\|\sum_{i=1}^n \varepsilon_i (k_{w^\top x_i} - k_{\Tilde{w}^\top x_i}) \bigg\|_\mathcal{H}     \leq \sum_{i=1}^n |\varepsilon_i| \cdot \|k_{w^\top x_i} - k_{\Tilde{w}^\top x_i} \|_\mathcal{H} \\
    &= \sum_{i=1}^n |\varepsilon_i|  \sqrt{|w^\top x_i - \Tilde{w}^\top x_i|} \leq \sum_{i=1}^n |\varepsilon_i|  \sqrt{\|w - \Tilde{w}\|\| x_i\|^*} \\
    &\leq \sqrt{\|w - \Tilde{w}\|} \sum_{i=1}^n |\varepsilon_i| \sqrt{\|x_i\|^*} \leq  \zeta^{1/2} \sum_{i=1}^n |\varepsilon_i| \sqrt{\|x_i\|^*}.
\end{align*}
Next, we get
$$
    \sqrt{\varepsilon^\top K^{(w)} \varepsilon} = \sqrt{\varepsilon^\top K^{(\tilde{w})} \varepsilon} + \sqrt{\varepsilon^\top K^{(w)} \varepsilon} - \sqrt{\varepsilon^\top K^{(\tilde{w})} \varepsilon}  
    \leq  \sqrt{\varepsilon^\top K^{(\tilde{w})} \varepsilon} + \zeta^{1/2} \sum_{i=1}^n |\varepsilon_i| \sqrt{\|x_i\|^*}.
$$
Taking the supremum and dividing by the sample size $ n $,
\begin{equation}
\label{eq:2sup}
   \sup_{w \in \mathcal{S}^{d-1}} \frac{\sqrt{\varepsilon^\top K^{(w)} \varepsilon}}{n} \leq \sup_{\tilde{w} \in \mathcal{W}^M } \frac{\sqrt{\varepsilon^\top K^{(\tilde{w})} \varepsilon}}{n} + \zeta^{1/2} \sum_{i=1}^n |\varepsilon_i| \sqrt{\|x_i\|^*}.
\end{equation}
Considering the expectation over $ \varepsilon $ of Equation~\eqref{eq:2sup}, we get
\begin{equation*}
  \mathbb{E}_\varepsilon \left( \sup_{w \in \mathcal{S}^{d-1}} \frac{\sqrt{\varepsilon^\top K^{(w)} \varepsilon}}{n}\right) \leq \mathbb{E}_\varepsilon \left( \sup_{\tilde{w} \in \mathcal{W}^M } \frac{\sqrt{\varepsilon^\top K^{(\tilde{w})} \varepsilon}}{n} \right) + \zeta^{1/2} \mathbb{E}_\varepsilon \left( \frac{1}{n} \sum_{i=1}^n |\varepsilon_i| \sqrt{\|x_i\|^*}\right).
\end{equation*}

We now handle $ \mathbb{E}_\varepsilon \left( \sup_{\tilde{w} \in \mathcal{W}^M } \frac{\sqrt{\varepsilon^\top K^{(\tilde{w})}\varepsilon}}{n}\right) $ using standard concentration tools for supremum of infinitely many random variables. Consider $ t>0 $, then
\begin{align*}
   \mathbb{E}_\varepsilon \left( \sup_{\tilde{w} \in \mathcal{W}^M } \sqrt{\varepsilon^\top  K^{(\tilde{w})}\varepsilon}\right) 
   & \leq  \sqrt{\mathbb{E}_\varepsilon \left( \sup_{\tilde{w} \in \mathcal{W}^M } \varepsilon^\top  K^{(\tilde{w})}\varepsilon \right)} \\
   & \leq \sqrt{\frac{1}{t} \log\bigg(\mathbb{E}_\varepsilon \bigg( e^{t \sup_{\tilde{w} \in \mathcal{W}^M }\varepsilon^\top  K^{(\tilde{w})}\varepsilon}\bigg)\bigg)} \\
   & = \sqrt{\frac{1}{t} \log\left(\mathbb{E}_\varepsilon \left( \sup_{\tilde{w} \in \mathcal{W}^M } e^{t  \varepsilon^\top  K^{(\tilde{w})}\varepsilon}\right)\right)} \\
   & \leq \sqrt{\frac{1}{t} \log\bigg(\mathbb{E}_\varepsilon \bigg( \sum_{\tilde{w} \in \mathcal{W}^M } e^{t \varepsilon^\top  K^{(\tilde{w})}\varepsilon}\bigg)\bigg)} \\
   & =  \sqrt{\frac{1}{t} \log\bigg( \sum_{\tilde{w} \in \mathcal{W}^M } \mathbb{E}_\varepsilon  \left( e^{t \varepsilon^\top K^{(\tilde{w})}\varepsilon}\right)\bigg)}.
\end{align*}
Fix $ \Tilde{w} \in \mathcal{W}^M $ and consider $ \mathbb{E}_\varepsilon \left( e^{t  \varepsilon^\top  K^{(\tilde{w})}\varepsilon}\right) $. Diagonalising $ K^{(\tilde{w})} $ to $ U_{\Tilde{w}}D_{\tilde{w}}U_{\Tilde{w}}^\top $, we have that $ U_{\Tilde{w}}^\top \varepsilon $ is still a Gaussian vector with identity covariance matrix. When $ t $ is small enough, i.e., $ \forall i \in [n], 2t (D_{\Tilde{w}})_i < 1 $, or $ t < \frac{1}{2\max_i (D_{\Tilde{w}})_i} $,
\begin{align*} 
\mathbb{E}_\varepsilon \left( e^{t  \varepsilon^\top K^{(\tilde{w})}\varepsilon }\right) &= \mathbb{E}_\varepsilon \left( e^{t  \sum_{i=1}^n (D_{\tilde{w}})_i\varepsilon_i^2} \right) = \prod_{i=1}^n \mathbb{E}_\varepsilon( e^{t (D_{\Tilde{w}})_i \varepsilon_i^2}) \\
&= \prod_{i=1}^n \int_{\mathbb{R}} \frac{1}{\sqrt{2\pi}}e^{(t (D_{\Tilde{w}})_i -1/2) \varepsilon_i^2 } \, {\rm d}\varepsilon_i \\
& = \prod_{i=1}^n \int_{\mathbb{R}} \frac{1}{\sqrt{2\pi}}e^{(2t (D_{\Tilde{w}})_i -1) \frac{\varepsilon_i^2}{2} } \, {\rm d}\varepsilon_i = \prod_{i=1}^n (1 - 2t(D_{\tilde{w}})_i)^{-1/2}.
\end{align*}
Re-injecting this, we obtain
\begin{align*}
    \log \left(\mathbb{E}_\varepsilon \left( e^{t  \varepsilon^\top  K^{(\tilde{w})}\varepsilon }\right)\right) 
    &= \log\left(\prod_{i=1}^n (1 - 2t(D_{\tilde{w}})_i)^{-1/2}\right) \\
    & \leq \frac{-1}{2}\sum_{i=1}^n \log(1 - 2t(D_{\tilde{w}})_i).
\end{align*}
To bound this further, take $ t \leq \frac{1}{4\max((D_{\tilde{w}})_i)} $, which implies both $ 2t(D_{\tilde{w}})_i <1/2 $ and $ -\log(1 - 2t(D_{\tilde{w}})_i) \leq 4t(D_{\tilde{w}})_i $, leading to
\begin{align*}
    \log \left(\mathbb{E}_\varepsilon \left( e^{t  \varepsilon^\top  K^{(\tilde{w})}\varepsilon }\right)\right) &\leq 2t\sum_{i=1}^n (D_{\tilde{w}})_i  \leq 2 t \trace(K^{(\tilde{w})}) \leq 2 t \sum_{i=1}^n \|x_i\|^*.
\end{align*}
Taking $ t \leq \min_{\tilde{w} \in \mathcal{W}^M } \frac{1}{4\max((D_{\tilde{w}})_i)} $, we obtain
\begin{align*}
     \mathbb{E}_\varepsilon \left( \sup_{\tilde{w} \in \mathcal{W}^M } \frac{\sqrt{\varepsilon^\top K^{(\tilde{w})}\varepsilon}}{n} \right)
   &\leq \frac{1}{n} \sqrt{\frac{1}{t} \log \left(M e^{2t \sum_{i=1}^n \|x_i\|^*}\right)} \\
   & \leq \frac{1}{n} \sqrt{\frac{1}{t} \left( \log M  + 2t \sum_{i=1}^n \|x_i\|^*\right)}.
\end{align*}
\red{ Let us first check that  $ t=\frac{1}{4\sum_{i=1}^n \|x_i\|^*} $, fulfils the previously required conditions. This is equivalent to showing that we have for any $\tilde{w} \in \mathcal{W}^M, \max_{i \in [n]} (D_{\tilde{w}})_i \leq \sum_{i=1}^n \|x_i\|^*$. First, $\max_{i \in [n]} (D_{\tilde{w}})_i$ is the largest eigenvalue of $K^{(\tilde{w})}$, hence it is smaller than the trace of $K^{(\tilde{w})}$. We also know that this trace is equal to $\sum_{i=1}^n |\tilde{w}^\top x_i| $. Now since $\|\tilde{w}\| =1$, this means that $\sum_{i=1}^n |\tilde{w}^\top x_i|  \leq \sum_{i=1}^n \|x_i\|^*$. Thus by taking this $t$, we get}
\begin{equation*}
     \mathbb{E}_\varepsilon \left( \sup_{\tilde{w} \in \mathcal{W}^M } \frac{\sqrt{\varepsilon^\top  K^{(\tilde{w})}\varepsilon}}{n} \right) \leq \frac{1}{\sqrt{n}}\sqrt{\frac{\sum_{i=1}^n \|x_i\|^*}{n}}\sqrt{4\log M + 2}.
\end{equation*}
In the end, we obtain
\begin{align*}
     \mathbb{E}_\varepsilon \left( \sup_{w \in \mathcal{S}^{d-1}} \frac{\sqrt{\varepsilon^\top K^{(w)}\varepsilon}}{n} \right) &\leq \frac{1}{\sqrt{n}}\sqrt{\frac{\sum_{i=1}^n \|x_i\|^*}{n}}\sqrt{4\log M + 2} + \zeta^{1/2} \sqrt{\frac{\sum_{i=1}^n \|x_i\|^*}{n}},
\end{align*}
where we have used $ \mathbb{E}_\varepsilon |\varepsilon_i| \leq \sqrt{\mathbb{E}_\varepsilon (\varepsilon_i)^2} = 1 $ and $ \frac{\sum_{i=1}^n \sqrt{\|x_i\|^*}}{n} \leq \sqrt{\frac{\sum_{i=1}^n \|x_i\|^*}{n}} $.

We know that $ M \leq (1+2/\zeta)^d $ \citep[Lemma 5.7]{Wainwright_2019}, yielding
\begin{align*}
  \mathbb{E}_\varepsilon \left( \sup_{w \in \mathcal{S}^{d-1}} \frac{\sqrt{\varepsilon^\top K^{(w)}\varepsilon}}{n} \right) &\leq \sqrt{\frac{\sum_{i=1}^n \|x_i\|^*}{n}}\left( \frac{\sqrt{4d\log(1+\frac{2}{\zeta}) +2}}{\sqrt{n}} + \zeta^{1/2} \right) \\
  &    \leq \sqrt{\frac{\sum_{i=1}^n \|x_i\|^*}{n}}\left( \frac{\sqrt{4d\log(1+\frac{n}{2d}) +2}}{\sqrt{n}} + \red{\sqrt{\frac{4d}{n}} }\right) \\
 & \leq 2\sqrt{\frac{\sum_{i=1}^n \|x_i\|^*}{n}}\frac{\sqrt{d}}{\sqrt{n}}\left( \sqrt{\log\left(1+\frac{n}{2d}\right) +\frac{1}{2d}} + 1 \right) \\
 & \leq  4\sqrt{\frac{\sum_{i=1}^n \|x_i\|^*}{n}}\frac{\sqrt{d}}{\sqrt{n}}\left( \sqrt{\log\left(1+\frac{n}{2d}\right)} + 1 \right) \\
 & \leq  8\sqrt{\frac{\sum_{i=1}^n \|x_i\|^*}{n}}\frac{\sqrt{d}}{\sqrt{n}}\sqrt{\log (n+1)},
\end{align*}
where to get the second line, we took $ \zeta = 4d/n $. By taking the expectation over the data set~$ \mathcal{D}_n $, since $ \mathbb{E}_{\mathcal{D}_n}\big( \sqrt{n^{-1}\sum_{i=1}^n \|x_i\|^*} \big) \leq \sqrt{\mathbb{E}(\|X\|^*)} $, we have the desired result.
\end{proof}
\vspace{0.0em}
\subsubsection{Dimension-Independent Bound}
\label{sec:dim_independant}
We now bound the Gaussian complexity with a quantity that does not explicitly depend on the dimension of the data. Recall that we aim to bound
\begin{align*}
    G_n = \mathbb{E}_{\varepsilon, \mathcal{D}_n} \left( \sup_{\|g\|_\mathcal{H} \leq 1, w \in \mathcal{S}^{d-1}} \frac{1}{n} \sum_{i=1}^n \varepsilon_i g(w^\top x_i) \right),
\end{align*}

where $\varepsilon$ is a centred Gaussian vector with an identity covariance matrix. First, recall that the functions in $\mathcal{H}$ with norm bounded by 1 are not Lipschitz functions but are instead $1/2$-Hölder functions (Lemma~\ref{lemma:caracterisation_f_infty}). Specifically, let $g\in \mathcal{H}, \|g\|_\mathcal{H} \leq 1$, then for any $a,b \in \mathbb{R}$, we have $|g(a) - g(b)| \leq \|k_a -k_b\|_\mathcal{H} = \sqrt{|a-b|}$. 

An interesting result for a fixed 1-Lipschitz function $h$ is that we can apply the contraction principle \citep[Proposition 4.3]{francis_book} to the Rademacher complexity. Informally, this yields
\begin{equation*}
    \mathbb{E}_{\varepsilon} \left( \sup_{w \in \mathcal{S}^{d-1}} \frac{1}{n} \sum_{i=1}^n \varepsilon_i h(w^\top x_i) \right) \leq \mathbb{E}_\varepsilon \left( \sup_{w \in \mathcal{S}^{d-1}} \frac{1}{n} \sum_{i=1}^n \varepsilon_i w^\top x_i \right),
\end{equation*}
where exceptionally $\varepsilon$ is composed of independent Rademacher variables. The supremum in the second term can then be taken explicitly. We will make use of this idea by first approximating the functions in the unit ball of $\mathcal{H}$ with Lipschitz functions, before using Slepian's lemma \citep[Corollary 3.14]{talagrand} to obtain similar results on the Gaussian complexity.

\begin{lemma}[Lipschitz Approximation]
\label{lemma:delta_lip}
Let $g \in \mathcal{H}$ with $\|g\|_\mathcal{H} \leq 1$, and let $\zeta > 0$. There exists a $(1/\zeta)$-Lipschitz function $g_\zeta:\mathbb{R} \to \mathbb{R}$ with $g_\zeta(0) = 0$ such that $\|g - g_\zeta\|_\infty \leq \zeta$.
\end{lemma}
The proof can be found in Appendix~\ref{proof:lemma:delta_lip}. This lemma indicates that we can approximate functions in the unit ball of the RKHS $\mathcal{H}$ up to any precision in the infinite norm by Lipschitz functions with a Lipschitz constant equal to the inverse of the precision.

\begin{theorem}[Dimension-Independent Bound] 
\label{theo:dim_indep}
If $\mathcal{S}^{d-1}$ is the $\ell_1$ or the $\ell_2$ sphere, then
\begin{equation*}
    G_n \leq \frac{\red{6}}{n^{1/6}}\left((\log 2d)^{1/4} \mathds{1}_{*=\infty} +\mathds{1}_{*=2}\right)\left(\mathbb{E}_{\mathcal{D}_n}\bigg(\max_{i \in [n]} \left(\|X_i\|^*\right)^2\bigg)\right)^{1/4}.
\end{equation*}
\end{theorem}
Recall that in the $\ell_1$ sphere case,   $\|\cdot \|^* = \|\cdot\|_\infty$, and in the $\ell_2$ case $\|\cdot\|^*=\|\cdot\|_2$. Here, we obtain a bound on the Gaussian complexity that depends only mildly on the data dimension~$d$, either not at all in the case of the $\ell_2$ sphere or logarithmically for the $\ell_1$ sphere. This means that the estimator \textsc{BKerNN} can be effectively used in high-dimensional settings, where the data dimension may be exponentially large relative to the sample size. This improved dependency on the dimension $d$ comes at the cost of a worse dependency on the sample size $n$ compared to Theorem~\ref{theo:dim}. Note also that there can be an implicit dependency on the dimension through the data distribution, which we discuss in Lemma~\ref{lemma:data_dependent} at the end of Section~\ref{sec:stat_analysis} under different data-generating mechanisms.

\begin{proof}[Proof of Theorem~\ref{theo:dim_indep}]
By applying Lemma~\ref{lemma:delta_lip}, we have for any $\zeta_1>0$
\begin{align*}
   \hat{G}_n &:= \mathbb{E}_\varepsilon \left( \sup_{w \in \mathcal{S}^{d-1}, \|g\|_\mathcal{H} \leq 1} \frac{1}{n} \sum_{i=1}^n \varepsilon_i g(w^\top x_i) \right) \\
   &= \mathbb{E}_\varepsilon \left( \sup_{w \in \mathcal{S}^{d-1}, \|g\|_\mathcal{H} \leq 1} \frac{1}{n} \sum_{i=1}^n \varepsilon_i \left( g_{\zeta_1}(w^\top x_i) + g(w^\top x_i) - g_{\zeta_1}(w^\top x_i)\right) \right) \\
      &\leq \mathbb{E}_\varepsilon  \red{\left( \sup_{w \in \mathcal{S}^{d-1}, \|g\|_\mathcal{H} \leq 1} \left( \frac{1}{n} \sum_{i=1}^n \varepsilon_i  g_{\zeta_1}(w^\top x_i)  + \|g - g_{\zeta_1}\|_\infty  \right) \right)}.
\end{align*}
We can then change the supremum over the unit ball of $\mathcal{H}$ to a supremum over Lipschitz functions
\begin{align*}
  \hat{G}_n  &\leq \mathbb{E}_\varepsilon \left( \sup_{w \in \mathcal{S}^{d-1}, \ g_{\zeta_1}  (1/{\zeta_1})-{ \rm Lip}, \ g_{\zeta_1}(0)=0} \frac{1}{n} \sum_{i=1}^n \varepsilon_i g_{\zeta_1}(w^\top x_i) \right) + \zeta_1 \\
   & =  \frac{1}{\zeta_1}\mathbb{E}_\varepsilon \left( \sup_{h \ 1-{ \rm Lip}, \ h(0)=0} \sup_{w \in \mathcal{S}^{d-1}}\frac{1}{n} \sum_{i=1}^n \varepsilon_i h(w^\top x_i) \right) + \zeta_1 \\
   & = \red{2}\sqrt{\mathbb{E}_\varepsilon \left( \sup_{h \ 1-{ \rm Lip}, \ h(0)=0} \sup_{w \in \mathcal{S}^{d-1}}\frac{1}{n} \sum_{i=1}^n \varepsilon_i h(w^\top x_i) \right)} ,
\end{align*}
by choosing the best $\zeta_1$. Technically, we can restrict ourselves to the following class of function: $\mathcal{F}_{1-{ \rm Lip}}:= \{h: [-\max_{i \in [n]} \|x_i\|^*, \max_{i \in [n]}\|x_i\|^*] \to \mathbb{R} \mid h(0)=0, h \ { \rm  is } \ 1-{ \rm Lipschitz} \}$. 

We then use a covering argument. To cover $\mathcal{F}_{1-{ \rm Lip}}$ up to precision $\zeta_2>0$ in $\| \cdot\|_\infty$ norm with $M$ functions from $\mathcal{F}_{1-{ \rm Lip}}$, one needs $M \leq \left(\frac{8\max_{i \in [n]} \|x_i\|^*}{\zeta_2} +1\right)2^{\frac{4\max_{i \in [n]}\|x_i\|^*}{\zeta_2}}$ \cite[Theorem 17]{covering_lip}. Let $h_1, \ldots h_M$ be such a covering. This yields that
\begin{align*}
\hat{G}_n &\leq  \red{2}\sqrt{\mathbb{E}_\varepsilon \left( \sup_{h \in \mathcal{F}_{1-{ \rm Lip}}} \sup_{w \in \mathcal{S}^{d-1}}\frac{1}{n} \sum_{i=1}^n \varepsilon_i h(w^\top x_i) \right)} \\
    & \leq \red{2} \sqrt{\mathbb{E}_\varepsilon \left( \sup_{h \in \{h_1, \ldots, h_M\} } \sup_{w \in \mathcal{S}^{d-1}} \frac{1}{n} \sum_{i} \varepsilon_i h(w^\top x_i) \right) + \zeta_2}, 
\end{align*}
by proceeding as with the covering of the unit ball of $\mathcal{H}$. 

\red{We then use Lemma~\ref{lemma:inspired_by_bartlett} to bound the expectation on the supremum of the finite set of Lipschitz functions, which is inspired by \citet{Bartlett}. This yields }

\begin{align}
\label{eq:after_slepian}
\mathbb{E}_\varepsilon & \bigg( \sup_{h \in \{h_1, \ldots, h_M\}, w \in \mathcal{S}^{d-1}} \frac{1}{n} \sum_{i=1}^n \varepsilon_i h(w^\top x_i) \bigg) \nonumber \\
& \red{ \leq  \mathbb{E}_\varepsilon \left(  \bigg\| \frac{\sqrt{2}}{n} \sum_{i=1}^n \varepsilon_i x_i \bigg\|^* + \sqrt{8 \frac{\sum_{i=1}^n (\|x_i\|^*)^2}{n^2}}\sqrt{2 \log M } \right).}
\end{align}

We then consider each term of Equation~\eqref{eq:after_slepian} separately, while also taking expectation with regards to the data set. For the second term, using the bound on $M$  \citep[Theorem 17]{covering_lip} and basic inequalities to simplify the term, we have
\begin{align*}
\mathbb{E}_{\varepsilon, \mathcal{D}_n} & \left( \sqrt{8 \frac{\sum_{i=1}^n (\|x_i\|^*)^2}{n^2}}\sqrt{2 \log M } \right) \\ & \leq \mathbb{E}_{\mathcal{D}_n} \left( \sqrt{8 \frac{\sum_{i=1}^n (\|x_i\|^*)^2}{n^2}}\sqrt{\frac{4\max_{i \in [n]}\|x_i\|^*}{\zeta_2} \log 2 + \log \left(\frac{8\max_{i \in [n]} \|x_i\|^*}{\zeta_2} +1\right)}\right) \\
& \leq \mathbb{E}_{\mathcal{D}_n} \left( 8\sqrt{ \frac{\sum_{i=1}^n (\|x_i\|^*)^2}{n^2}}\sqrt{\frac{\max_{i \in [n]}\|x_i\|^*}{\zeta_2} }\right) \\
& \leq \frac{8}{\sqrt{n}}\frac{1}{\sqrt{\zeta_2}} \mathbb{E}_{\mathcal{D}_n} \left( \max_{i \in [n]} (\|x_i\|^*)^{3/2}  \right) \leq \frac{8}{\sqrt{n}}\frac{1}{\sqrt{\zeta_2}}\left( \mathbb{E}_{\mathcal{D}_n} \left( \max_{i \in [n]} (\|x_i\|^*)^{2}  \right)\right)^{3/4}.
\end{align*}
We can reinject, yielding
\begin{align*}
\red{\frac{G_n^2}{4}}&\leq   \mathbb{E}_{\varepsilon, \mathcal{D}_n} \left(  \bigg\| \frac{\sqrt{2}}{n} \sum_{i=1}^n \varepsilon_i x_i \bigg\|^*\right) + \frac{8}{\sqrt{n}}\frac{1}{\sqrt{\zeta_2}} \left(\mathbb{E}_{\mathcal{D}_n} \left( \max_{i \in [n]} (\|x_i\|^*)^{2}  \right)\right)^{3/4} + \zeta_2 \\
& \leq \mathbb{E}_{\varepsilon, \mathcal{D}_n} \left(  \bigg\| \frac{\sqrt{2}}{n} \sum_{i=1}^n \varepsilon_i x_i \bigg\|^*\right) + 2\left(\frac{8}{\sqrt{n}} \left(\mathbb{E}_{\mathcal{D}_n} \left( \max_{i \in [n]} (\|x_i\|^*)^{2}  \right)\right)^{3/4}\right)^{2/3} \\
& \leq\mathbb{E}_{\varepsilon, \mathcal{D}_n} \left(  \bigg\| \frac{\sqrt{2}}{n} \sum_{i=1}^n \varepsilon_i x_i \bigg\|^*\right) + 2 \frac{4}{n^{1/3}} \sqrt{\mathbb{E}_{\mathcal{D}_n} \left( \max_{i \in [n]} (\|x_i\|^*)^{2}  \right)},
\end{align*}
by taking $\zeta_2^{3/2} = \frac{8}{\sqrt{n}}\left(\mathbb{E}_{\mathcal{D}_n} \left( \max_{i \in [n]} (\|x_i\|^*)^{2}  \right)\right)^{3/4} $ in the second line.

Now for the first term from Equation~\eqref{eq:after_slepian} which we have to deal with still, consider first the case $\|\cdot \|^* = \|\cdot\|_2$ then,
\begin{align*}
\mathbb{E}_{\varepsilon, \mathcal{D}_n} \left(  \| \frac{\sqrt{2}}{n} \sum_{i=1}^n \varepsilon_i x_i \|_2 \right) &\leq \sqrt{\mathbb{E}_{\varepsilon, \mathcal{D}_n} \left(  \| \frac{\sqrt{2}}{n} \sum_{i=1}^n \varepsilon_i x_i \|_2^2 \right)} \\
& = \frac{\sqrt{2}}{n}\sqrt{\mathbb{E}_{\mathcal{D}_n} \left(  \sum_{i=1}^n \| x_i \|_2^2 \right)}  = \frac{\sqrt{2}}{\sqrt{n}}\sqrt{\mathbb{E}_{X} \left(  \|X \|_2^2 \right)}.
\end{align*}
In the other case where $\|\cdot \|^* = \| \cdot \|_\infty$, we can use \citet[Theorem 2.5]{boucheron2013concentration}, as for a fixed data set $\mathcal{D}_n$, $  \sum_{i=1}^n \varepsilon_i s(x_i)_a $ is a centred Gaussian vector with variance equal to $\sum_{i=1}^n ((x_i)_a)^2$ which is smaller than $\max_{a \in [d]} \sum_{i=1}^n ((x_i)_a)^2$. This yields that
\begin{align*}
\mathbb{E}_{\varepsilon, \mathcal{D}_n} \left(  \| \frac{\sqrt{2}}{n} \sum_{i=1}^n \varepsilon_i x_i \|_{\infty} \right) &= \frac{\sqrt{2}}{n}\mathbb{E}_{\varepsilon, \mathcal{D}_n} \left( \max_{a \in [d]} | \sum_{i=1}^n \varepsilon_i (x_i)_a| \right) \\
& = \frac{\sqrt{2}}{n}\mathbb{E}_{\varepsilon, \mathcal{D}_n} \left( \max_{a \in [d], s \in \{-1, 1\}} \sum_{i=1}^n \varepsilon_i s(x_i)_a \right) \\
& \leq \frac{\sqrt{2}}{n}\mathbb{E}_{\mathcal{D}_n} \left( \max_{a \in [d]}\sqrt{2\sum_{i=1}^n ((x_i)_a)^2 \log(2d)} \right).
\end{align*}
We then have
\begin{align*}
\mathbb{E}_{\varepsilon, \mathcal{D}_n} \left(  \| \frac{\sqrt{2}}{n} \sum_{i=1}^n \varepsilon_i x_i \|_{\infty} \right)  & \leq \frac{2}{n}\sqrt{\log 2d}\mathbb{E}_{\mathcal{D}_n}\left( \sqrt{\max_{a \in [d]} \sum_{i=1}^n ((x_i)_a)^2 } \right) \\
& \leq \frac{2}{n}\sqrt{\log 2d}\mathbb{E}_{\mathcal{D}_n}\left( \sqrt{\sum_{i=1}^n \max_{a \in [d]} ((x_i)_a)^2 } \right) \\
& \leq \frac{2}{n}\sqrt{\log 2d}\mathbb{E}_{\mathcal{D}_n}\left( \sqrt{\sum_{i=1}^n \|x_i\|_\infty^2 } \right) \\
& \leq \frac{2}{\sqrt{n}}\sqrt{\log 2d}\sqrt{\mathbb{E}_{X}\left( \|X\|_\infty^2  \right)}.
\end{align*}
This yields that the last term of Equation~\eqref{eq:after_slepian} can be bounded
\begin{align*}
\red{\frac{G_n^2}{4}}&\leq \left(\frac{2}{\sqrt{n}}\sqrt{\log 2d}\mathds{1}_{*=\infty} + \frac{\sqrt{2}}{\sqrt{n}}\mathds{1}_{*=2} \right)\sqrt{\mathbb{E}_{X} ((\|X\|^*)^2)} + 2 \frac{4}{n^{1/3}} \sqrt{\mathbb{E}_{\mathcal{D}_n} \left( \max_{i \in [n]} (\|x_i\|^*)^{2}  \right)} \\
& \leq  \left( \sqrt{\log 2d}\mathds{1}_{*=\infty} +
 \mathds{1}_{*=2} \right) \frac{8}{n^{1/3}} \sqrt{\mathbb{E}_{\mathcal{D}_n} \left( \max_{i \in [n]} (\|x_i\|^*)^{2}  \right)},
\end{align*}
hence
\begin{align*}
    G_n & \leq \left( \sqrt{\log 2d}\mathds{1}_{*=\infty} +
 \mathds{1}_{*=2} \right)^{1/4} \frac{\red{6}}{n^{1/6}} \left(\mathbb{E}_{\mathcal{D}_n} \left( \max_{i \in [n]} (\|x_i\|^*)^{2}  \right)\right)^{1/4}, 
\end{align*}
which concludes the proof.
\end{proof}
\vspace{0.0em}
\subsection{Bound on Expected Risk of Regularised Estimator}
\label{sec:bound_true_risk}
We now use the bounds on the Gaussian complexity we have obtained in Section~\ref{sec:gaussian_complex} to derive a bound on the expected risk of \textsc{BKerNN}. We show that, with explicit rates, the expected risk of our estimator converges with high-probability to that of the minimiser for data with subgaussian norms, which includes both bounded data and data with subgaussian components. 
First, we provide a definition of subgaussian real variables, as given by \citet{Vershynin_2018}.
\begin{definition}[Subgaussian Variables]
\label{def:subgaussian}
Let $Z$ be a real-valued (not necessarily centred) random variable. $Z$ is subgaussian with variance proxy $\sigma^2$ if and only if 
\begin{equation*}
\forall t > 0, \max\left( \mathbb{P}(Z \geq t), \mathbb{P}(Z \leq -t) \right) \leq e^{-\frac{t^2}{2\sigma^2}}.
\end{equation*}
\end{definition}

We now present the main theoretical result of the paper.
\begin{theorem}[Bound on Expected Risk with High-Probability] 
\label{theo:last}
Let the estimator \linebreak function be $\hat{f}_\lambda := \argmin_{f \in \mathcal{F}_\infty} \widehat{\mathcal{R}}(f) + \lambda \Omega(f)$. Assume the following:
\begin{enumerate}
\item \textbf{Well-specified model}: The minimiser $f^* := \argmin_{f \in \mathcal{F}_\infty} \mathcal{R}(f)$ exists.
    \item \textbf{Convexity of the loss}: For any $(x,y) \in \mathcal{X} \times \mathcal{Y}$, $f \in \mathcal{F}_\infty \to \ell(y, f(x))$ is convex.
    \item \textbf{Lipschitz condition}: The loss $\ell$ is $L$-Lipschitz in its second (bounded) argument, i.e., $\forall y \in \mathcal{Y}, a \in \{ f(x) \mid x \in \mathcal{X}, \ f \in \mathcal{F}_\infty, \ \Omega(f) \leq 2\Omega(f^*) \}, a \to \ell(y, a)$ is $L$-Lipschitz.
    \item \textbf{Data distribution}: The data set $ (x_i,y_i)_{i \in [n]}$ consists of i.i.d.~samples of the random variable $(X,Y)$ where $1 + \sqrt{\|X\|^*}$ is subgaussian with variance proxy $\sigma^2$.
\end{enumerate}

Then, for any $\delta \in (0,1)$, with probability larger than $1-\delta$, for $\lambda = 12L\left(\frac{1}{\sqrt{n}} + G_n\right) + \frac{288L\sigma}{\sqrt{n}} \sqrt{\log \frac{1}{\delta}}$,
\begin{equation*}
\mathcal{R}(\hat{f}_\lambda) \leq \mathcal{R}(f^*) +  24\Omega(f^*)L \left( \frac{1}{\sqrt{n}} + G_n + \frac{24\sigma}{\sqrt{n}} \sqrt{\log \frac{1}{\delta}}\right).
\end{equation*}

With the bounds on $G_n$ from Theorem~\ref{theo:dim} and Theorem~\ref{theo:dim_indep}, recall that if $\|\cdot\|$ is either $\|\cdot\|_2$ or $\|\cdot\|_1$, we have
\begin{align*}
G_n \leq \min\bigg(\frac{\red{6}}{n^{1/6}}\big((\log 2d)^{1/4}\mathds{1}_{*=\infty} + \mathds{1}_{*=2}\big) \bigg(\mathbb{E}_{\mathcal{D}_n}\big(\max_{i \in [n]} (\|X_i\|^*)^2\big)\bigg)^{1/4}&,\\
8\sqrt{\frac{d}{n}} \sqrt{\log(n+1)} \sqrt{ \mathbb{E}_X \|X\|^*}\bigg)&.
\end{align*}
\end{theorem}
\begin{proof}[Proof of Theorem~\ref{theo:last}]
This proof is primarily based on \citet[Proposition~4.7]{francis_book}.

Let $ f_\lambda^* $ be a minimiser of $\mathcal{R}_\lambda :=\mathcal{R} + \lambda \Omega$ over $\mathcal{F}_\infty$. Consider the set $\mathcal{C}_\tau := \{ f \in \mathcal{F}_\infty \mid \mathcal{R}_\lambda(f) - \mathcal{R}_\lambda(f_\lambda^*) \leq \tau \} $ for some $\tau > 0$ that will be chosen later. $\mathcal{C}_\tau$ is a convex set by the convexity assumption on the loss $\ell$.

First, we show that $\mathcal{C}_\tau$ is included in the set $\mathcal{B}_\tau := \{ f \in \mathcal{F}_\infty \mid \Omega(f) \leq \Omega(f^*) + \tau/\lambda \} $. This inclusion follows from the optimality of $ f^* $ and $ f^*_\lambda $. Let $ f \in \mathcal{C}_\tau $, then
\begin{align*}
    \mathcal{R}_\lambda(f) \leq \mathcal{R}_\lambda(f_\lambda^*) + \tau \leq \mathcal{R}_\lambda(f^*) + \tau \leq \mathcal{R}(f) + \lambda\Omega(f^*) + \tau,
\end{align*}
yielding $ f \in \mathcal{B}_\tau $.

Next, set $\tau = \lambda \Omega(f^*)$ with $\lambda$ to be chosen later. We show that $\hat{f}_\lambda$ belongs to $\mathcal{C}_\tau$ with high probability. If $\hat{f}_\lambda \notin \mathcal{C}_\tau$, since $ f^*_\lambda \in \mathcal{C}_\tau $ and $\mathcal{C}_\tau$ is convex, there exists a $\Tilde{f}$ in the segment $[\hat{f}_\lambda, f_\lambda^*]$ and which is on the boundary of $\mathcal{C}_\tau$, i.e. such that $\mathcal{R}_\lambda(\Tilde{f}) = \mathcal{R}_\lambda(f^*_\lambda) + \tau$. Since the empirical risk is convex, we have $\widehat{\mathcal{R}}_\lambda(\Tilde{f}) \leq \max(\widehat{\mathcal{R}}_\lambda(\hat{f}_\lambda), \widehat{\mathcal{R}}_\lambda(f^*_\lambda)) = \widehat{\mathcal{R}}_\lambda(f^*_\lambda)$. Then,
\begin{align}
\label{eq:contradiction}
    \widehat{\mathcal{R}}(f^*_\lambda) - \widehat{\mathcal{R}}(\Tilde{f}) -  \mathcal{R}(f^*_\lambda) + \mathcal{R}(\Tilde{f}) &= \widehat{\mathcal{R}}_\lambda(f^*_\lambda) - \widehat{\mathcal{R}}_\lambda(\Tilde{f}) -  \mathcal{R}_\lambda(f^*_\lambda) + \mathcal{R}_\lambda(\Tilde{f}) \nonumber \\ 
    &\geq -\mathcal{R}_\lambda(f^*_\lambda) + \mathcal{R}_\lambda(\Tilde{f}) = \tau.
\end{align}

Note that $\Omega(\Tilde{f}) \leq 2\Omega(f^*)$ and $\Omega(f_\lambda^*) \leq 2 \Omega(f^*)$. Combining Lemma~\ref{lemma:constrained_predictor} and Lemma~\ref{lemma:mcdiarmidv2}, for $\delta \in (0,1)$, with probability greater than $1-\delta$, we have for all $f \in \mathcal{F}_\infty$ such that $\Omega(f) \leq 2\Omega(f^*)$:
\begin{align*}
    \widehat{\mathcal{R}}(f^*_\lambda) - &\widehat{\mathcal{R}}(f) -  \mathcal{R}(f^*_\lambda) + \mathcal{R}(f) \\
    &\leq \mathbb{E}_{\mathcal{D}_n}\left(\sup_{f \in \mathcal{F}_\infty, \Omega(f) \leq 2\Omega(f^*)} \widehat{\mathcal{R}}(f) - \mathcal{R}(f) + \sup_{f \in \mathcal{F}_\infty, \Omega(f) \leq 2\Omega(f^*)} \mathcal{R}(f) - \widehat{\mathcal{R}}(f) \right) \\
    &\quad + \Omega(f^*)\frac{96\sqrt{2e}L\sigma}{\sqrt{n}} \sqrt{\log \frac{1}{\delta}} \\
    &\leq 12 \Omega(f^*)L\left(\frac{1}{\sqrt{n}} + G_n\right) + \Omega(f^*)\frac{96\sqrt{2e}L\sigma}{\sqrt{n}} \sqrt{\log \frac{1}{\delta}}.
\end{align*}

Now, choose $\lambda $ such that $\tau = \lambda \Omega(f^*) \geq 12\Omega(f^*)L\left(\frac{1}{\sqrt{n}} + G_n\right) +\Omega(f^*)\frac{96\sqrt{2e}L\sigma}{\sqrt{n}} \sqrt{\log \frac{1}{\delta}}$. This yields a contradiction with Equation~\eqref{eq:contradiction}. Thus, with such a $\lambda$, with probability greater than $1-\delta$, we have $\hat{f}_\lambda \in \mathcal{C}_\tau$, hence
\begin{align*}
    \mathcal{R}_\lambda(\hat{f}_\lambda) &\leq \mathcal{R}_\lambda(f_\lambda^*) + \lambda \Omega(f^*), \\
    \mathcal{R}(\hat{f}_\lambda) &\leq \mathcal{R}(f^*) + 2\lambda \Omega(f^*).
\end{align*}

For $\lambda = 12L\left(\frac{1}{\sqrt{n}} + G_n\right) +\frac{288L\sigma}{\sqrt{n}} \sqrt{\log \frac{1}{\delta}}$, this yields
\begin{equation*}
\mathcal{R}(\hat{f}_\lambda) \leq \mathcal{R}(f^*) + \Omega(f^*)\left(24L\left(\frac{1}{\sqrt{n}} + G_n\right) +\frac{576L\sigma}{\sqrt{n}} \sqrt{\log \frac{1}{\delta}}\right).
\end{equation*}
\end{proof}
\vspace{0.0em}
We now provide insightful comments on Theorem~\ref{theo:last}. We first remark that the result could be  proven more directly for bounded data using McDiarmid's inequality, resulting in a better constant.

The chosen $\lambda$ does not depend on unknown quantities such as $\Omega(f^*)$, but only on known quantities such as the Lipschitz constant of the loss, the sample size, or the dimension of the data. This allows $\lambda$ to be explicitly chosen for a fixed probability $\delta$, although it is usually computed through cross-validation.

Classical losses typically satisfy our assumptions. For instance, the square loss is always convex and $L$-Lipschitz if the data and response are bounded, with $L = 2\sup_{y \in \mathcal{Y}} |y| + 4\Omega(f^*)\sup_{x \in \mathcal{X}} \|x\|^*$. Similarly, the logistic loss is always convex and $L$-Lipschitz with $L=1$ (in the context of outputs in $\{-1, 1\})$. \red{While our primary analysis relies on the Lipschitz property, an alternative approach using offset Rademacher complexity \citep{pmlr-v40-Liang15} or the direct computations proposed by \citet[Chapter 8]{francis_book} could accommodate the square loss.}

Our approach stands out by requiring minimal assumptions on the data-generating mechanism, which is less restrictive compared to other methodologies in the multi-index model domain. This emphasis on general applicability is also why we do not include feature recovery results, as such outcomes typically necessitate strong assumptions about the data and often require prior knowledge of the distribution.

The rates obtained depend explicitly on the dimension of the data through the bound on the Gaussian complexity. Considering the first term in the minimum, we observe that the bound is independent (up to logarithmic factors) of the data dimension, making \textsc{BKerNN} suitable for high-dimensional problems. However, this bound has a less favourable dependency on the sample size compared to the dimension-dependent bound, which is the second term in the minimum. We conjecture that the actual rate has the best of both worlds, achieving an explicit dependency on dimension $d$ and sample size $n$ of $n^{-1/2}$ (up to logarithmic factors).

Comparing the rate between \textsc{BKerNN}, neural networks with ReLU activations, and kernel methods, we find that in well-specified settings (where the Bayes estimator belongs to each function space considered), \textsc{KRR} yields a $O(n^{-1/2})$ rate independent of dimension, but require very smooth functions, for example, a Sobolev space of order $s$ (i.e. the derivatives up to order $s$ are square integrable) is only a RKHS if $s>d/2$ \citep[Chapter 7]{francis_book}. Neural networks with ReLU activation achieve a similar rate with fewer constraints, as their function space is typically larger than RKHS spaces \citep[Chapter 9]{francis_book}.

However, in the case of linear latent variables, i.e., under the multiple index model where~$ f^* = g^*(P^\top x) $ with $ P $ a $ d \times k $ matrix with $ k < d $ and orthonormal columns, the RKHS cannot take advantage of this hypothesis and the rates remain unchanged. In contrast, the neural network can, assuming that $ g^* $ has bounded Banach norm, then we only pay the price of the $k$ underlying dimensions and not the full $d$ dimensions \citep[Section~9.4]{francis_book}. \textsc{BKerNN} also has this property, which is visible by using the simple arguments presented in the discussion in \citet[Section 9.3.5]{francis_book}, which show that $\Omega(f^*) \leq \Omega(g^*)$. Moreover, the optimisation process for \textsc{BKerNN} is much easier than that of neural networks, and our function space is larger, underscoring the attractiveness of \textsc{BKerNN}. \red{Our analysis in this section could be complemented by an analysis of the mis-specified case (i.e., $f^\ast$ only assumed to be Lipschitz), where, controlling both approximation and estimation errors, we should expect our method to benefit from the same adaptivity to linear latent variables as neural networks~\citep{bach2017breaking}.}

There is also an implicit dependency on the dimension in Theorem~\ref{theo:last} through data-dependent terms, namely the variance proxy $\sigma^2$ or the expectations in the bound of $G_n$. We now examine these quantities under two data-generating mechanisms: bounded and subgaussian variables.

\begin{lemma}[Analysis of Data-Dependent Terms in Theorem~\ref{theo:last}] \label{lemma:data_dependent}
The following inequalities hold.
\begin{enumerate}
    \item If $ X $ is bounded, i.e., $ \|X\|^* \leq R $ almost surely, then
    \[
    \sqrt{ \mathbb{E}_X \|X\|^*} \leq \sqrt{R}, \quad \left(\mathbb{E}_{\mathcal{D}_n}\left(\max_{i\in [n]} (\|X_i\|^*)^2\right)\right)^{1/4} \leq \sqrt{R}.
    \]
    Moreover, $ 1+\sqrt{\|X\|^*} $ is subgaussian with variance proxy $ \sigma^2 \leq 1+\sqrt{R} $.

    \item If $ X $ is a vector of subgaussian variables (not necessarily centred or independent) with variance proxy $ \sigma_a^2 $ for component $ X_a $, then
    \[
    \sqrt{\mathbb{E}_X(\|X\|_2)} \leq \sqrt{6}\left(\sum_{a=1}^d \sigma_a^2\right)^{1/4}, \quad 
    \sqrt{\mathbb{E}_X(\|X\|_\infty)} \leq 4(\log d)^{1/4} \max_{a \in [d]}\sqrt{\sigma_a},
    \]
    \[
    \mathbb{E}_{\mathcal{D}_n}\left( \max_{i\in [n]}\|X_i\|_2^2 \right)^{1/4} \leq 4(1+\log(n))^{1/4} \left(\sum_{a=1}^d \sigma_a^2\right)^{1/4},
    \]
    \[
    \mathbb{E}_{\mathcal{D}_n}\left( \max_{i\in [n]}\|X_i\|_\infty \right)^{1/4} \leq 4(1+\log(nd))^{1/4} \max_{a \in [d]} \sqrt{\sigma_a}.
    \]
    Furthermore, $ 1+\sqrt{\|X\|_2} $ is subgaussian with variance proxy $ \sigma^2 \leq (1+ \sum_{a=1}^d \sigma_a)^2 $, and $ 1+\sqrt{\|X\|_\infty} $ is subgaussian with variance proxy $ \sigma^2 \leq 2 + \max_{a \in [d]} \sigma_a^2(1+\sqrt{\log(2d)})^2 $.
\end{enumerate}
\end{lemma}

See the proof in Appendix~\ref{proof:lemma:data_dependent}. Note that $R$ usually does not implicitly depend on the dimension in the $\|\cdot \|^*=\|\cdot\|_\infty$ case, and $R$ can typically be $O(d^{1/2})$ in the $\|\cdot \|_2$ case. For the subgaussian mechanism, each $\sigma_a$ typically does not depend on the dimension.

\section{Numerical Experiments}
\label{sec:experiments}
In this section, we present and analyse the properties of \textsc{BKerNN}. The \textsc{BKerNN} implementation in Python is fully compatible with Scikit-learn \citep{scikit-learn}, ensuring seamless integration with existing machine learning workflows. The source code, along with all necessary scripts to reproduce the experiments, is available at \url{https://github.com/BertilleFollain/BKerNN}. We define the scores and other estimators in the section below.

\subsection{Introduction to Scores and Competitors}
\label{sec:intro_num_exp}
In the experiments below, we use two scores to assess performance. The prediction score is defined as the coefficient of determination, a classical metric in the statistics literature \citep{r2}, $ R^2$, which ranges from $ -\infty $ to 1, where a score of 1 indicates perfect prediction, a score of 0 indicates that the model predicts no better than the mean of the target values, and negative values indicate that the model performs worse than this baseline. Mathematically, the $ R^2 $ score is defined as follows

\begin{equation}
\label{eq:r2}
  R^2 = 1 - \frac{\sum_{i=1}^n (y_i - \hat{y}_i)^2}{\sum_{i=1}^n (y_i - \bar{y})^2},
\end{equation}
where $ y_i $ are the true values, $ \hat{y}_i $ are the predicted values, $ \bar{y} $ is the mean of the true values, and $ n $ is the number of samples. 

The feature learning score measures the model's ability to identify and learn the true feature space (1 being the best, 0 the worst). It is computable only when the underlying feature space (in the form of a matrix $P \in \mathbb{R}^{d \times k}$, with $k$ the number of features) is known and relevant only when features are of similar importance, which we have ensured in the experiments below.

Depending on the regularisation type, the estimated feature matrix $ \hat{P} $ is computed via singular value decomposition (SVD) for $\Omega_{{ \rm feature}}$, $\Omega_{{ \rm concave \ feature}}$ or $\Omega_{{ \rm basic}}$ regularisation, or by selecting the top $ k $ variables for $\Omega_{{ \rm variable}}$ or $\Omega_{{ \rm concave \ variable}}$ regularisation. We then compute the projection matrices $ \pi_{\hat{P}} $ and $ \pi_{P} $ and calculate the feature learning error as the normalised Frobenius norm of their difference
\[
\pi_{\hat{P}} = \hat{P} (\hat{P}^\top \hat{P})^{-1} \hat{P}^\top \quad { \rm and} \quad \pi_{P} = P (P^\top P)^{-1} P^\top,
\]
\begin{equation}
\label{eq:feature_score}
{ \rm score} = \begin{cases}
1 - \frac{\| \pi_{P} - \pi_{\hat{P}} \|_F^2}{2k} & { \rm if } k \leq \frac{n_{{ \rm features}}}{2}, \\
1 - \frac{\| \pi_{P} - \pi_{\hat{P}} \|_F^2}{2n_{{ \rm features}} - 2k} & { \rm if } k > \frac{n_{{ \rm features}}}{2},
\end{cases}
\end{equation}
where the score is 1 if $ k = n_{{ \rm features}} $.

In several experiments, we compare the performance of \textsc{BKerNN} against \textsc{ReLUNN} and \textsc{B\textsc{KRR}}. \textsc{B\textsc{KRR}} refers to Kernel Ridge Regression using the multi-dimensional Brownian kernel $ k^{(mdB)}(x, x^\prime) =  (\|x\| + \|x^\prime\| - \|x - x^\prime\|)/2 $. \textsc{ReLUNN} is a simple one-hidden-layer neural network with ReLU activations, trained using batch stochastic gradient descent.

\subsection{Experiment 1: Optimisation Procedure, Importance of Positive Homogeneous Kernel}
\label{sec:num_exp1}

In this experiment, we compare \textsc{BKerNN} with two methods that differ from \textsc{BKerNN} only through the kernel that is used. We wish to illustrate the importance of the homogeneity assumptions discussed in Section~\ref{sec:optim_guarantees}. Specifically, we consider \textsc{ExpKerNN} with the (rescaled) exponential kernel $ k^{{ \rm exp}}(a,b) = e^{-|a-b|/2} $ and \textsc{GaussianKerNN} with the Gaussian kernel $ k^{{ \rm Gaussian}}(a,b) = e^{-|a-b|^2/2} $. Unlike the Brownian kernel used in \textsc{BKerNN}, the exponential and Gaussian kernels are not positively 1-homogeneous.

We trained all three methods on a simulated data set, using cross-validation to select the regularisation parameter $\lambda$ while keeping other parameters fixed ($m=100$, basic regularisation, more details are provided in Appendix~\ref{app:exp1}). The training set consisted of 214 samples and the test set of 1024. The data had $d=45$ dimensions with $k=5$ relevant features, and Gaussian additive noise with a standard deviation of 0.5. An orthogonal matrix $ P $ of size $ d \times d $ was sampled uniformly from the orthogonal group before being truncated to size $d\times k$.  The covariates were sampled uniformly from $[-1, 1]^d$, and the target variable $ y $ was computed as $y = 2\pi\left| \sum_{a=1}^k (P^\top x)_a \right| + { \rm noise}$.

We displayed the mean squared error (MSE) on both the training and test sets for the selected $\lambda$ for each method in Figure~\ref{fig:exp1}. While all three methods perform very well on the training set, the test set performance of \textsc{ExpKerNN} and \textsc{GaussianKerNN} is significantly worse compared to \textsc{BKerNN}. This discrepancy is not due to suboptimal regularisation choices, as cross-validation was used to select the best $\lambda$ for each method. 

Instead, the superior test performance of \textsc{BKerNN} underscores its effective optimisation process, avoiding the pitfalls of local minima that seem to trap \textsc{ExpKerNN} and \textsc{GaussianKerNN}. Our observations in Figure~\ref{fig:exp1} strongly support our discussion in Section~\ref{sec:optim_guarantees} on the critical role of the positive homogeneity of the kernel in ensuring convergence to a global minimum. 
\label{sec:exp1}
\begin{figure}
    \centering
    \includegraphics[width=\textwidth]{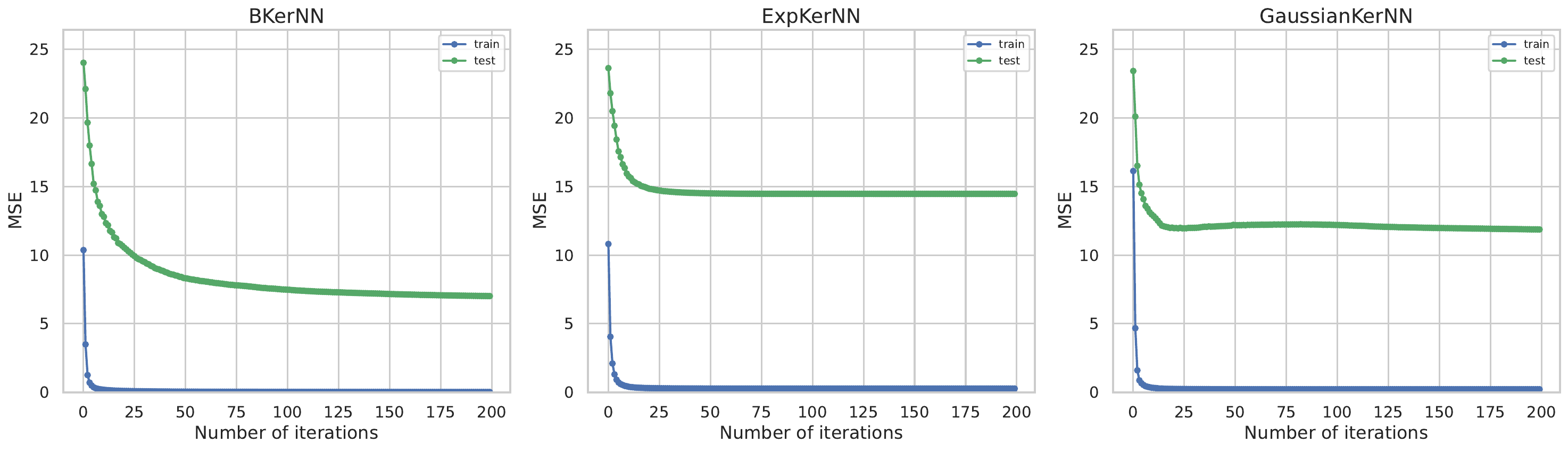}
    \caption{MSE across optimisation procedure for different kernels.}
    \label{fig:exp1}
\end{figure}

\subsection{Experiments 2 \& 3: Influence of Parameters (Number of Particles \texorpdfstring{$m$}{m}, Regularisation Parameter \texorpdfstring{$\lambda$}{lambda}, and Type of Regularisation)}
\label{sec:exp2&3}

In these experiments, we explore the impact of various parameters on the performance of \textsc{BKerNN}. Detailed descriptions can be found in Appendix~\ref{app:num_exp2&3}, and the results are presented in Figure~\ref{fig:exp2&3}. The $R^2$ score used to assess performance is described in Equation~\eqref{eq:r2}.

\subsubsection{Experiment 2}
The first two subplots of Figure~\ref{fig:exp2&3} illustrate the effects of the number of particles $m$ and the regularisation parameter $\lambda$ while keeping the data generation process consistent. The data set is the same for the two subplots. We used 412 training samples and 1024 test samples, with a data dimensionality of $d=20$ and $k=5$ relevant features. The standard deviation of additive Gaussian noise was set to $0.1$. The covariates were sampled uniformly from $[-1, 1]^d$. The target variable $y$  was computed using the formula $y = \sum_{a=1}^{k} \left|2\pi x_a\right| + { \rm noise}$.

Number of Particles ($m$): the first subplot shows that with too few particles, the estimator struggles to fit the training data, leading to poor performance on the test set. However, beyond a certain threshold, increasing the number of particles does not yield significant improvements in performance.
    
Regularisation Parameter ($\lambda$): the second subplot demonstrates the typical behaviour of a regularised estimator. When $\lambda$ is too small, the model overfits the training data, resulting in poor test performance. Conversely, when $\lambda$ is too large, the model underfits, performing poorly on both the training and test sets. Optimal performance on both sets is achieved with an intermediate value of $\lambda$.

\subsubsection{Experiment 3}
The third subplot in Figure~\ref{fig:exp2&3} examines the influence of the type of regularisation across three distinct data-generating mechanisms: (1) without underlying features, i.e., where all of the data is needed, (2) with few relevant variables, (3) with few relevant features. We used 214 training samples and 1024 test samples, with a data dimensionality of $d=20$ and $k=5$ relevant features. The standard deviation of additive Gaussian noise was set to $0.5$, and the data set was generated  20 times with different seeds. The covariates were always sampled uniformly on $[-1,1]^d$ but the response was generated in three different ways. In the ``no underlying structure'' data set, we had $ y = \sum_{a=1}^{d} \sin(X_a) + { \rm noise}$. In the ``few relevant variables'' data set, we had $ y = \sum_{a=1}^{k} \sin(x_a) + { \rm noise}$. In the ``few relevant features data set'', we sampled~$P$ a $d\times d$ matrix from the orthogonal group uniformly, truncated it to size $d \times k$ and the response was generated as $ y = \sum_{a=1}^{k} \sin((P^\top x)_a) + { \rm noise}$. The mean and standard deviation of the~$R^2$ score on the test set are reported.

When there is no underlying structure, all regularisers perform somewhat similarly. However, for data sets featuring relevant variables, the $\Omega_{{ \rm variable}}$ and $\Omega_{{ \rm concave \ variable}}$ regularisations shine, delivering superior performance. Similarly for the $\Omega_{{ \rm feature}}$ and $\Omega_{{ \rm concave \ feature}}$ regularisations on data sets with few relevant features. Remarkably, for data with underlying structure, the concave versions of both $\Omega_{{ \rm variable}}$ and $\Omega_{{ \rm feature}}$ regularisations outperform their non-concave counterparts. This demonstrates their superior ability to effectively select relevant information in the data while maintaining strong predictive power.

\begin{figure}
    \centering
    \includegraphics[width=\textwidth]{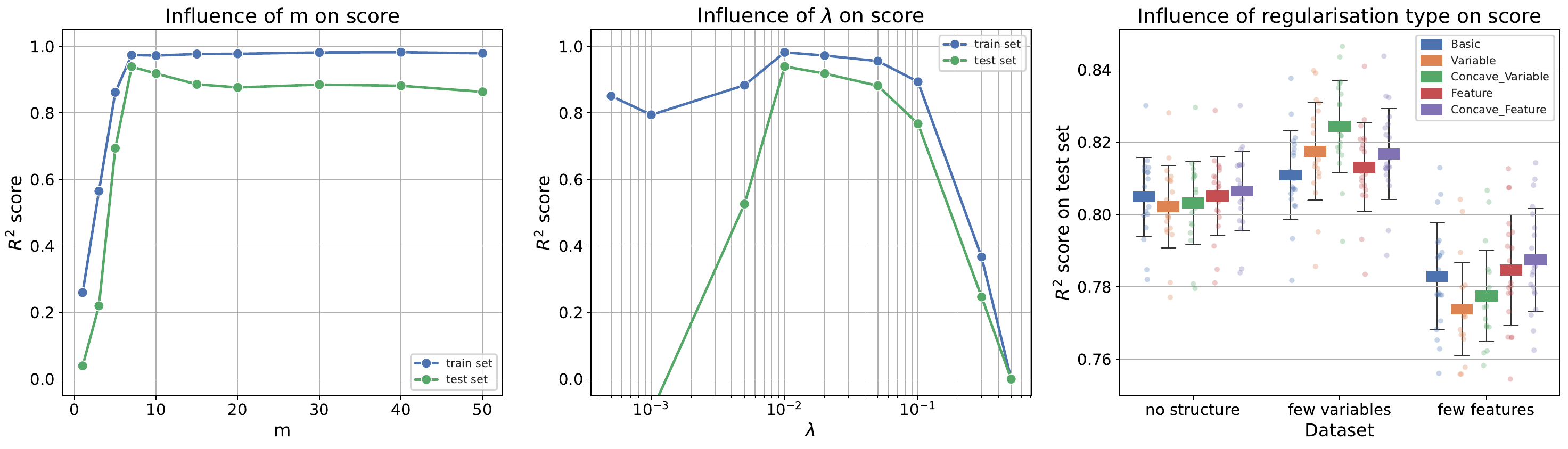}
    \caption{Influence of parameters: left: $m$, middle: $\lambda$, right: type of penalty.}
    \label{fig:exp2&3}
\end{figure}

\subsection{Experiment 4: Comparison to Neural Network on 1D Examples, Influence of Number of Particles/Width of Hidden Layer \texorpdfstring{$m$}{m}}
\label{sec:exp4}

In Experiment 4, we compare the learning capabilities of \textsc{BKerNN} against a simple neural network, \textsc{ReLUNN}. We study three distinct functions, corresponding to each row in Figure~\ref{fig:exp4}. In all rows, the training set is represented by small black crosses, while the target function is shown in blue. The first two columns depict \textsc{BKerNN} using two different numbers of particles: $m=1$ and $m=5$. The last three columns show results for \textsc{ReLUNN} with varying numbers of neurons in the hidden layer: 1, 5, and 32. See Appendix~\ref{app:num_exp4} for more experimental details.

Notably, \textsc{BKerNN} demonstrates great learning capabilities, successfully capturing the functions even with just one particle. Increasing the number of particles (second column) offers minimal additional benefit, underscoring \textsc{BKerNN}'s efficiency. In stark contrast, \textsc{ReLUNN} struggles significantly when limited to the same number of hidden neurons as \textsc{BKerNN}'s particles. However, once the hidden layer is expanded to 32 neurons, \textsc{ReLUNN} begins to show satisfactory learning capabilities. These results highlight \textsc{BKerNN}'s superior efficiency in learning functions with a minimal number of particles, outperforming \textsc{ReLUNN}, which requires a more complex architecture to achieve comparable performance.

\subsection{Experiment 5: Prediction Score and Feature Learning Score Against Growing Dimension and Sample Size, a Comparison of \textsc{BKerNN} with Brownian Kernel Ridge Regression and a ReLU Neural Network}
\label{sec:exp5}

In Experiment 5, we evaluate the performance of \textsc{BKerNN}, \textsc{B\textsc{KRR}}  and \textsc{ReLUNN} across varying sample sizes and dimensions on simulated data sets. The estimators are presented in Section~\ref{sec:intro_num_exp}. The $R^2$ and feature learning score used to assess performance are described in Equations~\eqref{eq:r2} and \eqref{eq:feature_score} respectively. The results are presented in Figure~\ref{fig:exp5}. For more details about the experiment, see Appendix~\ref{app:num_exp5}.

\begin{figure}
    \centering
    \includegraphics[width=\textwidth]{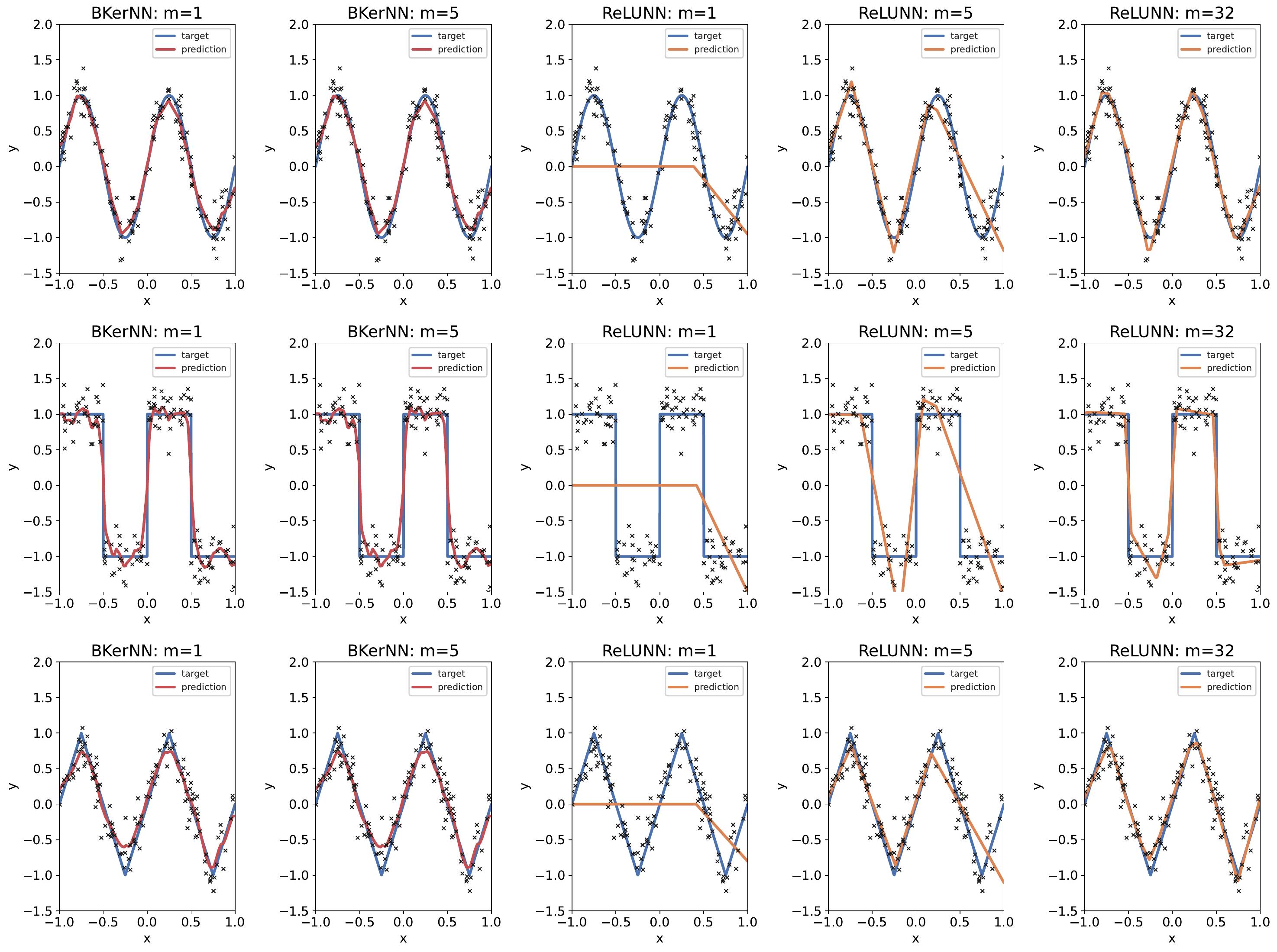}
    \caption{Comparison to neural network on 1D examples.}
    \label{fig:exp4}
\end{figure}

The two subplots on the top row of Figure~\ref{fig:exp5} show the effect of increasing the sample size while keeping the dimension fixed. In the two subplots of the bottom row, the sample size is fixed, and the dimension is increased. For each combination of sample size and dimension, ten data sets were generated. We display the two scores of each method on each data set, as well as the average score across data sets. The feature learning score for \textsc{BKRR} is not defined and, therefore, not displayed. The number of particles (for \textsc{BKerNN}) and hidden neurons (for \textsc{ReLUNN}) is fixed at 50 across all experiments.

For all the data sets, the covariates were uniformly sampled in $[-1,1]^d$, the underlying features matrix $P$ was uniformly sampled from the orthogonal group, then truncated to have~$k=3$ relevant features, and the response was set as $ y = \left|\sum_{a=1}^k \sin\left((P^\top x)_a\right)\right|$.

In the first two subplots, the dimension is fixed at 15. As the sample size increases, we observe improvements in the prediction scores of all three methods. However, the prediction score of \textsc{BKRR} improves at a much slower pace. Both \textsc{BKerNN} and \textsc{ReLUNN} achieve high prediction scores more rapidly, with \textsc{BKerNN} requiring fewer samples to do so. Notably, \textsc{BKerNN} excels in feature learning, effectively capturing the underlying feature space, while \textsc{ReLUNN} fails regardless of the number of samples.

In the last two subplots, where the sample size is fixed at 212, we notice a general decline in performance as the dimension increases. \textsc{B\textsc{KRR}} shows the most rapid deterioration because it cannot learn features, struggling significantly with higher dimensions. In contrast, \textsc{BKerNN} demonstrates remarkable resilience to increasing dimensionality, maintaining better performance compared to the other methods. \textsc{ReLUNN} falls somewhere in between, neither as robust as \textsc{BKerNN} nor as weak as \textsc{B\textsc{KRR}}. Similarly, for the feature learning score, both \textsc{BKerNN} and \textsc{ReLUNN} show decreased performance, but \textsc{BKerNN} is slightly less affected, underscoring its ability to handle high-dimensional data.

\begin{figure}
    \centering
   \includegraphics[width=\textwidth]{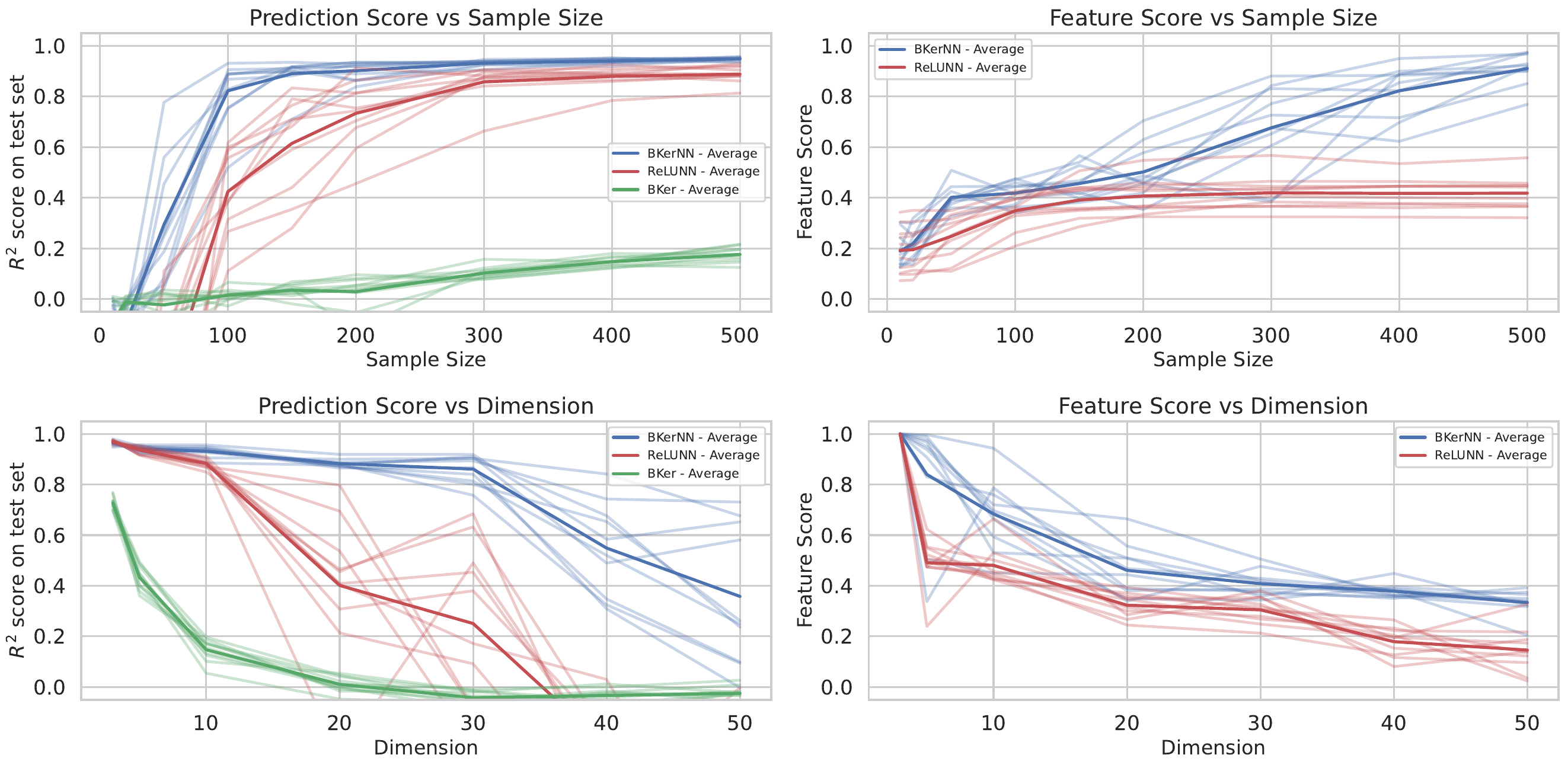}
    \caption{Performance comparison across varying sample sizes and dimensions.}
    \label{fig:exp5}
\end{figure}

\subsection{Experiment 6: Comparison on Real Data Sets Between \textsc{BKerNN}, Brownian Kernel Ridge Regression and a ReLU Neural Network}
\label{sec:numerical_real_data set}

In Experiment 6, we evaluate the $R^2$ scores, defined in Equation~\eqref{eq:r2}, of four methods: \textsc{BKRR}, \textsc{BKerNN} with concave variable regularisation, \textsc{BKerNN} with concave feature regularisation, and \textsc{ReLUNN}, across 17 real-world data sets. These data sets were obtained from the tabular benchmark numerical regression suite via the OpenML platform, as described by \citet{grinch}. Each data set was processed to include only numerical variables and rescaled to have centred covariates with standard deviation equal to one. The data sets were uniformly cropped to contain 400 training samples and 100 testing samples \red{(except for a few datasets, see Appendix~\ref{app:num_exp6})}, with dimensionality varying across data sets as shown in Figure~\ref{fig:exp6}.  For both \textsc{BKerNN} and \textsc{ReLUNN}, the number of particles or hidden neurons was set to twice the dimension of each data set \red{(except for the dataset with the highest dimension, ``semeion'', $d=256$, where it was fixed to a 100)}, while the training parameters were fixed. Details are available in Appendix~\ref{app:num_exp6}.

The results indicate that \textsc{B\textsc{KRR}} often performs the worst among all methods. In contrast, \textsc{BKerNN} with concave feature regularisation and \textsc{ReLUNN} frequently emerge as the best estimators, performing similarly well on average across the various data sets.

\begin{figure}
    \centering
    \includegraphics[width=\textwidth]{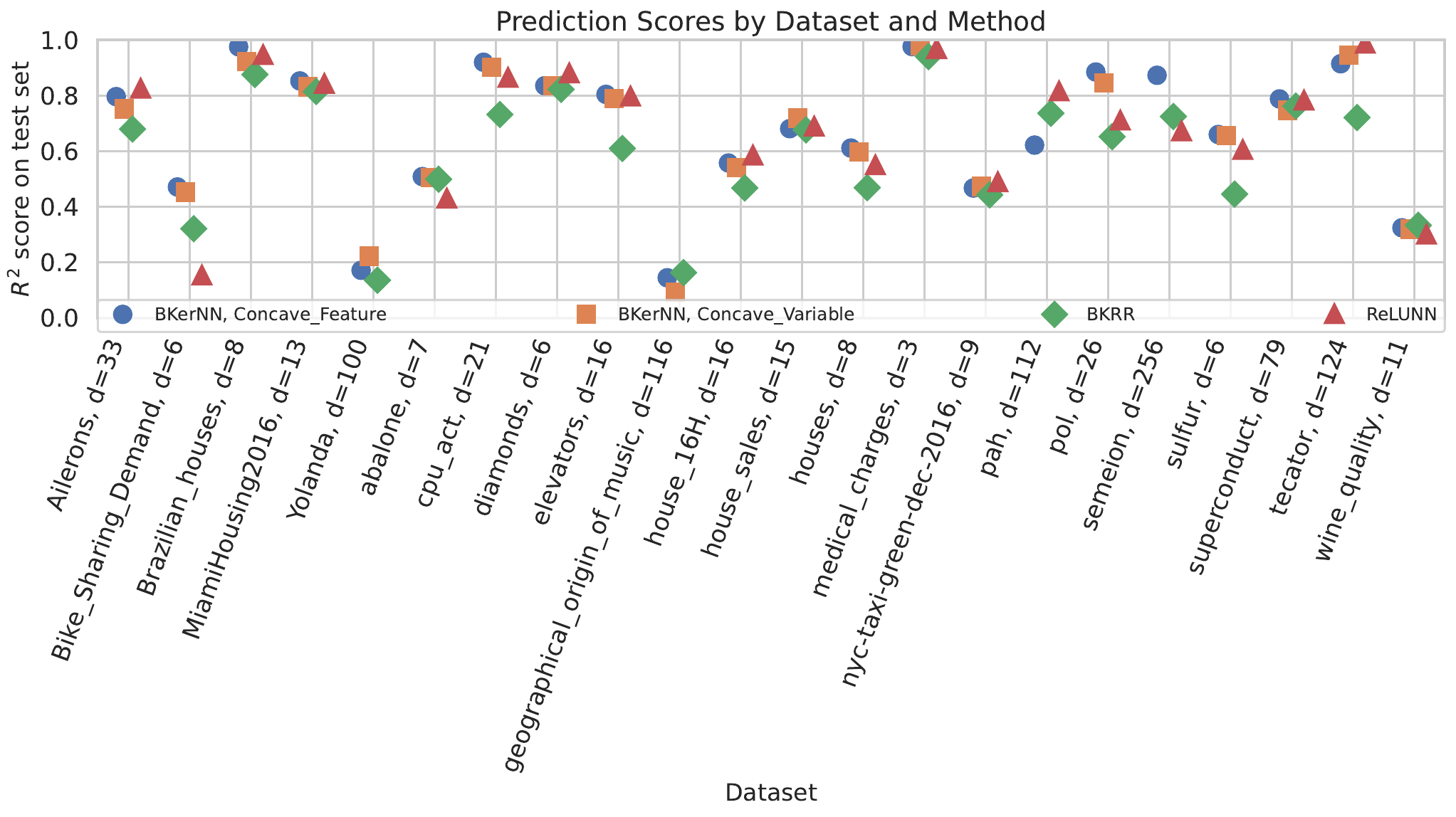}
    \caption{Comparison of $R^2$ scores on real data sets.}
    \label{fig:exp6}
\end{figure}

\section{Conclusion}
\label{sec:conclusion}
To conclude, we have introduced a novel framework for feature learning and function estimation in supervised learning, termed Brownian kernel neural network (\textsc{BKerNN}). By leveraging regularised empirical risk minimisation over averages of Sobolev spaces on one-dimensional projections of the data, we established connections to kernel ridge regression and infinite-width one-hidden layer neural networks. We provide a \red{particle-based} computational method for \textsc{BKerNN}, emphasising the importance of the positive homogeneity of the Brownian kernel. Through rigorous theoretical analysis, we demonstrated that, in the well-specified setting for subgaussian data, \textsc{BKerNN} achieves convergence of its expected risk to the minimal risk with explicit rates, potentially independent of the data dimension, underscoring the efficacy of our approach. We have extensively discussed the relationship between the space of functions we propose and other classical functions spaces. Numerical experiments across simulated scenarios and real data sets confirm \textsc{BKerNN}'s superiority over traditional kernel ridge regression and competitive performance with neural networks employing ReLU activations, achieved with fewer particles or hidden neurons. Future research directions include the development of more efficient algorithms for the computation of the estimator, improved analysis of the Gaussian complexity, and theoretical investigation of other penalties.


\acks{B. Follain would like to thank Adrien Taylor and David Holzmüller for fruitful discussions regarding this work. We acknowledge support from the French government under the management of the Agence Nationale de
la Recherche as part of the “Investissements d’avenir” program, reference ANR-19-P3IA-0001 (PRAIRIE
3IA Institute). We also acknowledge support from the European Research Council (grant SEQUOIA
724063).}


\vskip 0.2in

\appendix
\section{Extra Lemmas and Proofs}
\label{app:proofs}
In this appendix, we present and/or prove some of the results needed in the main text.

\subsection{Proofs of Section~\ref{sec:charact_f_infty} Lemmas}
\label{sec:proof_lemma_caracterisation_f_infty}

Here we give the proofs of the lemmas describing characteristics of the function space $\mathcal{F}_\infty$.

\subsubsection{Proof of Lemma~\ref{lemma:caracterisation_f_infty}}
\label{proof:lemma:caracterisation_f_infty}
\begin{proof}[Proof of Lemma~\ref{lemma:caracterisation_f_infty}]
We first check that $\mathcal{F}_\infty$ is a vector space.

\red{Let $f \in \mathcal{F}_\infty$, a constant $c$ and a sequence of measures $(\mu_i)_{i \in \mathbb{N}}$ defining $f$ (i.e. $\forall i \in \mathbb{N}, f = c + \int_{\mathcal{H}\times \mathcal{S}^{d-1}} g(w^\top \cdot) {\rm d}\mu_i(g,w) )$ such that $\Omega_0(f) = \lim_{i \to \infty} \int_{\mathcal{H}\times \mathcal{S}^{d-1}} \|g\|_\mathcal{H} {\rm d} \mu_i (g,w)$. Recall that $c$ is unique for a given $f$ as it is equal to $f(0)$. }

\red{Then, let $\tau \in \mathbb{R}$. We have $\forall i \in \mathbb{N}, \tau f = \tau c + \int_{\mathcal{H}\times \mathcal{S}^{d-1}} \tau g(w^\top \cdot) {\rm d} \mu_i(g,w) = \tau c + \int_{\mathcal{H}\times \mathcal{S}^{d-1}} \tilde{g}(w^\top \cdot) {\rm d} \tilde{\mu_i}(\tilde{g},w)$, with the change of variable $\tilde{g} = \tau g$ inducing the change on the measure ${\rm d}\tilde{\mu_i}(\tilde{g},w) = {\rm d }\mu_i( \tilde{g}/\tau, w) $ and $\tilde{\mu_i}$ is still a probability measure  (using the pushforward measure technique, see \citet{peyré2019computational}) . We also see that for any $i$, 
\begin{equation*}
  \int_{\mathcal{H}\times \mathcal{S}^{d-1}} \|\tilde{g}\|_\mathcal{H} {\rm d}\tilde{\mu_i}(\tilde{g},w) = |\tau| \int_{\mathcal{H}\times\mathcal{S}^{d-1}} \|g\|_\mathcal{H} {\rm d}\mu_i(g,w), 
\end{equation*}
and hence since $|\tau|\Omega_0(f) = |\tau |\lim_{i \to \infty} \int_{\mathcal{H}\times \mathcal{S}^{d-1}} \|g\|_\mathcal{H} {\rm d} \mu_i (g,w)$, we have 
\begin{equation*}
\lim_{i \to \infty}  \int_{\mathcal{H}\times \mathcal{S}^{d-1}} \|\tilde{g}\|_\mathcal{H} {\rm d}\tilde{\mu_i}(\tilde{g},w) = |\tau| \Omega_0(f)    
\end{equation*}
and therefore $\Omega_0(\tau f) < \infty$ and $\Omega_0(\tau f) \leq |\tau| \Omega_0(f)$, which means that $\tau f$ belongs to $\mathcal{F}_\infty$. Now we can also use this property to obtain that $\Omega_0(f) = \Omega_0(\frac{1}{\tau} \tau f) \leq \frac{1}{|\tau|} \Omega_0(\tau f) \leq \Omega_0(f)$ and hence $\Omega_0(\tau f)$ is actually equal to $|\tau| \Omega_0(f)$.}

\red{Now let $f_1, f_2 \in \mathcal{F}_\infty$, defined respectively using $c_1, c_2 \in \mathbb{R}$ and two sequences of measures $(\mu_i^{(1)})_{i \in \mathbb{N}}, (\mu_i^{(2)})_{i \in \mathbb{N}}$ such that $\Omega_0(f_1) = \lim_{i \to \infty}  \int_{\mathcal{H}\times \mathcal{S}^{d-1}} \|g\|_\mathcal{H} {\rm d} \mu_i^{(1)} (g,w)$ and $\Omega_0(f_2) = \lim_{i \to \infty}  \int_{\mathcal{H}\times \mathcal{S}^{d-1}} \|g\|_\mathcal{H} {\rm d} \mu_i^{(2)} (g,w)$. Now for any $i \in \mathbb{N}$, we have 
\begin{equation*}
    (f_1 + f_2)(\cdot) = c_1 + c_2 + \int_{\mathcal{H}\times \mathcal{S}^{d-1}} g(w^\top \cdot) {\rm d}\mu_i^{(1)}(g,w) + \int_{\mathcal{H}\times \mathcal{S}^{d-1}}g(w^\top \cdot) {\rm d}\mu_i^{(2)}(g,w),
\end{equation*}
and hence   $(f_1 + f_2)(\cdot)  = (c_1 + c_2) +  \int_{\mathcal{H} \times \mathcal{S}^{d-1}} g(w^\top \cdot ) {\rm d}\left(\mu_{i}^{(1)} + \mu_i^{(2)}\right)(g,w)$ and moreover,  $(f_1 + f_2)(\cdot)  = (c_1 + c_2) +  \int_{\mathcal{H} \times \mathcal{S}^{d-1}} 2g(w^\top \cdot ) {\rm d}\left(\frac{\mu_{i}^{(1)} + \mu_i^{(2)}}{2}\right)(g,w)$.}

\red{Using a change of variable as we did for $\tau f$, this means that we can write $f_1+f_2$ using a probability measure and a real constant, in the same way as the functions in $\mathcal{F}_\infty$. For $f_1+f_2$ to belong to $\mathcal{F}_\infty$, we now need to check that $\Omega_0(f_1+f_2)$ is well defined. }

\red{We first consider the function $ \tilde{f} :=\int_{\mathcal{H} \times \mathcal{S}^{d-1}} g(w^\top \cdot ) {\rm d}\left(\frac{\mu_{i}^{(1)} + \mu_i^{(2)}}{2}\right)(g,w)$ (for any $i \in \mathbb{N})$. We have that 
\begin{equation*}
    \lim_{i \to \infty} \int_{\mathcal{H}\times \mathcal{S}^{d-1}} \|g\|_\mathcal{H} {\rm d }\left(\frac{\mu_{i}^{(1)} + \mu_i^{(2)}}{2}\right)(g,w) = \frac{\Omega_0(f_1) + \Omega_0(f_2)}{2},
\end{equation*}
by simply splitting the integral. Therefore the function $\tilde{f}$ is such that $\Omega_0(\tilde{f}) < \infty$ and belongs to $\mathcal{F}_\infty$ with $\Omega_0(\tilde{f}) \leq  \frac{\Omega_0(f_1) + \Omega_0(f_2)}{2}$. It suffices then to notice that $f_1 + f_2 = c_1 + c_2 + 2 \tilde{f}$ (and that the intercept has no impact on the definition of $\Omega_0$) to obtain that $\Omega_0(f_1 + f_2)$ is well defined and that $\Omega_0(f_1 + f_2) \leq \Omega_0(f_1) + \Omega_0(f_2)$, hence $f_1 + f_2$ belongs to $\mathcal{F}_\infty$.}

This yields that $\mathcal{F}_\infty$ is a vector space. We also see that $\Omega(f_1 + f_2) = \max(f_1(0) + f_2(0), \Omega_0(f_1+ f_2)) \leq \Omega(f_1) + \Omega(f_2)$. Since  $\Omega(f) = 0 \iff f=0$, we have that  $\Omega$ is a norm on $\mathcal{F}_\infty$.

\red{We now check the Hölder continuity property. Let $\mu$ be any measure defining $f \in \mathcal{F}_\infty$ and $c = f(0)$, then}   
\begin{align*}
f(x)-f(x^\prime)  &= c +  \int_{\red{\mathcal{H}\times \mathcal{S}^{d-1}}} \red{g}(w^\top x) {\rm d}\mu(\red{g,w}) - c - \int_{\red{\mathcal{H}\times \mathcal{S}^{d-1}}} \red{g}(w^\top x^\prime) {\rm d}\mu(\red{g,w}) \\
  & =  \int_{\red{\mathcal{H}\times \mathcal{S}^{d-1}}} \langle \red{g}, k_{w^\top x} - k_{w^\top x^\prime}\rangle {\rm d}\mu(\red{g,w}) \\
  |f(x)-f(x^\prime)| &\leq \int_{\red{\mathcal{H} \times \mathcal{S}^{d-1}}} \|\red{g}\|_\mathcal{H} \|k_{w^\top x} - k_{w^\top x^\prime}\|_\mathcal{H} {\rm d}\mu(\red{g,w}) \\
  & \leq \int_{\red{\mathcal{H} \times \mathcal{S}^{d-1}}} \|\red{g}\|_\mathcal{H} \sqrt{|w^\top (x-x^\prime)|} {\rm d}\mu(\red{g,w}) \\
  & \leq \int_{\red{\mathcal{H} \times \mathcal{S}^{d-1}}} \|\red{g}\|_\mathcal{H} \sqrt{\|x-x^\prime\|^*} {\rm d}\mu(\red{g,w})  \leq  \Omega_0(f) \sqrt{\|x-x^\prime\|^*}.
  \end{align*}
\end{proof}
\vspace{0.0em}
\subsubsection{Proof of Lemma~\ref{lemma:caracterisation_f_infty_2}}
\label{proof:lemma:caracterisation_f_infty_2}
\begin{proof}[Proof of Lemma~\ref{lemma:caracterisation_f_infty_2}] Let us assume now that we only consider functions $f$ with support on the ball with centre $0$, radius $R$ and norm $\|\cdot\|^*$, which we denote $B(0,R)$. Then we can actually consider the functions \red{$g$} which define $\mathcal{F}_\infty$ to belong to $\mathcal{H}^\prime := \{ g: \mathbb{R} \to \mathbb{R} \dim g(0)=0, \int_{-R}^R (g^\prime(t))^2 {\rm d}t \}$, and it is still a RKHS with the same reproducing kernel. \red{Let $f \in \mathcal{F}_\infty, f = c  + \int_{\mathcal{H} \times \mathcal{S}^{d-1}} g(w^\top \cdot) {\rm d}\mu(g,w)$.} Then, because we have restrained on the ball and we are continuous, $f$ is necessarily in $L_1(B(0,R))$ (set of integrable functions) and in $L_2(B(0,R))$. It has a Fourier decomposition, with 
\begin{align*}
    f(x) = \frac{1}{(2\pi)^d}\int_{\mathbb{R}^d} \hat{f}(\omega) e^{i \omega^\top x} {\rm d } \omega,
\end{align*}
and then we have
\begin{align*}
    \Omega_0(f) \leq \frac{1}{(2\pi)^d} \int_{\mathbb{R}^d} |\hat{f}(\omega)| \Omega_0(e^{i \omega^\top x}) {\rm d }\omega
\end{align*}
and we can then study $\Omega_0(e^{i \omega^\top x})$. 

We have
$ e^{i \omega^\top x} = g_\omega\big( \frac{\omega}{\|\omega\|}^\top x\big)$ with
$g_\omega: t \in [-R, R] \to e^{i t\|\omega\|}$ which belongs to (the complex version of) $\mathcal{H}$, with $\|g_\omega\|_\mathcal{H} = \sqrt{\int_{-R}^R \|\omega\|^2 |e^{i t\|\omega\|}|^2 {\rm d }t} \leq \sqrt{2R}\|\omega\| $.

This yields 
\begin{equation*}
    \Omega_0(f) \leq \frac{\sqrt{2R}}{(2\pi)^d} \int_{\mathbb{R}^d} |\hat{f}(\omega)| \cdot \|\omega\| {\rm d }\omega .
\end{equation*}
\end{proof}
\vspace{0.0em}
\subsection{Proofs of Section~\ref{sec:learning_the_kernel_or_NN} Lemmas}
\label{sec:proof_lemma_rewrite_problem}

In this section we present the proof of lemmas used to transform the optimisation problem defining \textsc{BKerNN}.

\subsubsection{Proof of Lemma~\ref{lemma:rewrite_problem}}
\label{proof:lemma:rewrite_problem}
\begin{proof}[Proof of Lemma~\ref{lemma:rewrite_problem}]
Our goal is to transform Equation~\eqref{eq:original_optim}. We begin with the following trick for the $m$ particles setting
\begin{equation*} 
\frac{1}{m} \sum_{j=1}^m \|g_j\|_{\mathcal{H}} = \inf_{\beta \in \mathbb{R}^m_+} \frac{1}{2m} \sum_{j=1}^m \left( \frac{\|g_j\|_{\mathcal{H}}^2}{\beta_j} + \beta_j \right).
\end{equation*}

\red{This allows us to rewrite Equation~\ref{eq:original_optim} in the following way
\begin{equation*}
    \min_{c \in \mathbb{R}, w_1, \ldots, w_m \in \mathcal{S}^{d-1}, g_1, \ldots, g_m \in \mathcal{H}, \beta \in \mathbb{R}^m_+} \frac{1}{n} \sum_{i=1}^n \ell\bigg(y_i, c + \frac{1}{m} \sum_{j=1}^m g_j(w_j^\top x_i)\bigg) + \lambda\frac{1}{2m} \sum_{j=1}^m \left( \frac{\|g_j\|_{\mathcal{H}}^2}{\beta_j} + \beta_j \right).
\end{equation*}}

\red{Now if we fix $(w_j)_{j \in [m]}$ and $(\beta_j)_{j \in [m]}$ in the equation above, the minimisation problem is equivalent to }
\begin{equation}
\label{eq:optim_only_g}
    \min_{c \in \mathbb{R}, g_1, \ldots, g_m \in \mathcal{H}} \frac{1}{n} \sum_{i=1}^n \ell(y_i, c + \frac{1}{m} \sum_{j=1}^m g_j(w_j^\top x_i)) + \frac{\lambda}{2} \frac{1}{m} \sum_{j=1}^m \frac{\|g_j\|_{\mathcal{H}}^2}{\beta_j}.
\end{equation}

Using the representer theorem \citep{representer_theorem}, we express each $x \to g_j(w_j^\top x)$ as
\begin{equation*}
    x \to \sum_{i=1}^n \alpha_i^{(j)} k^{(B)}(w_j^\top x_i, w_j^\top x),
\end{equation*}
which leads to
\begin{equation*}
    \|g_j\|^2_{\mathcal{H}} = \sum_{i, i^\prime =1}^n \alpha_i^{(j)} \alpha_{i^\prime}^{(j)} k^{(B)}(w_j^\top x_i, w_j^\top x_{i^\prime}).
\end{equation*}

Rewriting the norm and evaluation in kernel form with $K^{(w_j)}_{i, i^\prime} = k^{(B)}(w_j^\top x_i, w_j^\top x_{i^\prime})$, we obtain
\begin{equation*}
    \|g_j\|^2_{\mathcal{H}} = (\alpha^{(j)})^\top K^{(w_j)} \alpha^{(j)},
\end{equation*}
and
\begin{equation*}
    g_j(w_j^\top x_i) = (K^{(w_j)} \alpha^{(j)})_i.
\end{equation*}

Thus, we transform Equation~\eqref{eq:optim_only_g} into
\begin{equation*}
\min_{c \in \mathbb{R}, \red{\alpha^{(1)}, \ldots, \alpha^{(m)} \in \mathbb{R}^{\red{n}}}} \frac{1}{n} \sum_{i=1}^n \ell(y_i, \frac{1}{m} \sum_{j=1}^m (K^{(w_j)} \alpha^{(j)})_i + c) + \frac{\lambda}{2} \frac{1}{m} \sum_{j=1}^m \frac{ (\alpha^{(j)})^\top K^{(w_j)} \alpha^{(j)}}{\beta_j}.
\end{equation*}

We show that minimisation is attained for vectors $\alpha^{(j)}$ equal to $\beta_j \alpha$ for a single vector~$\alpha$. Consider the convex problem
\begin{equation*}
    \min_{\red{\alpha^{(1)}, \ldots, \alpha^{(m)} \in \mathbb{R}^n}} \frac{1}{2} \frac{1}{m} \sum_{j=1}^m \frac{(\alpha^{(j)})^\top K^{(w_j)} \alpha^{(j)}}{\beta_j},
\end{equation*}
subject to $\frac{1}{m} \sum_{j=1}^m K^{(w_j)} \alpha^{(j)} = z$ where $z \in \mathbb{R}^d$. We define the Lagrangian
\begin{equation*}
    \mathcal{L}(\alpha^{(1)}, \cdots, \alpha^{(m)}, \alpha) = \frac{1}{2} \frac{1}{m} \sum_{j=1}^m \frac{(\alpha^{(j)})^\top K^{(w_j)} \alpha^{(j)}}{\beta_j} + \alpha^\top \left(z - \frac{1}{m} \sum_{j} K^{(w_j)} \alpha^{(j)}\right).
\end{equation*}

By taking the differential of $\mathcal{L}$ with respect to $\alpha^{(j)}$ at the optimum, we get
\begin{equation*}
\label{eq:partial_alpha_j}
  \frac{\partial \mathcal{L}}{\partial \alpha^{(j)}} = \frac{1}{m} K^{(w_j)} \left(\frac{\alpha^{(j)}}{\beta_j} - \alpha \right) = 0.
\end{equation*}

The differential with respect to $\alpha$ yields that at the optimum, the constraint is verified, i.e., $z = \frac{1}{m} \sum_{j} K^{(w_j)} \alpha^{(j)}$. We note that for $\alpha^{(j)} = \beta_j \alpha$, all equations are satisfied, yielding the desired result.

We can then write Equation~\eqref{eq:original_optim} as
\begin{equation*}
    \min_{w_1, \ldots, w_m \in \mathbb{R}^d, c \in \mathbb{R}, \beta \in \mathbb{R}^m_+, \alpha \in \mathbb{R}^n} \frac{1}{n} \sum_{i=1}^n \ell(y_i, (K \alpha)_i + c) + \frac{\lambda}{2} \alpha^\top K \alpha + \frac{\lambda}{2} \frac{1}{m} \sum_{j=1}^m \beta_j,
\end{equation*}
with the constraints $\forall j \in [m], w_j \in \mathcal{S}^{d-1}$, and $K = \frac{1}{m} \sum_{j=1}^m \beta_j K^{(w_j)}$.

We notice that $\beta_j K^{(w_j)} = K^{(\beta_j w_j)}$ due to the positive homogeneity of the Brownian kernel. We therefore introduce the change of variable $\beta_j w_j = \tilde{w}_j$
\begin{equation*}
    \min_{\tilde{w}_1, \ldots, \tilde{w}_m \in \mathbb{R}^d, c \in \mathbb{R}, \alpha \in \mathbb{R}^n} \frac{1}{n} \sum_{i=1}^n \ell(y_i, (K \alpha)_i + c) + \frac{\lambda}{2} \alpha^\top K \alpha + \frac{\lambda}{2} \frac{1}{m} \sum_{j=1}^m \|\tilde{w}_j\|,
\end{equation*}
with $K = \frac{1}{m} \sum_{j=1}^m K^{(\tilde{w}_j)}$ and no constraint on the norm of $\tilde{w}_j$. For ease of exposition in the main text, we replace $\tilde{w}$ by $w$.
\end{proof}
\vspace{0.0em}
\subsubsection{Proof of Lemma~\ref{lemma:rewrite_problem_infinite}}
\label{proof:lemma:rewrite_problem_infinite}
\begin{proof}[Proof of Lemma~\ref{lemma:rewrite_problem_infinite}]
The proof follows the same steps as the proof of Lemma~\ref{lemma:rewrite_problem}, systematically replacing any $\frac{1}{m} \sum_{j=1}^m$ with the appropriate integral over $\red{\mathcal{H}\times \mathcal{S}^{d-1}}$ with respect to measure $\mu$. Before the change of variables, the problem is
\begin{equation*}
    \min_{\mu \in \mathcal{P}(\red{\mathcal{H}\times \mathcal{S}^{d-1}}), c \in \mathbb{R}, (\beta_{\red{g}})_{\red{g}} \in \mathbb{R}_+^{\red{\mathcal{H}}}, \alpha \in \mathbb{R}^n} \frac{1}{n} \sum_{i=1}^n \ell(y_i, (K\alpha)_i + c) + \frac{\lambda}{2} \alpha^\top K \alpha + \frac{\lambda}{2} \int_{\red{\mathcal{H}\times \mathcal{S}^{d-1}}} \red{\beta_g} {\rm d}\mu(\red{g,w}),
\end{equation*}
where $K = \int_{\red{\mathcal{H} \times \mathcal{S}^{d-1}}} \beta_{\red{g}} K^{(w)} \, d\mu(\red{g,w}) = \int_{\red{\mathcal{H}\times \mathcal{S}^{d-1}}}K^{(\beta_{\red{g}} w)} \, d\mu(\red{g,w})$. The change of variables $\beta_{\red{g}} w = \tilde{w}$ transforms the problem into
\begin{equation*}
    \min_{(\beta_{\red{g}})_{\red{g}} \in \mathbb{R}_+^{\red{\mathcal{H}}}, \nu \in \mathcal{P}(\{ \beta_{\red{g}} w, \red{g \in \mathcal{H}}, w \in \mathcal{S}^{d-1} \}), c \in \mathbb{R}, \alpha \in \mathbb{R}^n}  \frac{1}{n} \sum_{i=1}^n \ell(y_i, (K\alpha)_i + c) + \frac{\lambda}{2} \alpha^\top K \alpha + \frac{\lambda}{2} \int_{\mathbb{R}^d} \|\tilde{w}\| {\rm d} \nu(\tilde{w}),
\end{equation*}
with $K = \int_{\mathbb{R}^d} K^{(\tilde{w})} {\rm d}\nu(\tilde{w})$. We can consider the integral over $\mathbb{R}^d$ instead of $\{ \beta_{\red{g}} w, \red{g \in \mathcal{H}}, w \in \mathcal{S}^{d-1} \}$ by extending $\nu$ with $\nu(\mathbb{R}^d \setminus \{ \beta_{\red{g}} w, \red{g \in \mathcal{H}}, w \in \mathcal{S}^{d-1} \}) = 0$. This is equivalent to considering the minimum over $\nu \in \mathcal{P}(\mathbb{R}^d)$ instead of the minimum over $(\beta_{\red{g}})_{\red{g \in \mathcal{H}}} \in \mathbb{R}_+^{\red{\mathcal{H}}}$ and $\nu \in \mathcal{P}(\{ \beta_{\red{g}} w, \red{g \in \mathcal{H}}, w \in \mathcal{S}^{d-1} \})$, as we discuss now.

Indeed, the first minimum is smaller as it is considered over a larger space, but they are equal because both the norm $\|\cdot\|$ and the kernel $K$ are positively homogeneous. Hence, the problem finally becomes
\begin{equation*}
    \min_{\nu \in \mathcal{P}(\mathbb{R}^d), c \in \mathbb{R}, \alpha \in \mathbb{R}^n} \frac{1}{n} \sum_{i=1}^n \ell(y_i, (K\alpha)_i + c) + \frac{\lambda}{2} \alpha^\top K \alpha + \frac{\lambda}{2} \int_{\mathbb{R}^d} \|\tilde{w}\| {\rm d}\nu(\tilde{w}),
\end{equation*}
with $K = \int_{\mathbb{R}^d} K^{(\tilde{w})} {\rm d}\nu(\tilde{w})$. 
For ease of exposition in the main text, we replace $\tilde{w}$ by $w$.
\end{proof}
\vspace{0.0em}
\subsection{Proofs of Section~\ref{sec:optim_procedure} Lemmas}
\label{sec:proof_sec_optim_procedure}
In this section, we provide the proofs of the lemmas used to compute the estimator.

\subsubsection{Proof of Lemma~\ref{lemma:fixed_K}}
\label{proof:lemma:fixed_K}

\begin{proof}[Proof of Lemma~\ref{lemma:fixed_K}]
For a fixed $\alpha$, the optimal $c$ is given by $c = \frac{\mathds{1}^\top Y}{n} - \frac{\mathds{1}^\top K \alpha}{n}$. Substituting this back into the objective function, we obtain
\begin{align*}
    \min_{\alpha \in \mathbb{R}^n} \frac{1}{2n} \|\Pi Y - \Pi K \alpha \|_2^2 + \frac{\lambda}{2} \alpha^\top K \alpha,
\end{align*}
which is minimised for $\alpha$ satisfying $(K \Pi K + n \lambda K) \alpha = K \Pi Y$. We can further simplify this by observing that if $(\Pi K + n \lambda I) \alpha = \Pi Y$, then the previous condition is satisfied.

From the equation $n \lambda \alpha = \Pi Y - \Pi K \alpha$, we can deduce that $\Pi \alpha = \alpha$ because $\Pi^2 = \Pi$. Therefore, we can express $\alpha$ as $\alpha = \Pi \tilde{\alpha}$. Substituting this change of variable into the original problem, we define $\tilde{K} := \Pi K \Pi$ and $\tilde{Y} := \Pi Y$, transforming the problem into
\begin{align*}
    \min_{\tilde{\alpha} \in \mathbb{R}^n} \frac{1}{2n} \|\tilde{Y} - \tilde{K} \tilde{\alpha} \|_2^2 + \frac{\lambda}{2} \tilde{\alpha}^\top \tilde{K} \tilde{\alpha}.
\end{align*}

This is a standard kernel ridge regression problem (noting that $\tilde{K}$ is still a valid kernel matrix), for which the solution is known to be $\tilde{\alpha} = (\tilde{K} + n \lambda I)^{-1} \tilde{Y}$. We also have $\Pi \tilde{\alpha} = \tilde{\alpha}$, implying $\alpha = \tilde{\alpha}$ because one can show that $\mathds{1}^\top \tilde{\alpha} = 0$. To see why, note that $\mathds{1}^\top \tilde{\alpha} = \langle \mathds{1}, \tilde{\alpha} \rangle = \langle (\tilde{K} + n \lambda I)^{-1} \mathds{1}, \tilde{Y} \rangle$. Since $(\tilde{K} + n \lambda I)^{-1} \mathds{1}$ is proportional to $\mathds{1}$ (as $\mathds{1}$ is an eigenvector of $\tilde{K} + n \lambda I$ and its inverse), and $\langle \tilde{Y}, \mathds{1} \rangle = 0$, we obtain the desired result.

Finally, we verify the optimal condition $(K \Pi K + n \lambda K) \alpha = K \Pi Y$. Given $(\Pi K \Pi + n \lambda I) \tilde{\alpha} = \Pi Y$ by definition, multiplying by $K$ yields $(K \Pi K \Pi + n \lambda K) \tilde{\alpha} = K \Pi Y$. Since $\tilde{\alpha} = \alpha = \Pi \alpha$, the desired result follows.
\end{proof}
\vspace{0.0em}
\subsubsection{Proof of Lemma~\ref{lemma:derivative_of_G}}
\label{proof:lemma:derivative_of_G}

\begin{proof}[Proof of Lemma~\ref{lemma:derivative_of_G}]
First, we compute the derivative of $G = \frac{\lambda}{2}\tilde{Y}^\top (\tilde{K} + \lambda n I)^{-1} \tilde{Y}$ with respect to $w_j$
\begin{align}
\label{eq:G_derivative}
    \frac{\partial G}{\partial w_j} &= \sum_{i, i^\prime = 1}^n \frac{\partial G}{\partial K_{i, i^\prime}} \frac{\partial K_{i, i^\prime}}{\partial w_j} \nonumber \\
    &= \frac{1}{m} \sum_{i, i^\prime = 1}^n \frac{\partial G}{\partial K_{i, i^\prime}} \frac{ \left( { \rm sign}(w_j^\top x_i) x_i + { \rm sign}(w_j^\top x_{i^\prime}) x_{i^\prime} - { \rm sign}(w_j^\top (x_i - x_{i^\prime}))(x_i - x_{i^\prime}) \right)}{2}.
\end{align}

We know that
\begin{align*}
    \frac{\partial G}{\partial (\tilde{K} + \lambda n I)} &= - \frac{\lambda}{2} (\tilde{K} + \lambda n I)^{-1} \tilde{Y} \tilde{Y}^\top (\tilde{K} + \lambda n I)^{-1},
\end{align*}
thus
\begin{align*}
    \frac{\partial G}{\partial K_{i, i^\prime}} &= \sum_{l, k} \frac{\partial G}{\partial (\tilde{K} + \lambda n I)_{l, k}} \frac{\partial (\Pi K \Pi + \lambda n I)_{l, k}}{\partial K_{i, i^\prime}} \\
    &= \sum_{l, k} - \frac{\lambda}{2} ((\tilde{K} + \lambda n I)^{-1} \tilde{Y} \tilde{Y}^\top (\tilde{K} + \lambda n I)^{-1})_{l, k} \Pi_{l, i} \Pi_{i^\prime, k} \\
    &= - \frac{\lambda}{2} \left( \Pi (\tilde{K} + \lambda n I)^{-1} \tilde{Y} \tilde{Y}^\top (\tilde{K} + \lambda n I)^{-1} \Pi \right)_{i, i^\prime} \\
    &= - \frac{\lambda}{2} (\Pi (\tilde{K} + \lambda n I)^{-1} \tilde{Y})_i (\Pi (\tilde{K} + \lambda n I)^{-1} \tilde{Y})_{i^\prime}.
\end{align*}

Substituting this back into Equation~\eqref{eq:G_derivative} and introducing $S_j \in \mathbb{R}^{n \times n}$ with $(S_j)_{i, i^\prime} = ({ \rm sign}(w_j^\top x_i) x_i + { \rm sign}(w_j^\top x_{i^\prime}) x_{i^\prime} - { \rm sign}(w_j^\top (x_i - x_{i^\prime}))(x_i - x_{i^\prime})) / 2$, we get
\begin{align*}
    \frac{\partial G}{\partial w_j} 
    &= -\frac{\lambda}{2m} \sum_{i, i^\prime = 1}^n (\Pi (\tilde{K} + \lambda n I)^{-1} \tilde{Y})_i (\Pi (\tilde{K} + \lambda n I)^{-1} \tilde{Y})_{i^\prime} (S_j)_{i, i^\prime} \\
    &= -\frac{\lambda}{2m} \trace\left( (\Pi (\tilde{K} + \lambda n I)^{-1} \tilde{Y})^\top S_j (\Pi (\tilde{K} + \lambda n I)^{-1} \tilde{Y}) \right) \\
    &= -\frac{\lambda}{2m} \trace\left( ((\tilde{K} + \lambda n I)^{-1} \tilde{Y})^\top \Pi S_j \Pi ((\tilde{K} + \lambda n I)^{-1} \tilde{Y}) \right).
\end{align*}

This implies that we can replace $(S_j)_{i, i^\prime}$ with the $i$-th, $i^\prime$-th component of any matrix with the same centred version, such as $\tilde{S}_j$ where $(\tilde{S}_j)_{i, i^\prime} = -{ \rm sign}(w_j^\top (x_i - x_i^\prime))(x_i - x_i^\prime)$, yielding the desired result.
\end{proof}
\vspace{0.0em}
\subsubsection{Proof of Lemma~\ref{lemma:prox}}
\label{proof:lemma:prox}

\begin{proof}[Proof of Lemma~\ref{lemma:prox}]
We consider each penalty separately.
\begin{enumerate}
    \item For $\Omega_{{ \rm basic}}(W) = \frac{1}{2m} \sum_{j=1}^m \|w_j\|$, the penalty corresponds to a group Lasso penalty on $W \in \mathbb{R}^{d \times m}$, where the groups are the columns. The proximal operator is given by:
    \[
    \big({\rm prox}_{\lambda \gamma \Omega}(W)\big)_j = \bigg( 1 - \frac{\lambda \gamma}{2m} \frac{1}{\|w_j\|} \bigg)_+ w_j,
    \]
    as detailed in \cite[Section 3.3]{livreviolet}.
    \item For $\Omega_{{ \rm variable}}(W) = \frac{1}{2} \sum_{a=1}^{d} \big( \frac{1}{m} \sum_{j=1}^m |(w_j)_a|^2 \big)^{1/2}$, this is a group Lasso setting where the groups are the rows of $W$. The proximal operator is:
    \[
    ({\rm prox}_{\lambda \gamma \Omega}(w))^{(a)} = \bigg( 1 - \frac{\lambda \gamma}{2\sqrt{m}} \frac{1}{\|W^{(a)}\|_2} \bigg)_+ W^{(a)},
    \]
    also found in \citet[Section 3.3]{livreviolet}.
    \item For $\Omega_{{ \rm feature}}(W) = \frac{1}{2} \trace\big( \big( \frac{1}{m} \sum_{j=1}^m w_j w_j^\top \big)^{1/2} \big)$, this penalty corresponds to a Lasso penalty on the singular values. Given $W = U S V^\top$ (SVD), we have:
    \[
    {\rm prox}_{\lambda \gamma \Omega}(W) = U \tilde{S} V^\top \quad { \rm with} \quad \tilde{S} = \bigg(1 - \frac{\lambda \gamma}{2\sqrt{m} |S|} \bigg)_+ S,
    \]
    using results from \citet[Section 3.3]{livreviolet}.
    \item For $\Omega_{{ \rm concave \ variable}}(W) = \frac{1}{2s} \sum_{a=1}^d \log\big( 1 + s \big( \frac{1}{m} \sum_{j=1}^m |(w_j)_a|^2 \big)^{1/2}\big)$, the loss is separable along the $d$ dimensions. Considering each $W^{(a)}$ separately, we compute the proximal operator:
    \begin{align*}
        {\rm prox}_{\frac{\lambda \gamma}{2s} \log(1 + \frac{s}{\sqrt{m}} \|\cdot\|_2)}(W^{(a)}) &= \min_{u^{(a)} \in \mathbb{R}^m} \frac{1}{2} \|W^{(a)} - u^{(a)}\|_2^2 + \frac{\lambda \gamma}{2s} \log(1 + \frac{s}{\sqrt{m}} \|u^{(a)}\|_2).
    \end{align*}
    The subgradients of $\mathcal{L}(u^{(a)}) := \frac{1}{2} \|W^{(a)} - u^{(a)}\|_2^2 + \frac{\lambda \gamma}{2s} \log(1 + \frac{s}{\sqrt{m}} \|u^{(a)}\|_2)$ are:
    \[
    \frac{\partial \mathcal{L}}{\partial u^{(a)}} = -(W^{(a)} - u^{(a)}) + \frac{\lambda \gamma}{2s} \frac{s}{\sqrt{m}} \frac{1}{1 + \frac{s}{\sqrt{m}} \|u^{(a)}\|_2} v^{(a)},
    \]
    where $\|v^{(a)}\|_2 \leq 1$ if $u^{(a)} = 0$, and otherwise $v^{(a)} = u^{(a)} / \|u^{(a)}\|_2$.

    For $u^{(a)} \neq 0$, there is a scalar $c \in \mathbb{R}^+$ such that $u^{(a)} = c W^{(a)}$, yielding:
    \[
    c \bigg( 1 + \frac{\lambda \gamma}{2 \sqrt{m}} \frac{1}{c \|W^{(a)}\|_2} \frac{1}{1 + \frac{s c}{\sqrt{m}} \|W^{(a)}\|_2} \bigg) = 1.
    \]
    This is a second-order polynomial in $c$ that can be solved explicitly. The determinant $\Delta$ is
    \[
    \Delta = \big( 1 - \frac{s}{\sqrt{m}} \|W^{(a)}\|_2 \big)^2 - 4 \left(\frac{\lambda \gamma}{2 \sqrt{m}} \frac{1}{\|W^{(a)}\|_2} - 1 \right) \frac{s}{\sqrt{m}} \|W^{(a)}\|_2.
    \]
    When $\Delta \leq 0$, the proximal operator is $u^{(a)} = 0$. Otherwise, it suffices to compare the two possible values of $c$ and choose the one for which $\mathcal{L}$ is the smallest.

    \item For $\Omega_{{ \rm concave \ feature}}(W) = \frac{1}{2s} \sum_{a=1}^d \log\big( 1 + \frac{s}{\sqrt{m}} \sigma_a(w_1, \ldots, w_n) \big)$, we combine the results of the third and fourth items above. The proximal operator is
    \[
    {\rm prox}_{\lambda \gamma \Omega}(W) = U \tilde{S} V^\top,
    \]
    where $\tilde{S}$ is obtained by replacing all $\|W^{(a)}\|_2$ by $\sigma_a$ in the computations of the proximal of $\Omega_{{ \rm concave \ variable}}$.
\end{enumerate}
\end{proof}
\vspace{0.0em}
\subsection{Extra Lemma and Proofs Related to Section~\ref{sec:stat_analysis} Except Section~\ref{sec:bound_true_risk}}
\label{sec:proof_sec_gaussian_complex}
Here we provided the proofs of the lemmas used to bound the Gaussian complexity.

\subsubsection{Proof of Lemma~\ref{lemma:rademacher_without_integral}}
\label{proof:lemma:rademacher_without_integral}
\begin{proof}[Proof of Lemma~\ref{lemma:rademacher_without_integral}]
Recall that 
\begin{equation*}
    G_n(\{ f \in \mathcal{F}_\infty, \Omega(f) \leq D \}) =  \mathbb{E}_{\varepsilon, \mathcal{D}_n}   \bigg( \sup_{f \in \mathcal{F}_\infty, \Omega_0(f) \leq D, c \leq D}  \frac{1}{n}\sum_{i=1}^n \varepsilon_i f(x_i) \bigg).
\end{equation*}

\red{We start by considering the quantity without any expectation. Using the definitions, we obtain}
\red{\begin{align*}
     \sup_{f \in \mathcal{F}_\infty, \Omega(f) \leq D}  \frac{1}{n}\sum_{i=1}^n \varepsilon_i f(x_i)  &= \sup_{|c| \leq D, \red{\ \mu \ \text{s.t.}\int_{\mathcal{H} \times \mathcal{S}^{d-1}}\|g\|_\mathcal{H}{\rm d}\mu(g,w)} \leq D} \frac{1}{n}\sum_{i=1}^n \varepsilon_i \left( c+ \red{\int_{\mathcal{H} \times \mathcal{S}^{d-1}} g(w^\top x_i){\rm d}\mu(g,w)} \right) \\
     &=  D\frac{1}{n}\left|\sum_{i=1}^n \varepsilon_i \right| +\sup_{\ \mu \ \text{s.t.} \int_{\red{\mathcal{H} \times \mathcal{S}^{d-1}}}\red{\|g\|_\mathcal{H}{\rm d}\mu(g,w)} \leq D} \frac{1}{n}\sum_{i=1}^n \varepsilon_i \int_{\red{\mathcal{H}\times \mathcal{S}^{d-1}}} \langle \red{g}, k^{(B)}_{w^\top x_i}\rangle{\rm d}\red{\mu(g,w)}  
\end{align*}}
\red{where the last equality is obtained by splitting the sup and taking $c = D {\rm sign  }\left(\sum_{i=1}^n \varepsilon_i \right)$, which is the explicit value attaining the supremum}. For the first term, we can simply bound $ \mathbb{E}_{\varepsilon} \bigg( D \frac{1}{n}\left|\sum_{i=1}^n \varepsilon_i \right| \bigg) \leq  D\frac{1}{\sqrt{n}}$ using the gaussianity of $\varepsilon$. For the second term of the equation right above, we then have equality to

\red{\begin{align*}
&\sup_{\ \mu \ \text{s.t.} \int_{\mathcal{H}\times \mathcal{S}^{d-1}}\|g\|_\mathcal{H}{\rm d}\mu(g,w) \leq D}  \int_{\mathcal{H}\times \mathcal{S}^{d-1}} \big| \frac{1}{n}\sum_{i=1}^n \varepsilon_i \langle g, k^{(B)}_{w^\top x_i}   \rangle \big|  {\rm d}\mu(g, w)   \\
& \leq \sup_{\ \mu \ \text{s.t.} \int_{\mathcal{H}\times \mathcal{S}^{d-1}}\|g\|_\mathcal{H}{\rm d}\mu(g,w) \leq D}  \int_{\mathcal{H}\times \mathcal{S}^{d-1}} \mathds{1}_{g \neq 0}\|g\|_\mathcal{H} \big| \frac{1}{n}\sum_{i=1}^n \varepsilon_i \langle \frac{g}{\|g\|_\mathcal{H}}, k^{(B)}_{w^\top x_i}   \rangle \big|  {\rm d}\mu(g, w)  \\
& \leq \sup_{\ \mu \ \text{s.t.} \int_{\mathcal{H}\times \mathcal{S}^{d-1}}\|g\|_\mathcal{H}{\rm d}\mu(g,w) \leq D}  \int_{\mathcal{H}\times \mathcal{S}^{d-1}} \mathds{1}_{g \neq 0}\|g\|_\mathcal{H} \sup_{w \in \mathcal{S}^{d-1}, \| \tilde{g}\|_\mathcal{H} =1}\big| \frac{1}{n}\sum_{i=1}^n \varepsilon_i \langle \tilde{g}, k^{(B)}_{w^\top x_i}   \rangle \big|  {\rm d}\mu(g, w)   \\
& \leq  \sup_{\ \mu \ \text{s.t.} \int_{\mathcal{H}\times \mathcal{S}^{d-1}}\|g\|_\mathcal{H}{\rm d}\mu(g,w) \leq D}  \sup_{w \in \mathcal{S}^{d-1}, \| \tilde{g}\|_\mathcal{H} =1}\big| \frac{1}{n}\sum_{i=1}^n \varepsilon_i \langle \tilde{g}, k^{(B)}_{w^\top x_i}   \rangle \big| \int_{\mathcal{H}\times \mathcal{S}^{d-1}} \mathds{1}_{g \neq 0}\|g\|_\mathcal{H}  {\rm d}\mu(g, w)  \\
& \leq  D \sup_{w \in \mathcal{S}^{d-1}, \| \tilde{g}\|_\mathcal{H} =1}\big| \frac{1}{n}\sum_{i=1}^n \varepsilon_i \langle \tilde{g}, k^{(B)}_{w^\top x_i}   \rangle \big|   \leq  D \sup_{w \in \mathcal{S}^{d-1}, \| \tilde{g}\|_\mathcal{H} \leq 1}\big| \frac{1}{n}\sum_{i=1}^n \varepsilon_i \langle \tilde{g}, k^{(B)}_{w^\top x_i}   \rangle \big| \\
& \leq D  \sup_{w \in \mathcal{S}^{d-1}, \| \tilde{g}\|_\mathcal{H} \leq 1}  \frac{1}{n}\sum_{i=1}^n \varepsilon_i \langle \tilde{g}, k^{(B)}_{w^\top x_i}   \rangle  = D   \sup_{f = \tilde{g}(w^\top \cdot), w \in \mathcal{S}^{d-1}, \| \tilde{g}\|_\mathcal{H} \leq 1}  \frac{1}{n}\sum_{i=1}^n \varepsilon_i f(x_i).
\end{align*}}

Taking the expectation over the data set and $\varepsilon$ on both sides \red{and renaming $\tilde{g}$ as $g$ for ease of exposition} yields the desired result.
\end{proof}
\vspace{0.0em}
\subsubsection{Lemma~\ref{lemma:best_g} and its Proof}
\label{sec:lemma_best_g}
This lemma provides an explicit formula for computing the supremum over functions within the unit ball of $\mathcal{H}$, which we can then use for the calculation of Gaussian complexity.
\begin{lemma}[Optimal $g$ in Gaussian Complexity]
\label{lemma:best_g}
For any data set $(x_1, \ldots, x_n)$ and $w\in \mathbb{R}^d$, with $K^{(w)} \in \mathbb{R}^{n \times n}$ the kernel matrix of kernel $k^{(B)}$ with data $(w^\top x_1, \ldots, w^\top x_n)$ and $\varepsilon \in \mathbb{R}^n$,
\begin{align*}
\sup_{\|g\|_\mathcal{H} \leq 1} \frac{1}{n} \sum_{i=1}^n \varepsilon_i g(w^\top x_i) = \frac{1}{n} \sqrt{\varepsilon^\top K^{(w)} \varepsilon}
\end{align*}   
\end{lemma}

\begin{proof}[Proof of Lemma~\ref{lemma:best_g}]
By applying the definitions, we obtain
\begin{align*}
\sup_{\|g\|_\mathcal{H} \leq 1 }\frac{1}{n}\sum_{i=1}^n \varepsilon_i g(w^\top x_i) &= \sup_{\|g\|_\mathcal{H} \leq 1 } \frac{1}{n}\sum_{i=1}^n \varepsilon_i  \langle g, k^{(B)}_{w^\top x_i} \rangle   =  \sup_{\|g\|_\mathcal{H} \leq 1 } \langle g, \frac{1}{n}  \sum_{i=1}^n \varepsilon_i k^{(B)}_{w^\top x_i} \rangle \\
& = \frac{1}{n} \left\langle \frac{\sum_{i=1}^n \varepsilon_i k^{(B)}_{w^\top x_i}}{\|\sum_{i=1}^n \varepsilon_i k^{(B)}_{w^\top x_i}\|_\mathcal{H} } , \sum_{j=1}^n \varepsilon_j k^{(B)}_{w^\top x_j}  \right\rangle \\
& = \frac{1}{n} \|\sum_{i=1}^n \varepsilon_i k^{(B)}_{w^\top x_i}\|_\mathcal{H} = \frac{1}{n} \sqrt{\varepsilon^\top K^{(w)}\varepsilon},
\end{align*}
which is the desired result.
\end{proof}
\vspace{0.0em}
\subsubsection{Proof of Lemma~\ref{lemma:delta_lip}}
\label{proof:lemma:delta_lip}
\begin{proof}[Proof of Lemma~\ref{lemma:delta_lip}]
Define $ g_\zeta $ such that $ g_\zeta(0) = 0 $ and $ g_\zeta^\prime(x) = \min(|g^\prime(x)|, 1/\zeta)\sign(g^\prime(x)) $. Note that $ \|g_\zeta^\prime\|_\infty \leq \frac{1}{\zeta} $, thus $ g_\zeta $ is $ \frac{1}{\zeta} $-Lipschitz. Additionally, for any $ a \in \mathbb{R} $,

\begin{align*}
    |g_\zeta(a) - g(a)| &= \left| \int_{0}^{a} \big( g_\zeta^\prime(t) - g^\prime(t) \big) \, {\rm d}t \right| \\
    &\leq \int_{0}^{a} \left| g_\zeta^\prime(t) - g^\prime(t) \right| \, {\rm d}t  \leq \int_{-\infty}^{+\infty} \mathds{1}_{|g^\prime(t)| \geq 1/\zeta} \big( |g^\prime(t)| - 1/\zeta \big) \, {\rm d}t \\
    & \red{\leq \int_{-\infty}^{+\infty} \mathds{1}_{|g^\prime(t)| \geq 1/\zeta} |g^\prime(t)| \, {\rm d}t \leq \int_{-\infty}^{+\infty} \zeta |g^\prime(t)|^2 \, {\rm d}t }\\ 
    & \red{\leq \zeta \quad { \rm since} \quad \int_{-\infty}^{+\infty} \big(g^\prime(t)\big)^2 \, {\rm d}t \leq 1},
\end{align*}
yielding the desired result.
\end{proof}
\vspace{0.0em}

\red{\subsubsection{Lemma~\ref{lemma:inspired_by_bartlett} and its Proof}}
\red{\begin{lemma}[Gaussian Complexity of Finite Set of Lipschitz Functions]
\label{lemma:inspired_by_bartlett}
\red{Let $h_1, \ldots h_M$ be 1-Lipschitz functions from $\mathbb{R}$ to $\mathbb{R}$ and let $\varepsilon$ be a random centred Gaussian vector with identity covariance matrix. Then}
\begin{align*}
\mathbb{E}_\varepsilon \bigg( \sup_{h \in \{h_1, \ldots, h_M\}, w \in \mathcal{S}^{d-1}} &\frac{1}{n} \sum_{i=1}^n \varepsilon_i h(w^\top x_i) \bigg) \\
&\leq  \mathbb{E}_\varepsilon \left(  \bigg\| \frac{\sqrt{2}}{n} \sum_{i=1}^n \varepsilon_i x_i \bigg\|^* + \sqrt{8 \frac{\sum_{i=1}^n (\|x_i\|^*)^2}{n^2}}\sqrt{2 \log M } \right).
\end{align*}
\end{lemma}}

\red{This lemma is inspired by \citet{Bartlett}.}

\begin{proof}[Proof of Lemma~\ref{lemma:inspired_by_bartlett}]
\red{We use Slepian's lemma \citep[Corollary 3.14]{talagrand}. For $h\in \{h_1, \ldots h_M\}, w \in \mathcal{S}^{d-1}$, let 
\begin{equation*}
    X_{h,w} := \frac{1}{n}\sum_{i=1}^n \varepsilon_i h(w^\top x_i) \text{ and } Y_{h,w} = \frac{\sqrt{2}}{n} \sum_{i=1}^n \varepsilon_i w^\top x_i + \sum_{j=1}^M \mathds{1}_{h=h_j} \tilde{\varepsilon}_j \sqrt{8 \frac{\sum_{i=1}^n (\|x_i\|^*)^2}{n^2}},
\end{equation*}
where $\tilde{\varepsilon}$ is a centred Gaussian vector with identity covariance matrix independent of $\varepsilon$. Notice that for $h, \Tilde{h} \in \{h_1, \ldots h_M\}, w, \Tilde{w} \in \mathcal{S}^{d-1}$, we have
\begin{align*}
    \mathbb{E}_\varepsilon& ((X_{h,w} - X_{\Tilde{h}, \Tilde{w}} )^2) = \frac{1}{n^2} \sum_{i=1}^n (h(w^\top x_i) - \Tilde{h}(\Tilde{w}^\top x_i))^2  \\
    & \leq \frac{1}{n^2} \sum_{i=1}^n  (h(w^\top x_i) -  h(\tilde{w}^\top x_i) + h(\tilde{w}^\top x_i) -\Tilde{h}(\Tilde{w}^\top x_i))^2 \\
    & \leq \frac{2}{n^2} \sum_{i=1}^n(h(w^\top x_i) -  h(\tilde{w}^\top x_i))^2 + (h(\tilde{w}^\top x_i) -\Tilde{h}(\Tilde{w}^\top x_i))^2 .
\end{align*}
We can then deal with the two terms separately. For the left term, the fact that $h$ is 1-Lipschitz yields that 
\begin{equation*}
 \frac{2}{n^2}   \sum_{i=1}^n(h(w^\top x_i) -  h(\tilde{w}^\top x_i))^2  \leq \frac{2}{n^2} \sum_{i=1}^n (w^\top x_i -  \tilde{w}^\top x_i)^2.
\end{equation*}
Then, using the fact that $h-\tilde{h}$ is $2$-Lipschitz and $h(0)=\tilde{h}(0) =0$, we have
\begin{align*}
\frac{2}{n^2} \sum_{i=1}^n (h(\tilde{w}^\top x_i) -\Tilde{h}(\Tilde{w}^\top x_i))^2 & =   \frac{2}{n^2}\sum_{i=1}^n (h(\tilde{w}^\top x_i) -\Tilde{h}(\Tilde{w}^\top x_i) -(h(0) - \tilde{h}(0))^2 \\
& \leq \frac{2}{n^2} \sum_{i=1}^n \mathds{1}_{h\neq h^\prime} 4 (w^\top x_i)^2  \leq   \mathds{1}_{h \neq \Tilde{h}} \frac{8}{n^2} \sum_{i=1}^n (\|x_i\|^*)^2 .
\end{align*}}

\red{All in all $ \mathbb{E}_\varepsilon ((X_{h,w} - X_{\Tilde{h}, \Tilde{w}} )^2) \leq \mathbb{E}_{\varepsilon, \tilde{\varepsilon}}((Y_{h,w} - Y_{\Tilde{h}, \Tilde{w}} )^2) $ therefore we can apply Slepian's lemma and obtain
\begin{align*}
& \mathbb{E}_\varepsilon \left( \sup_{h \in \{h_1, \ldots, h_M\}, w \in \mathcal{S}^{d-1}} \frac{1}{n} \sum_{i=1}^n \varepsilon_i h(w^\top x_i) \right) \nonumber \\
     &\leq \mathbb{E}_{\varepsilon, \tilde{\varepsilon}} \left( \sup_{h \in \{h_1, \ldots, h_M\}, w\in \mathcal{S}^{d-1}} \frac{\sqrt{2}}{n} \sum_{i=1}^n \varepsilon_i w^\top x_i + \sum_{j=1}^M \Tilde{\varepsilon}_j \mathds{1}_{h=h_j} \sqrt{8 \frac{\sum_{i=1}^n (\|x_i\|^*)^2}{n^2}} \right).
\end{align*}}

\red{We then remark that the first term of the expectation does not depend on $h$ and that we can take the supremum over the sphere explicitly, while the second term does not depend on $w$ and we can also take the supremum over $\{h_1, \ldots, h_M\}$ explicitly
\begin{align*}
\mathbb{E}_\varepsilon &\left( \sup_{h \in \{h_1, \ldots, h_M\}, w \in \mathcal{S}^{d-1}} \frac{1}{n} \sum_{i=1}^n \varepsilon_i h(w^\top x_i) \right) \nonumber \\
&\leq \mathbb{E}_{\varepsilon, \tilde{\varepsilon}} \left( \sup_{w \in \mathcal{S}^{d-1}} \frac{\sqrt{2}}{n} \sum_{i=1}^n \varepsilon_i w^\top x_i + \sup_{h \in \{h_1, \ldots, h_M\}} \sum_{j=1}^M \Tilde{\varepsilon}_j \mathds{1}_{h=h_j} \sqrt{8 \frac{\sum_{i=1}^n(\|x_i\|^*)^2}{n^2}} \right) \nonumber \\
& \leq \mathbb{E}_{\varepsilon, \tilde{\varepsilon}} \left( \sup_{w \in \mathcal{S}^{d-1}} \frac{\sqrt{2}}{n} \sum_{i=1}^n \varepsilon_i w^\top x_i + \sup_{j \in [M]}  \Tilde{\varepsilon}_j \sqrt{8 \frac{\sum_{i=1}^n (\|x_i\|^*)^2}{n^2}} \right) \nonumber \\
&\leq  \mathbb{E}_{\varepsilon} \left(  \bigg\| \frac{\sqrt{2}}{n} \sum_{i=1}^n \varepsilon_i x_i \bigg\|^* \right) + \sqrt{8 \frac{\sum_{i=1}^n (\|x_i\|^*)^2}{n^2}}\sqrt{2 \log M } .
\end{align*}}
\end{proof}

\subsubsection{Proof of Lemma~\ref{lemma:data_dependent}}
\label{proof:lemma:data_dependent}
\begin{proof}[Proof of Lemma~\ref{lemma:data_dependent}]
We begin with the bounded case. The bounds on the expectations are clearly valid. Then, since $1 + \sqrt{\|X\|^*}$ is a bounded variable, it is necessarily subgaussian with a variance proxy bounded by $\frac{(1 + \sqrt{R})^2}{2\log(2)} \leq (1 + \sqrt{R})^2$ \citep[Proposition 2.5.2 (iv)]{Vershynin_2018}.

Next, we consider the subgaussian case. Using the Cauchy-Schwarz inequality, we handle the case where $\|\cdot\|^* = \|\cdot\|_2$ using \citet[Proposition 2.5.2]{Vershynin_2018}
\[
\sqrt{\mathbb{E}_X (\|X\|_2)} \leq \big(\mathbb{E}_X (\|X\|_2^2)\big)^{1/4} \leq \sqrt{6} \big(\sum_{a=1}^d \sigma_a^2\big)^{1/4}.
\]

For the $\|\cdot\|_\infty$ case, applying \citet[Exercise 2.5.10]{Vershynin_2018} with the  constant made explicit yields the desired result.

For the second expectation with $\|\cdot \|^* = \|\cdot \|_2$, we have
\begin{align*}
\mathbb{E}_{\mathcal{D}_n}\big(\max_{i\in [n]} \|X_i\|_2^2\big) &= \mathbb{E} \max_{i \in [n]} \sum_{a=1}^d ((X_i)_a)^2 \leq \sum_{a=1}^d \mathbb{E} \max_{i \in [n]} \big((X_i)_a\big)^2 \\
& \leq \sum_{a=1}^d \frac{1}{t} \log\big( \mathbb{E} \big( e^{t \max_{i \in [n]} ((X_i)_a)^2} \big) \big)  \leq \sum_{a=1}^d \frac{1}{t} \log \big( n \mathbb{E} \big( e^{t ((X_i)_a)^2} \big) \big),
\end{align*}
for all $t > 0$. We can then bound this by $\sum_{a=1}^d \frac{1}{t} \log (n e^{t (6\sqrt{2e}\sigma_a)^2})$ for $t < 1/(6\sqrt{2e}\sigma_a)^2$, yielding:
\[
\mathbb{E}_{\mathcal{D}_n}\big(\max_{i\in [n]} \|X_i\|_2^2\big) \leq 72e(1 + \log(n)) \sum_{a=1}^d \sigma_a^2.
\]
The same proof technique applies to $\mathbb{E}_{\mathcal{D}_n}\big(\max_{i\in [n]} \|X_i\|_\infty^2\big)$, yielding the desired result.

Finally, we consider the subgaussianity of $1 + \sqrt{\|X\|^*}$. Note that the sum of two subgaussian variables is subgaussian. Using \citet[Proposition 2.5.2 (ii)]{Vershynin_2018}, for two real random variables $Z$ and $\Tilde{Z}$ with variance proxies $\sigma^2$ and $\Tilde{\sigma}^2$ respectively, we have that $Z + \Tilde{Z}$ is subgaussian with variance proxy $(\sigma + \Tilde{\sigma})^2$. Additionally, the absolute value of a subgaussian variable is also subgaussian with the same variance proxy \citep[Proposition 2.5.2]{Vershynin_2018}.

For $\|\cdot\| = \|\cdot\|_2$, we have $1 + \sqrt{\|X\|_2} \leq 1 + \sum_{a=1}^d |X_a|$. Since 1 and $X_a$ are subgaussian variables, this yields the desired result.

For $\|\cdot \|_\infty$, for all $t > 0$,
\begin{align*}
\mathbb{P}(\|X\|_\infty \geq \sqrt{2\sigma^2 \log(2d)} + t) &\leq 2de^{-\frac{(\sqrt{2\sigma^2 \log(2d)} + t)^2}{2\sigma^2}} \\
& \leq 2e^{-\frac{t^2}{2\sigma^2} - \frac{t\sqrt{\log(2d)}}{\sqrt{2\sigma^2}}}   \leq 2e^{-\frac{t^2}{2\sigma^2}}.
\end{align*}
Thus, $\|X\|_\infty - \sqrt{2\sigma^2\log(2d)}$ is subgaussian with variance proxy $\sigma^2$. Therefore, $\|X\|_\infty$ is subgaussian with variance proxy bounded by $\sigma^2(1 + \sqrt{\log(2d)})^2$. Then, $1 + \sqrt{\|X\|_\infty}$ is subgaussian because it is less than $2 + \|X\|_\infty$, which is subgaussian \citep[Proposition 2.5.2]{Vershynin_2018} with a variance proxy bounded by that of $2 + \|X\|_\infty$, yielding the desired result.
\end{proof}
\vspace{0.0em}
\subsection{Lemmas Needed for Section~\ref{sec:bound_true_risk} and their Proofs}
Here we provide lemmas necessary for the proof of Theorem~\ref{theo:last} and the analysis of its distribution-dependent terms.
\subsubsection{Lemma~\ref{lemma:constrained_predictor} and its Proof}
Lemma~\ref{lemma:constrained_predictor} relates the Gaussian complexity to useful quantities to bound the expected risk.
\begin{lemma}{(Use of Gaussian Complexity)}
\label{lemma:constrained_predictor}
Let $D > 0$ and the data set $\mathcal{D}_n= (x_i,y_i)_{i \in [n]}$ consists of i.i.d.~samples of the random variable $(X,Y) \in \mathcal{X}\times \mathcal{Y}$. Assume that the loss $\ell$ is $L$-Lipschitz in its second (bounded) argument, i.e., $\forall y \in \mathcal{Y}, a \in \{ f(x) \mid x \in \mathcal{X}, \ f \in \mathcal{F}_\infty, \ \Omega(f) \leq D \}, a \to \ell(y, a)$ is $L$-Lipschitz. Then, we have
\begin{equation*}
  \mathbb{E}_{\mathcal{D}_n}\left(\sup_{f \in \mathcal{F}_\infty, \Omega(f) \leq D} \widehat{\mathcal{R}}(f) - \mathcal{R}(f) + \sup_{f \in \mathcal{F}_\infty, \Omega(f) \leq D} \mathcal{R}(f) -  \widehat{\mathcal{R}}(f) \right) \leq 6 DL\left(\frac{1}{\sqrt{n}} + G_n\right).
\end{equation*}
\end{lemma}

\begin{proof}[Proof of Lemma~\ref{lemma:constrained_predictor}]
By \citet[Proposition 4.2]{francis_book}, we have
\begin{align*}
     &\mathbb{E}_{\mathcal{D}_n}\left(\sup_{f \in \mathcal{F}_\infty, \Omega(f) \leq D} \widehat{\mathcal{R}}(f) - \mathcal{R}(f) + \sup_{f \in \mathcal{F}_\infty, \Omega(f) \leq D}  \mathcal{R}(f) - \widehat{\mathcal{R}}(f) \right) \\&\leq 4 \mathbb{E}_{\tilde{\varepsilon}, \mathcal{D}_n} \left( \sup_{f \in \mathcal{F}_\infty, \Omega(f) \leq D} \frac{1}{n}\sum_{i=1}^n \varepsilon_i \ell(y_i, f(x_i)) \right),
\end{align*}
where $\tilde{\varepsilon}$ consists of i.i.d.~Rademacher variables. 

Next, applying the contraction principle from \citet[Proposition 4.3]{francis_book}, we get
\begin{equation*}
    \mathbb{E}_{\tilde{\varepsilon}, \mathcal{D}_n} \left( \sup_{f \in \mathcal{F}_\infty, \Omega(f) \leq D} \frac{1}{n}\sum_{i=1}^n \varepsilon_i \ell(y_i, f(x_i)) \right) \leq  \mathbb{E}_{\tilde{\varepsilon}, \mathcal{D}_n} \left( \sup_{f \in \mathcal{F}_\infty, \Omega(f) \leq D} \frac{1}{n}\sum_{i=1}^n \tilde{\varepsilon}_i f(x_i) \right).
\end{equation*}

Then, using \citet[Exercise 5.5]{Wainwright_2019}, we have
\begin{equation*}
    \mathbb{E}_{\tilde{\varepsilon}, \mathcal{D}_n} \left( \sup_{f \in \mathcal{F}_\infty, \Omega(f) \leq D} \frac{1}{n}\sum_{i=1}^n \tilde{\varepsilon}_i f(x_i) \right) \leq \sqrt{\frac{\pi}{2}} \mathbb{E}_{\varepsilon, \mathcal{D}_n} \left( \sup_{f \in \mathcal{F}_\infty, \Omega(f) \leq D} \frac{1}{n}\sum_{i=1}^n \varepsilon_i f(x_i) \right),
\end{equation*}
where $\varepsilon \sim \mathcal{N}(0, I_d)$.

Finally, by applying Lemma~\ref{lemma:rademacher_without_integral} and combining all these results, we obtain the desired inequality
\begin{equation*}
  \mathbb{E}_{\mathcal{D}_n}\left(\sup_{f \in \mathcal{F}_\infty, \Omega(f) \leq D} \widehat{\mathcal{R}}(f) - \mathcal{R}(f) + \sup_{f \in \mathcal{F}_\infty, \Omega(f) \leq D} \mathcal{R}(f) - \widehat{\mathcal{R}}(f) \right) \leq 6 DL\left(\frac{1}{\sqrt{n}} + G_n\right).
\end{equation*}
\end{proof}
\vspace{0.0em}
\subsubsection{Lemma~\ref{lemma:subgaussian_cosh} and its Proof}
Lemma~\ref{lemma:subgaussian_cosh} describes a useful property on the expectation of the hyperbolic cosine of a subgaussian random variable.
\begin{lemma}{(Technical Lemma on Subgaussian Random Variables)}
\label{lemma:subgaussian_cosh}
Let $Z$ be a real-valued random variable (not necessarily centred) that is subgaussian (see Definition~\ref{def:subgaussian}.)  Then, for all $\lambda \in \mathbb{R}$,
\begin{equation*}
    \mathbb{E}\left( \cosh(\lambda Z) \right) \leq e^{(6\sqrt{2e})^2\sigma^2\lambda^2}.
\end{equation*}
\end{lemma}

\begin{proof}[Proof of Lemma~\ref{lemma:subgaussian_cosh}]
An equivalent definition of subgaussianity is that for all $\lambda \in \mathbb{R}$, if $6\sqrt{2e} \sigma|\lambda| \leq 1$, then $\mathbb{E}(e^{\lambda^2 Z^2}) \leq e^{(6\sqrt{2e})^2 \sigma^2 \lambda^2}$, see  \citet[Proposition 2.5.2]{Vershynin_2018}.

First, in the case $|\lambda| \leq \frac{1}{6\sqrt{2e}\sigma}$. Using the inequality $e^x \leq x + e^{x^2}$ for all $x \in \mathbb{R}$, we get
\begin{equation*}
    \mathbb{E}\left( \cosh(\lambda Z) \right) \leq \mathbb{E}\left( \frac{\lambda Z + e^{\lambda^2 Z^2} - \lambda Z + e^{\lambda^2 Z^2}}{2} \right) = \mathbb{E}\left( e^{\lambda^2 Z^2} \right) \leq e^{(6\sqrt{2e})^2\sigma^2\lambda^2}.
\end{equation*}

Next, consider the case $|\lambda| \geq \frac{1}{6\sqrt{2e}\sigma}$. We can bound the expectation as follows
\begin{align*}
   \mathbb{E}\left( \cosh(\lambda Z) \right) &\leq \mathbb{E}\left( e^{|\lambda Z|} \right) = \mathbb{E}\left( e^{6\sqrt{2e}\sigma|\lambda| \frac{|Z|}{6\sqrt{2e}\sigma}} \right) \leq \mathbb{E}\left( e^{(6\sqrt{2e})^2\sigma^2\lambda^2/2 +  \frac{Z^2}{2(6\sqrt{2e})^2\sigma^2}} \right) \\
   & \leq e^{(6\sqrt{2e})^2\sigma^2\lambda^2/2} e^{1/2} \leq e^{(6\sqrt{2e})^2\sigma^2\lambda^2},
\end{align*}
where we use the fact that $(6\sqrt{2e})^2\sigma^2\lambda^2 \geq 1$ to justify the final inequality.

Thus, in both cases, we have shown that $\mathbb{E}\left( \cosh(\lambda Z) \right) \leq e^{(6\sqrt{2e})^2\sigma^2\lambda^2}$, proving the lemma.
\end{proof}
\vspace{0.0em}
\subsubsection{Lemma~\ref{lemma:mcdiarmidv2} and its Proof}
Lemma~\ref{lemma:mcdiarmidv2} is an application of McDiarmid's inequality (a specific version by \citet{meir_zhang} for subgaussian random variables) to our learning problem.
\begin{lemma}{(Use of McDiarmid's Inequality)}
\label{lemma:mcdiarmidv2}
Let $D>0$ and $\delta \in (0,1)$. Assume that $1+\sqrt{\|X\|^*}$ is subgaussian with variance proxy $\sigma^2$ and that the  loss $\ell$ is $L$-Lipschitz in its second (bounded) argument, i.e., $\forall y \in \mathcal{Y}, a \in \{ f(x) \mid x \in \mathcal{X}, \ f \in \mathcal{F}_\infty, \ \Omega(f) \leq D \}, a \to \ell(y, a)$ is $L$-Lipschitz. Then, with probability greater than $1-\delta$,
\begin{align*}
\sup_{f \in \mathcal{F}_\infty, \Omega(f) \leq D}&\widehat{\mathcal{R}}(f) - \mathcal{R}(f) + \sup_{f \in \mathcal{F}_\infty, \Omega(f) \leq D} \mathcal{R}(f) - \widehat{\mathcal{R}}(f) \\
&\leq \mathbb{E}_{\mathcal{D}_n} \left( \sup_{f \in \mathcal{F}_\infty, \Omega(f) \leq D}\widehat{\mathcal{R}}(f) - \mathcal{R}(f) + \sup_{f \in \mathcal{F}_\infty, \Omega(f) \leq D} \mathcal{R}(f) - \widehat{\mathcal{R}}(f) \right) \\
&\quad + \frac{48\sqrt{2e}LD\sigma}{\sqrt{n}} \sqrt{\log \frac{1}{\delta}}.
\end{align*}
\end{lemma}

\begin{proof}[Proof of Lemma~\ref{lemma:mcdiarmidv2}]
We use a specific version of McDiarmid's inequality \citep[Theorem 3]{meir_zhang}. First, we show that the conditions for applying the theorem are met. Let
$\tilde{H} := \{ h: (x,y) \in \mathcal{X} \times \mathcal{Y} \to \ell(y, f(x)) - \ell(y, \tilde{f}(x)) \mid \Omega(f) \leq D, \Omega(\tilde{f}) \leq D \}$.
For any $\lambda > 0$, we have
\begin{align*}
  \mathbb{E}_{X,Y}&\left(\sup_{h, \tilde{h} \in \tilde{H}} \cosh(2\lambda (h(X,Y) - \tilde{h}(X,Y)))\right) \\
    & = \mathbb{E}_{X,Y}\left(\sup_{f, \Omega(f) \leq D, \tilde{f}, \Omega(\tilde{f}) \leq D} \cosh(2\lambda (\ell(Y, f(X)) - \ell(Y, \tilde{f}(X))))\right) \\
    &\leq \mathbb{E}_{X,Y}\left(\sup_{f, \Omega(f) \leq D, \tilde{f}, \Omega(\tilde{f}) \leq D} \cosh(2\lambda L |f(X) - \tilde{f}(X)|)\right) \\
    &\leq \mathbb{E}_{X,Y}\left(\sup_{f, \Omega(f) \leq D, \tilde{f}, \Omega(\tilde{f}) \leq D} \cosh(4\lambda LD (1 + \sqrt{\|X\|^*}))\right) \\
    &= \mathbb{E}_{X,Y}\left(\cosh(4\lambda LD (1 + \sqrt{\|X\|^*}))\right)  \leq e^{(48\sqrt{e})^2 L^2 D^2 \sigma^2 \lambda^2 / 2},
\end{align*}
where the last inequality follows from Lemma~\ref{lemma:subgaussian_cosh}. Hence, the condition is verified with $ M = 48\sqrt{e}LD\sigma $ and applying  \citet[Theorem 3]{meir_zhang} yields the desired result.
\end{proof}
\vspace{0.0em}
\section{Numerical Experiments}
In this section, we detail the parameters and methodology used in the different experiments. The code needed to run the experiments can be found at \url{https://github.com/BertilleFollain/BKerNN}.

\subsection{Experiment 1: Optimisation procedure, Importance of Positive Homogeneous Kernel}
\label{app:exp1}
Each method was tuned using 5-fold cross-validation with grid search, using negative mean squared error as the scoring metric. The training was set for 20 iterations and the step-size parameter ($\gamma$) was set to 500, with backtracking enabled. Regularisation parameter candidates were $\lambda = \{ 0.05, 0.1, 0.5, 1, 1.5 \} \times 2 \max_{i \in [n]} \|x_i\|_2/n$. Once the regularisation parameters had been selected, we trained from scratch for 200 iterations, with the other parameters kept as before.

\subsection{Experiments 2 \& 3: Influence of Parameters (Number of Particles \texorpdfstring{$m$}{m}, Regularisation Parameter \texorpdfstring{$\lambda$}{lambda}, and Type of Regularisation)}
\label{app:num_exp2&3}
For Experiment 2, in the first subplot, we set the step-size parameter $\gamma$ to 500 and the number of iterations to 50. The regularisation type was set to $\Omega_{{ \rm basic}}$ and the regularisation parameter to $\lambda =0.02$. The tested values of $m$ were 1, 3, 5, 7, 10, 15, 20, 30, 40, and 50.

In the second subplot, we varied the regularisation parameter $\lambda$ in 0.0005, 0.001, 0.005, 0.01, 0.02, 0.05, 0.1, 0.3, and 0.5,  while keeping the number of particles fixed at $m = 10$.

In Experiment 3, the \textsc{BKerNN} model was instantiated with a fixed number of particles $m = 20$, step-size parameter $\gamma = 500$, and number of iterations $25$. The regularisation parameter $\lambda$ was set as $2 \max_{i \in n} \|x_i\|_2 / n$.

\subsection{Experiment 4: Comparison to Neural Network on 1D Examples, Influence of Number of Particles/Width of Hidden Layer \texorpdfstring{$m$}{m}}
\label{app:num_exp4}

In Experiment 4, we investigated the performance of two learning methods, \textsc{BKerNN} and \textsc{ReLUNN}, on three different 1D functions. The training set always consists of 128 samples, with $x$ sampled uniformly between -1 and 1, while the target function/test set without noise consists of 1024 equally spread out points. The response was then generated as follows. For the first function, $ y = \sin(2\pi x) + { \rm noise}$, for the second $ y = { \rm sign}(\sin(2\pi x)) + { \rm noise} $, for the third  $ y = 4|x + 1 - 0.25 - \lfloor x + 1 - 0.25 \rfloor - 0.5| - 1 + { \rm noise} $, where the noise is always normal, centred and with standard deviation equal to 0.2. For \textsc{BKerNN}, the regularisation parameter $ \lambda $ was selected from [0.005, 0.01, 0.02, 0.05] using 5-fold cross-validation and the negative mean squared error score. \textsc{ReLUNN} was trained using a batch size of 16, a number of iterations equal to 400,000 and a step-size of 0.005.

\subsection{Experiment 5: Prediction Score and Feature Learning Score Against Growing Dimension and Sample Size, a Comparison of \textsc{BKerNN} with Kernel Ridge Regression and a ReLU Neural Network}
\label{app:num_exp5}

In Experiment 5, data sets were generated with input data uniformly sampled within the hypercube $[-1, 1]^d$. The feature matrix $P$ was generated from the orthogonal group. For each configuration, training and test sets of sizes $n$ and $n_{{ \rm test}} = 201$, respectively, were created. Output labels $y$ were computed as $y_i = |\sum_{a=1}^k (\sin((P^\top x_i)_a))|$, with $k = 3$ relevant features.

The first two plots fixed the dimension at 15 and varied sample sizes across [10, 20, 50, 100, 150, 200, 300, 400, 500]. The last two plots fixed the sample size at 212 and varied dimensions across [3, 5, 10, 20, 30, 40, 50]. Each configuration was repeated 10 times with different random seeds.

For \textsc{BKerNN}: $\lambda$ was set to $2 \max_{i \in [n]}(\|x_i\|_2) / n$, the number of particles was $m = 50$, the regularisation type $\Omega_{{ \rm feature}}$, the number of iterations 20, and step-size $\gamma = 500$ with backtracking line search. For \textsc{B\textsc{KRR}}, $\lambda$ was chosen similarly to \textsc{BKerNN}. For \textsc{ReLUNN}, the number of neurons was set to 50, learning rate to 0.05, batch size to 16, and number of iterations to 1500.

\subsection{Experiment 6: Comparison on Real Data Sets Between \textsc{BKerNN}, Kernel Ridge Regression and a ReLU Neural Network}
\label{app:num_exp6}

In Experiment 6, \textsc{B\textsc{KRR}} and both versions of \textsc{BKerNN} had regularisation parameter fixed equal to $\max_{i \in [n]}(\|x_i\|_2)/n$, where $n$ is the number of training samples (i.e. 400 \red{except for ``tecator'' where it was 192, ``semeion'' where it was 80 and ``pah'' where it was 64. The ratio of samples in the test set compared to the total dataset was nonetheless constant equal to 1/5, as in the other datasets}). Backtracking line search was used for \textsc{BKerNN} and the starting step-size was 500, while the number of iterations was 40 \red{(except for ``semeion'' where it was 4}. For \textsc{ReLUNN}, the batch size was 16, while the number of iterations was 2500 which corresponds to 100 epochs \red{(except for ``semeion'' where it was 250}, and the step-size was set to 0.01.

\bibliography{main}
\end{document}